%% file: log_2023.tex
\documentclass{article}
\usepackage{log_2023}						% for camera-ready version

\usepackage{booktabs}						% professional-quality tables
\usepackage{multirow}						% tabular cells spanning multiple rows
\usepackage{amsfonts}						% blackboard math symbols
\usepackage{graphicx}						% figures
\usepackage{duckuments}						% sample images

\usepackage[numbers,compress,sort]{natbib}	% for numerical citations

\input{include/00-packages.tex}

\title{Edge Directionality Improves Learning on Heterophilic Graphs}

\input{include/00-authors.tex}

\begin{document}

\maketitle

\begin{abstract}

Graph Neural Networks (GNNs) have become the de-facto standard tool for modeling relational data. However, while many real-world graphs are \emph{directed}, the majority of today's GNN models discard this information altogether by simply making the graph undirected.
The reasons for this are historical: 1) many early variants of spectral GNNs explicitly required undirected graphs, and 2) the first benchmarks on homophilic graphs did not find significant gain from using direction. 
In this paper, we show that in \emph{heterophilic} settings, treating the graph as directed increases the {\em effective homophily} of the graph, suggesting a potential gain from the correct use of directionality information.  
To this end, we introduce \ours{} (\oursacro{}), a novel general framework for deep learning on directed graphs. \oursacro{} can be used to extend \emph{any} Message Passing Neural Network (MPNN) to account for edge directionality information by performing separate aggregations of the \emph{incoming} and \emph{outgoing} edges. We prove that \oursacro{} matches the expressivity of the Directed Weisfeiler-Lehman test, exceeding that of conventional MPNNs. In extensive experiments, we validate that while our framework leaves performance unchanged on homophilic datasets, it leads to large gains over base models such as GCN, GAT and GraphSage on heterophilic benchmarks, outperforming much more complex methods and achieving new state-of-the-art results. 
The code for the paper can be found at \url{https://github.com/emalgorithm/directed-graph-neural-network}.
\end{abstract}

\input{include/01-introduction.tex}
\input{include/02-background.tex}

\input{include/03-directed_heterophily.tex}

\input{include/04-method.tex}
\input{include/05-related_works.tex}
\input{include/06-experiments.tex}
\input{include/07-conclusion.tex}

\section*{Acknowledgements}
Emanuele Rossi, Fabrizio Frasca and Michael Bronstein are supported in part by ERC Consolidator Grant No. 274228 (LEMAN).

\bibliographystyle{unsrtnat}
\bibliography{references}

\appendix
% \title{Appendix for "Edge Directionality Improves Learning on Heterophilic Graphs"}
% \maketitle
\input{include/08-appendix}

%%%%%%%%%%%%%%%%%%%%%%%%%%%%%%%%%%%%%%%%%%%%%%%%%%%%%%%%%%%%

\end{document}

%% file: include/00-packages.tex
\usepackage[utf8]{inputenc} % allow utf-8 input
\usepackage[T1]{fontenc}    % use 8-bit T1 fonts
\usepackage{url}            % simple URL typesetting
\usepackage{booktabs}       % professional-quality tables
\usepackage{amsfonts}       % blackboard math symbols
\usepackage{nicefrac}       % compact symbols for 1/2, etc.
\usepackage{microtype}      % microtypography
\usepackage{xcolor}         % colors
\usepackage{graphicx}       % Insert images
\usepackage{multirow}       % Multi-wors for tables
\usepackage{arydshln}       % Dashed lines for tables
\usepackage{enumitem}      % allow to change leftmargin with enum and itemize
\usepackage[textsize=tiny]{todonotes} % todonotes
\usepackage{caption}
\usepackage{subcaption}

\usepackage{amsmath}
\usepackage{amssymb}
\usepackage{mathtools}
\usepackage{amsthm}

\newcommand{\inp}{\leftarrow}
\newcommand{\out}{\rightarrow}
\newcommand\ours{Directed Graph Neural Network}
\newcommand\oursacro{Dir-GNN}
\newcommand{\R}{\mathbb{R}}

\newcommand\m{m}
\newcommand\nnodes{n}
\newcommand\nedges{m}
\newcommand\nfeatures{d}

\newcommand\nclass{C}
\newcommand\ilayer{k}
\newcommand\nlayers{K}
\newcommand\effhom{h^{(\text{eff})}}

\newcommand*{\ldblbrace}{\{\mskip-6mu\{}
\newcommand*{\rdblbrace}{\}\mskip-6mu\}}

\theoremstyle{plain}
\newtheorem{theorem}{Theorem}[section]

\newtheorem{lemma}[theorem]{Lemma}

\theoremstyle{definition}
\newtheorem{definition}[theorem]{Definition}

\theoremstyle{remark}

%% file: include/00-authors.tex
\author[E. Rossi et al.]{%
  Emanuele Rossi \\
  Imperial College London \\
  \And
  Bertrand Charpentier \\
  Technical University of Munich \\
  \And
  Francesco Di Giovanni \\
  University of Cambridge \\
  \And
  Fabrizio Frasca \\
  Imperial College London \\
  \And
  Stephan Günnemann \\
  Technical University of Munich \\
  \And
  Michael Bronstein \\
  University of Oxford \\
}

%% file: include/01-introduction.tex
\section{Introduction}
\label{sec:introduction}

% Most existing GNNs assume undirected graph
Graph Neural Networks (GNNs) have demonstrated remarkable success across a wide range of problems and fields~\cite{zhou2018gnn}. Most GNN models, however, assume that the input graph is undirected~\cite{kipf2016semi,velivckovic2017graph,hamilton2017inductive}, despite the fact that many real-world networks, such as citation and social networks, are inherently directed. Applying GNNs to directed graphs often involves either converting them to undirected graphs or only propagating information over incoming (or outgoing) edges, both of which may discard valuable information crucial for downstream tasks.

\begin{figure}[t!]
\centering
\vspace{-7mm}
\begin{subfigure}[b]{0.4\textwidth}
     \centering
     \includegraphics[width=\linewidth]{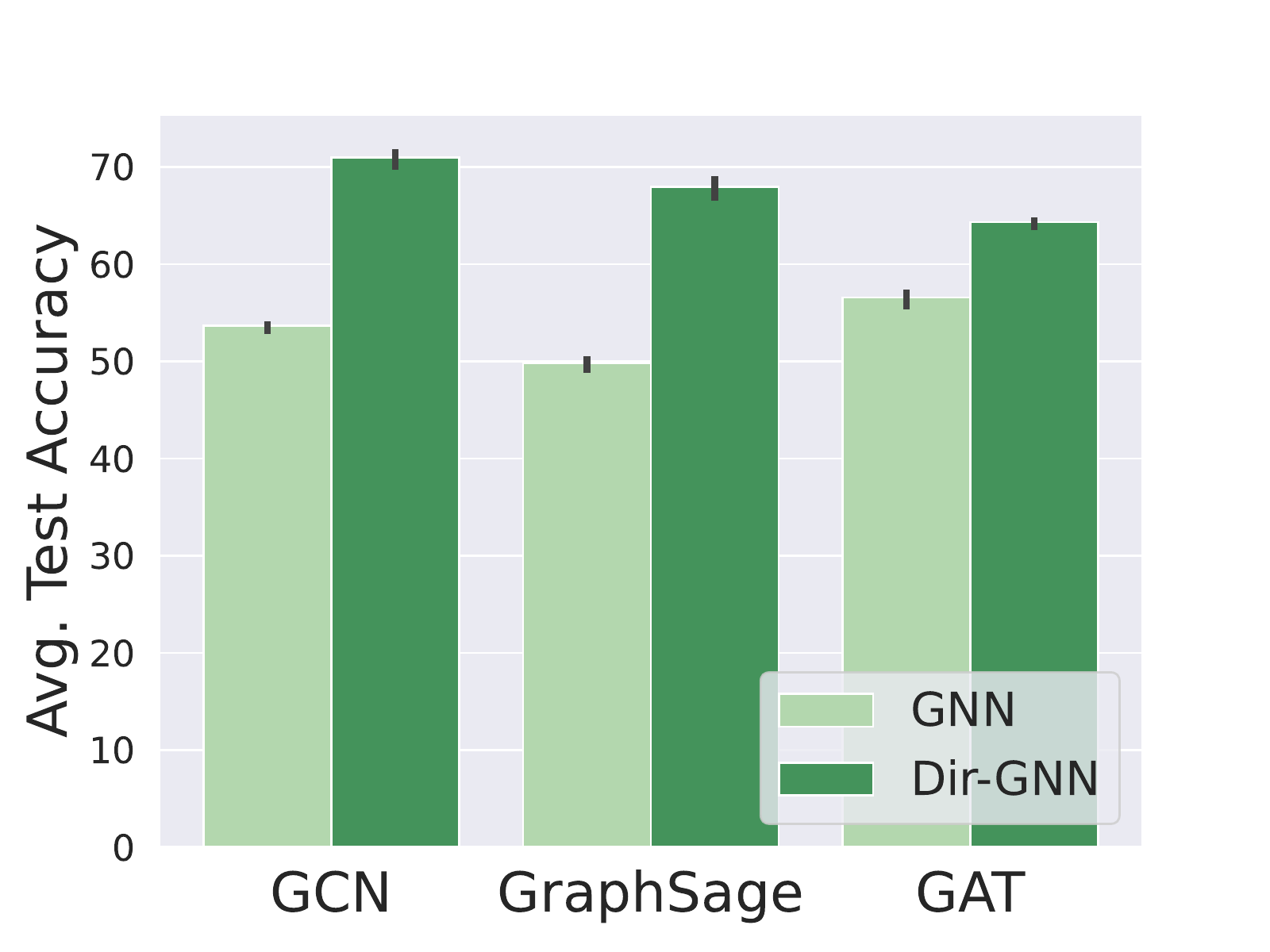}
     \caption{Heterophilic Graphs}
     \label{fig:aggregate_results_heterophilic}
\end{subfigure}
\begin{subfigure}[b]{0.375\textwidth}
     \centering
     \includegraphics[width=\linewidth,trim={1cm 0 0 0},clip]{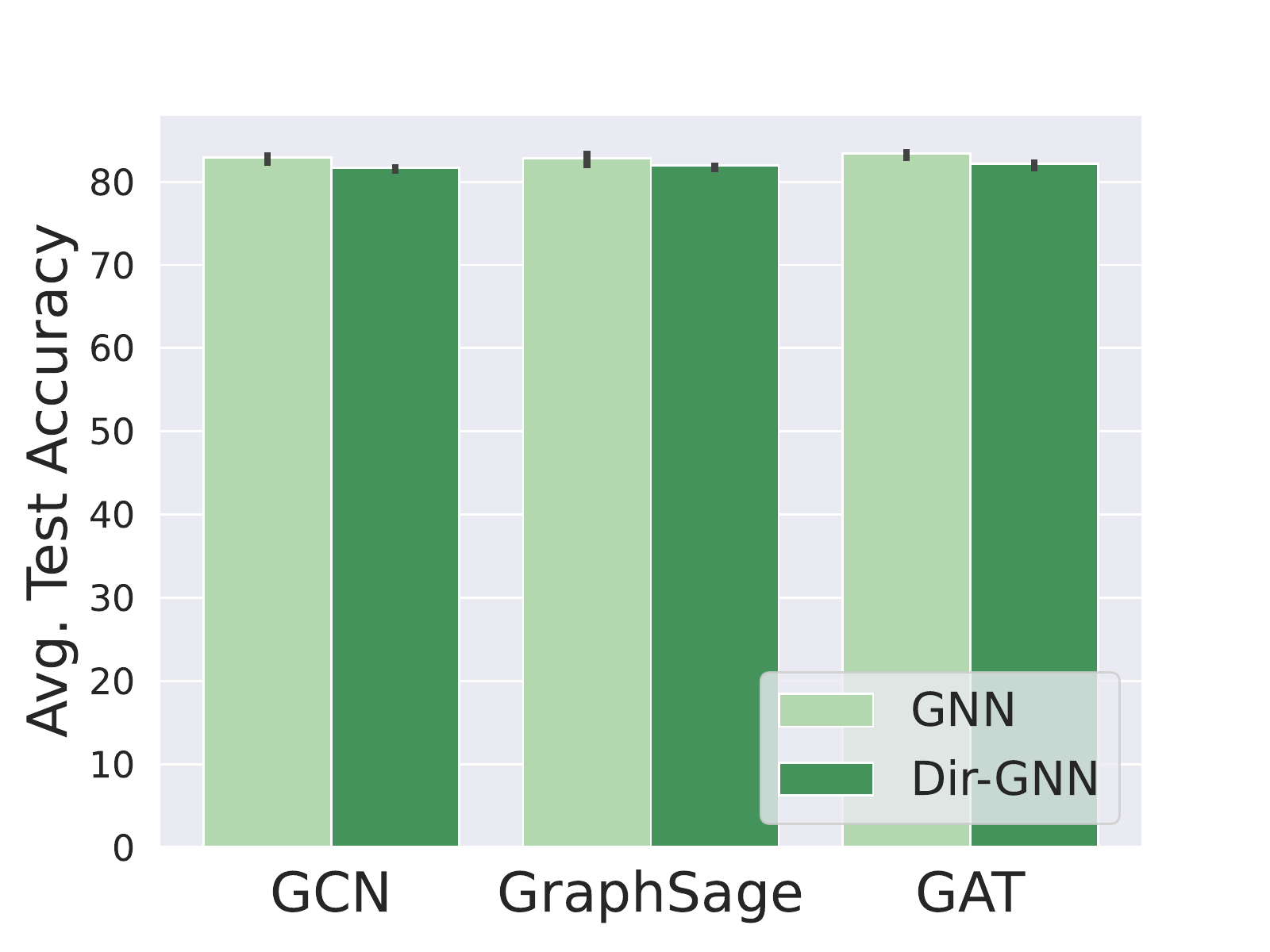}
     \caption{Homophilic Graphs}
     \label{fig:aggregate_results_homophilic}
\end{subfigure}
\vspace{-1mm}
\caption{Extending popular GNN architectures with our \oursacro{} framework to incorporate edge-directionality information brings large gains ($10\%$ to $15\%$) on heterophilic datasets (left), while leaving performance mostly unchanged on homophilic datasets (right). The plots illustrate the average performance over all datasets, while the full results are presented in \cref{tab:direction_ablation}.}
\label{fig:aggregate_results}
\vspace{-4mm}
\end{figure}

We believe the dominance of undirected graphs in the field is rooted in two ``original sins'' of GNNs.   
First, undirected graphs have symmetric Laplacians which admit orthogonal eigendecomposition. Orthogonal Laplacian eigenvectors act as a natural generalization of the Fourier transform and allow to express graph convolution operations in the Fourier domain. 
Since some of the early graph neural networks originated from the field of graph signal processing \cite{shuman2013emerging,sandryhaila2013discrete}, the undirected graph assumption was necessary for spectral GNNs \cite{bruna2013spectral,defferrard2016convolutional,kipf2016semi} to be properly defined. 
With the emergence of spatial GNNs, unified with the message-passing framework (MPNNs~\cite{gilmer2017neural}), 
this assumption was not strictly required anymore, as MPNNs can easily be applied to directed graphs by propagating over the directed adjacency, resulting however in information being propagated only in a single direction, at the risk of discarding useful information from the opposite one. However, early works empirically observed that making the graph undirected consistently leads to better performance on established node-classification benchmarks, which historically have been mainly \textit{homophilic} graphs such as Cora and Pubmed~\cite{sen:aimag08}, where neighbors tend to share the same label. Consequently, converting input graphs to undirected ones has become a standard part of the dataset preprocessing pipeline, to the extent that the popular GNN library PyTorch-Geometric~\cite{fey2019graph} includes a general utility function that automatically makes graphs undirected when loading datasets\footnote{ \href{https://github.com/pyg-team/pytorch_geometric/blob/66b17806b1f4a2008e8be766064d9ef9a883ff03/torch_geometric/io/npz.py\#L26}{This Pytorch-Geometric routine} is used to load datasets stored in an npz format. It makes some directed datasets, such as \href{https://github.com/pyg-team/pytorch_geometric/blob/6fa2ee7bfef32311df73ca78266c18c4449a7382/torch_geometric/datasets/citation_full.py\#L99}{Cora-ML} and \href{https://github.com/pyg-team/pytorch_geometric/blob/6fa2ee7bfef32311df73ca78266c18c4449a7382/torch_geometric/datasets/citation_full.py\#L99}{Citeseer-Full}, automatically undirected without any option to get the directed version instead.}.

% Edge direction contains crucial information to solve heterophilic tasks (but less information to solve homophilic tasks)
The key observation in this paper is that while accounting for edge directionality indeed does not help in homophilic graphs, it can bring extensive gains in \textit{heterophilic} settings (Fig.~\ref{fig:aggregate_results}), where neighbors tend to have different labels.
In the rest of the paper, we study \emph{why} and \emph{how} to use directionality to improve learning on heterophilic graphs. 

\textbf{Contributions.} 
Our contributions are the following:
\begin{itemize}[leftmargin=*]
    \setlength\itemsep{-0em}
    \item We show that considering the directionality of a graph substantially increases its \textit{effective homophily} in heterophilic settings, with negligible or no impact on homophilic settings (see~\cref{sec:directed_heterophily}). 
    \item We propose a novel and generic \ours{} framework (\oursacro{}) which extends \emph{any} MPNN to work on directed graphs by performing separate aggregations of the \emph{incoming} and \emph{outgoing} edges. Moreover, we show that \oursacro{} leads to more homophilic aggregations compared to their undirected counterparts (see~\cref{sec:method}). 
    \item Our theoretical analysis establishes that \oursacro{} is as expressive as the Directed Weisfeiler-Lehman test, while being \emph{strictly more expressive than MPNNs} (see ~\cref{sec:expressivity}).
    \item We empirically validate that augmenting popular GNN architectures with the \oursacro{} framework yields \emph{large improvements on heterophilic benchmarks}, achieving \emph{state-of-the-art results} and outperforming even more complex methods specifically designed for such settings. Moreover, this enhancement does not negatively impact the performance on homophilic benchmarks (see~\cref{sec:experiments}).
\end{itemize}

%% file: include/02-background.tex
\section{Background}
\label{sec:background}

We consider a directed graph $G=(V, E)$ with a node set $V$ of $\nnodes$ nodes, and an edge set $E$ of $\nedges$ edges. We define its respective directed adjacency matrix $\mathbf{A} \in \{0, 1\}^{\nnodes \times \nnodes}$ where $a_{ij}=1$ if $(i,j) \in E$ and zero otherwise, its respective undirected adjacency matrix $\mathbf{A}_u$ where $(a_{u})_{ij}=1$ if $(i,j) \in E$ or $(j, i) \in E$ and zero otherwise.
In this paper, we focus on the task of (semi-supervised) node classification on an attributed graph $G$ with node features arranged into the $\nnodes \times \nfeatures$ matrix $\mathbf{X}$ % \in \mathbb{R}^{\nnodes
 %\times \nfeatures}$ 
 and node labels $y_i \in \{1, ..., \nclass\}$.

\subsection{Homophily and Heterophily in Undirected Graphs} \label{sec:heterophily_in_gnns}

In this Section, we first review the heterophily metrics for undirected graphs, while in Sec.~\ref{sec:directed_heterophily} we propose heterophily metrics adapted to directed graphs. Most GNNs are based on the \textit{homophily assumption} for undirected graphs, i.e., that neighboring nodes tend to share the same labels. 
While a reasonable assumption in some settings, it turns out not to be true in many important applications such as gender classification on social networks or fraudster detection on e-commerce networks. 

\textbf{Homophily metrics.} Several metrics have been proposed to measure homophily on undirected graphs. {\em Node homophily} is defined as
\begin{equation}
h = \frac{1}{|V|} \sum_{i \in V} \frac{\sum_{j:(i,j)\in E}I[y_i = y_j]}{d_i}    
\end{equation}
where $I[y_i = y_j]$ is the indicator function with value 1 if $y_i = y_j$ or zero otherwise. Intuitively, node homophily measures the fraction of neighbors with the same label as the node itself, averaged across all nodes.
However, since heterophily is a complex phenomenon which is hard to capture with only a single scalar, a better representation of a graph's homophily is the $C\times C$  \textit{class compatibility matrix} $\mathbf{H}$~\cite{zhu2020beyond},  capturing the fraction of edges from nodes with label $k$ to nodes with label $l$: 
\begin{equation*}
h_{kl} = \frac{|(i,j) \in E : y_i=k \land y_j=l|}{|(i,j) \in E : y_i=k|}.    
\end{equation*}

\emph{Homophilic} datasets are expected to have most of the mass of their compatibility matrices concentrated in the diagonal, as most of the edges are between nodes of the same class (e.g. see Citeseer-Full in~\cref{fig:citeseer_chameleon_compat}). Conversely, \emph{heterophilic} datasets will have most of the mass away from the diagonal of the compatibility matrix (e.g. see Chameleon in~\cref{fig:citeseer_chameleon_compat}).

\subsection{Message Passing Neural Network}
In this Section, we first review the Message Passing Neural Network (MPNN) paradigm for undirected graphs, while in Sec.~\ref{sec:method} we extend this formalism to directed graphs. An MPNN is a parametric model which \emph{iteratively} applies  aggregation maps $\mathrm{AGG}^{(k)}$ and combination maps $\mathrm{COM}^{(k)}$ to compute embeddings $\mathbf{x}_{i}^{(k)}$ for node $i$ based on messages $\mathbf{m}^{(k)}_{i}$ containing information on its neighbors. Namely, the $k$-th layer of an MPNN is given by
\begin{align}
\begin{split}
\mathbf{\m}^{(\ilayer)}_{i} &=  \mathrm{AGG}^{(\ilayer)}\left(\ldblbrace (\mathbf{x}_{j}^{(\ilayer-1)}, \mathbf{x}_{i}^{(\ilayer-1)}):\, (i,j)\in E \rdblbrace\right) \\
\mathbf{x}_{i}^{(\ilayer)}   &=  \mathrm{COM}^{(\ilayer)}\left(\mathbf{x}_i^{(\ilayer-1)}, \mathbf{\m}^{(\ilayer)}_i\right) \label{eq:mpnn}
\end{split}
\end{align}
\noindent where $\ldblbrace\cdot\rdblbrace$ is a multi-set.
%{\bf [MB: Use \ldblbrace \rdblbrace\ldblbrace \rdblbrace properly defined, with smaller hspace. Use x instead of f for features of the layers]}
The aggregation maps $\mathrm{AGG}^{(k)}$ and the combination maps $\mathrm{COM}^{(k)}$ are {\em learnable} (usually a small neural network) and their different implementations  result in specific architectures (e.g.  graph convolutional neural networks (GCN) use linear aggregation, graph attention networks (GAT) use attentional layers, etc.). After the last layer $\nlayers$, the node representation $\mathbf{x}_i^{(\nlayers)}$ is mapped into the $\nclass$-probability simplex via a (learnable) decoding step, often given by an MLP. Independent of the choice of $\mathrm{AGG}$ and $\mathrm{COM}$, all MPNNs only send messages along the edges of the graph, which make them particularly suitable for tasks where edges do encode a notion of similarity -- as it is the case when adjacent nodes in the graph often share the same label (homophilic). Conversely, MPNNs tend to struggle in the scenario where they need to separate 
a node embedding from that of its neighbours \cite{nt2019revisiting}, often a challenging problem that has gained attention in the community and which we discuss in detail next.

%% file: include/03-directed_heterophily.tex
\section{Heterophily in Directed Graphs}
\label{sec:directed_heterophily}

In this Section, we discuss how accounting for directionality can be particularly helpful for dealing with heterophilic graphs. By leveraging the directionality information, we argue that even standard MPNNs that are traditionally thought to struggle in the heterophilic regime, can in fact perform extremely well. 

\textbf{Weighted homophily metrics.} First, we extend the homophily metrics introduced in Section \ref{sec:heterophily_in_gnns} to account for  directed edges and higher-order neighborhoods. 
%We start from node-wise measures. 
Given a possibly directed and weighted $n\times n$ message-passing matrix $\mathbf{S}$, we define the \textit{weighted node homophily} as
\begin{equation}
h(\mathbf{S}) = \frac{1}{|V|} \sum_{i \in V} \frac{\sum_{j \in V} s_{ij}  I[y_i = y_j]}{\sum_{j \in V} s_{ij}}
\end{equation}
\noindent Accordingly, by taking $\mathbf{S} = \mathbf{A}$ and $\mathbf{S} = \mathbf{A}^\top$ respectively, we can compute the node homophily based on outgoing or incoming edges. Similarly, we can also take $\mathbf{S}$ to be any {\em weighted} 2-hop matrix associated with a directed graph (see details below) and compute its node homophily. 

We can also extend the construction to edge-computations by defining the $C\times C$ \textit{weighted compatibility matrix} $\mathbf{H}(\mathbf{S})$ of a message-passing matrix $\mathbf{S}$ as 
\begin{equation}
h_{kl}(\mathbf{S}) = \frac{\sum_{i,j \in V : y_i=k \land y_j=l} s_{ij}}{\sum_{i,j \in V  : y_i=k} s_{ij}}
\end{equation}
\noindent As above, one can take $\mathbf{S} = \mathbf{A}$ or $\mathbf{S} = \mathbf{A}^\top$ to derive the compatibility matrix associated with the out and in-edges, respectively.

\textbf{Effective homophily.} Stacking multiple layers of a GNN effectively corresponds to taking powers of diffusion matrices, resulting in message propagation over higher-order hops. \citet{zhu2020beyond} noted that for heterophilic graphs, the 2-hop tends to be more homophilic than the 1-hop. This phenomenon of similarity within ``friends-of-friends'' has been widely observed and is commonly referred to as monophily~\cite{monophily}. If higher-order hops exhibit increased homophily, exploring the graph through layers can prove beneficial for the task. Consequently, we introduce the concept of \textit{effective homophily} as the maximum weighted node homophily observable at any hop of the graph.

For directed graphs, there exists an exponential number of $k$-hops. For instance, four 2-hop matrices can be considered: the squared operators $\mathbf{A}^2$ and $(\mathbf{A}^\top)^2$, which correspond to following the same forward or backward edge direction twice, as well as the {\em non}-squared operators $\mathbf{A}\mathbf{A}^\top$ and $\mathbf{A}^\top\mathbf{A}$, representing the forward/backward and backward/forward edge directions. Given a graph $G$, we define its effective homophily as follows:
\begin{equation}
\effhom = \max_{k \geq 1}
\max_{\mathbf{C} \in \mathcal{B}^k} h(\mathbf{C})
\end{equation}

where $\mathcal{B}^k$ denotes the set of all $k$-hop matrices for a graph. If $G$ is undirected, $\mathcal{B}^k$ contains only $\mathbf{A}^k$. In our empirical analysis, we will focus on the 2-hop matrices, as computing higher-order $k$-hop matrices becomes intractable for all but the smallest graphs~\footnote{This is attributed to the fact that while $\mathbf{A}$ is typically quite sparse, $\mathbf{A}^k$ grows increasingly dense as $k$ increases, quickly approaching $\nnodes^{2}$ non-zero entries.}. 

\input{tables/node_homophilies.tex}

\begin{figure}[t!]
\centering
\vspace{-7mm}
\begin{subfigure}[b]{0.44\textwidth}
     \centering
     \includegraphics[width=\linewidth]{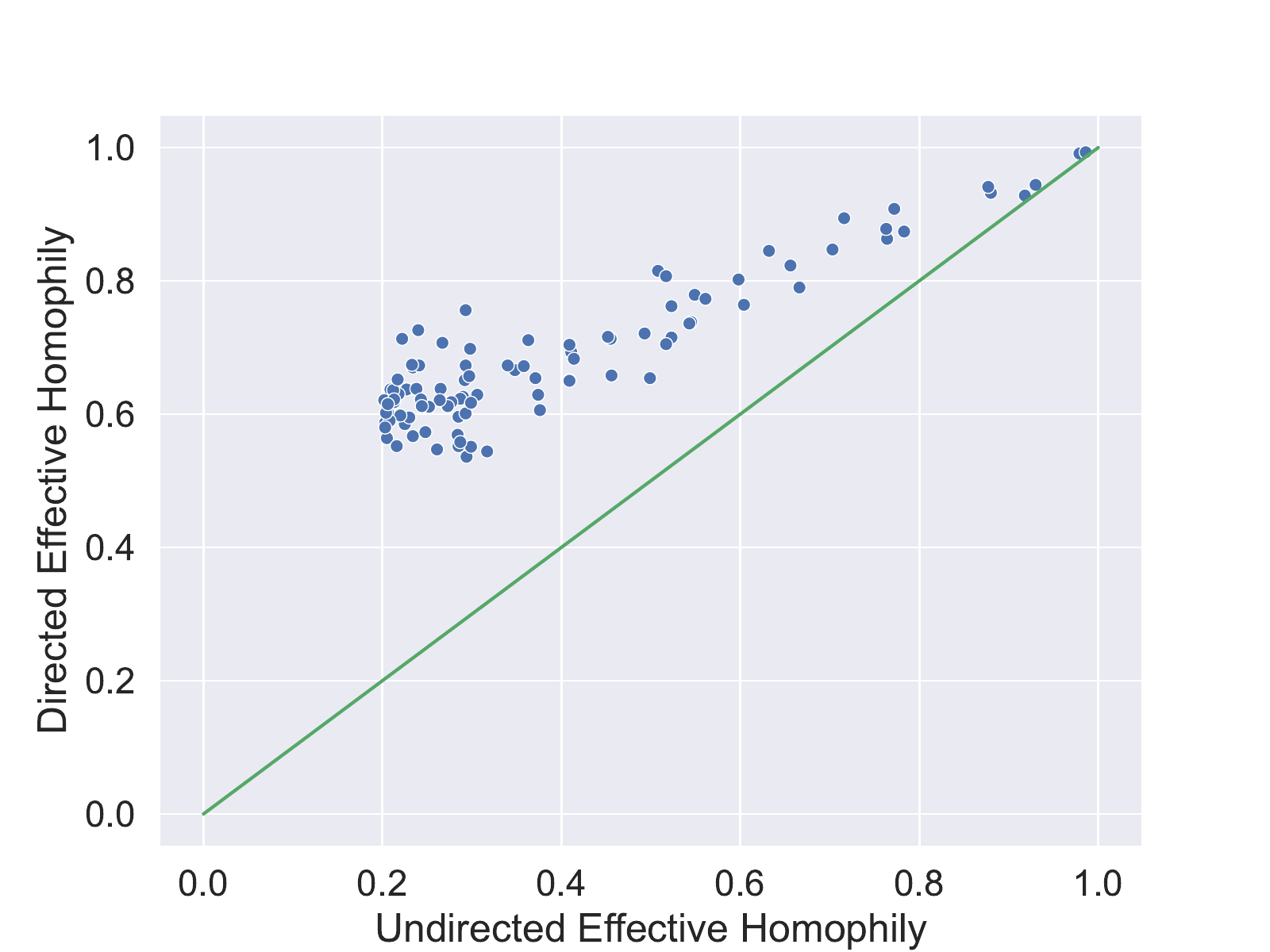}
     \caption{}
     \label{fig:synthetic_effective_homophily}
\end{subfigure}
\hspace{5mm}
\begin{subfigure}[b]{0.44\textwidth}
     \centering
     \includegraphics[width=\linewidth]{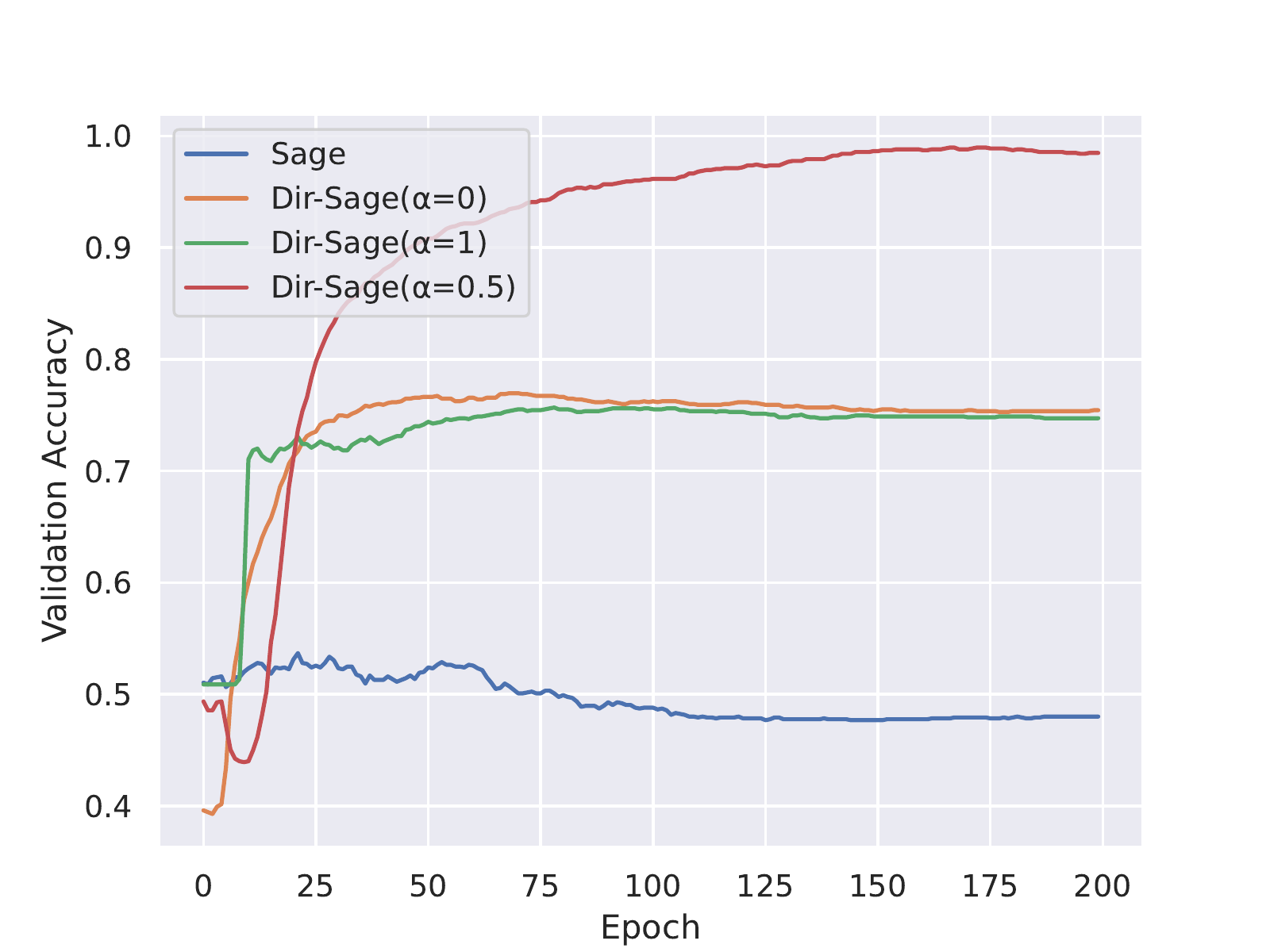}
     \caption{}
    \label{fig:synthetic_task}
\end{subfigure}

\caption{In our synthetic experiments, we observe the following: (a) the effective homophily of directed graphs is consistently higher compared to their undirected counterparts. Interestingly, this gap widens for graphs that are less homophilic. (b) When examining the performance of GraphSage and its \oursacro{} extensions on a synthetic task requiring directionality information, only Dir-Sage($\alpha$=0.5), which utilizes information from both directions, is capable of solving the task.}
\label{fig:test1}
\end{figure}

\looseness=-1
\textbf{Leveraging directionality to enhance effective homophily.}
We observe that converting a graph from directed to undirected results in lower effective homophily for heterophilic graphs, while the impact on homophilic graphs is negligible (refer to the last column of~\cref{tab:node_homophilies}). Specifically, $\mathbf{A}\mathbf{A}^\top$ and $\mathbf{A}^\top\mathbf{A}$ emerge as the most homophilic diffusion matrices for heterophilic graphs. In fact, the average relative gain of effective homophily, $\effhom_\text{gain}$, when using directed graphs compared to undirected ones is only around $3\%$ for homophilic datasets, while it is almost $30\%$ for heterophilic datasets. We further validate this observation on synthetic directed graphs exhibiting various levels of homophily, generated through a modified preferential attachment process (see \cref{sec:effective_homophily_synthetic} for more details). \cref{fig:synthetic_effective_homophily} displays the results: the directed graph consistently demonstrates higher effective homophily compared to its undirected counterpart, with the gap being particularly prominent for lower node homophily levels. The minimal effective homophily gain on homophilic datasets further substantiates the traditional practice of using undirected graphs for benchmarking GNNs, as the datasets were predominantly homophilic until recently. 
% We corroborate the empirical findings with the following theorem:

% \begin{theorem}\label{thm:homophily}
%     Let $G=(V, E)$ be a directed graph with adjacency matrix $\mathbf{A}$ and undirected adjacency matrix $\mathbf{A}_u$. We have that $max(h(\mathbf{A}), h(\mathbf{A}^\top)) \geq h(\mathbf{A}_u)$ if $G$ has no bi-directional edges.
% \end{theorem}
% A more formal statement of the theorem, along with a proof can be found in ~\cref{app:homophily_theory}. Intuitively, the theorem states that if a graph has no bi-directional edges, then we are guaranteed to observe better (or equal) homophily in its directed diffusion matrices than in the undirected one.

\textbf{Real-world example.} We illustrate the concept of effective homophily in heterophilic directed graphs by taking the concrete task of predicting the publication year of papers based on a directed citation network such as Arxiv-Year~\cite{lim2021large}. 
In this case, the different 2-hop neighbourhoods have very different semantics: the diffusion operator $(\mathbf{A}^2)_i$ represents papers that are cited by those papers that paper $i$ cites. As these 2-hop neighboring papers were published further in the past relative to paper $i$, they do not offer much information about the publication year of $i$.
On the other hand, $(\mathbf{A}^\top \mathbf{A})_i$ represents papers that share citations with paper $i$. Papers cited by the same sources are more likely to have been published in the same period, so the diffusion operator $\mathbf{A}^\top \mathbf{A}$ is expected to be more homophilic.
The undirected 2-hop operator $\mathbf{A}_u^2 = (\frac{1}{2} (\mathbf{A} + \mathbf{A}^\top))^ 2 = \frac{1}{4} (\mathbf{A}^2 + (\mathbf{A}^\top)^2 + \mathbf{A}\mathbf{A}^\top + \mathbf{A} \mathbf{A}^\top)$ is the average of the four directed 2-hops. Therefore, the highly homophilic matrix $\mathbf{A}^\top \mathbf{A}$ is diluted by the inclusion of $(\mathbf{A}^2)$, leading to a less homophilic operator overall.

% \begin{figure*}
%     \centering
%     \includegraphics[width=.8\textwidth]{images/compatibility_matrices/arxiv-year-compatibility-matrices-1-hop.pdf}
%     \caption{Weighted compatibility matrices of the undirected diffusion operator $\mathbf{A}_u$ and the two directed diffusion operators $\mathbf{A}$ and $\mathbf{A}^\top$ for Arxiv-Year. The last two have rows (classes) which are much more distinguishable then the first, despite still being heterophilic.}
%     \label{fig:arxiv-year-compat}
% \end{figure*}

%% file: tables/node_homophilies.tex
\begin{table*}[t]
\begin{center}
\begin{small}
% \begin{sc}
\resizebox{.8\textwidth}{!}{\begin{tabular}{lcccc:ccccc:r}
\toprule
           {} & {} &  $\mathbf{A}_u$ &  $\mathbf{A}_u^2$ &  $\effhom_u$ &    $\mathbf{A}$ &  $\mathbf{A}^\top$ &  $\mathbf{A}^\top \mathbf{A}$ &  $\mathbf{A}\mathbf{A}^\top$ &  $\effhom_d$ & $\effhom_\text{gain}$ \\
\midrule
\multirow{3}{5em}{Homophilic} & \textsc{citeseer-full} &  0.958 &    0.951 &  0.958 &  0.954 &  0.959 &   0.971 &   0.951 &  0.971 &         1.36\% \\
& \textsc{cora-ml}       &  0.810 &    0.767 &  0.810 &  0.808 &  0.833 &   0.803 &   0.779 &  0.833 &         2.84\% \\
& \textsc{ogbn-arxiv}    &  0.635 &    0.548 &  0.635 &  0.632 &  0.675 &   0.658 &   0.556 &  0.675 &          6.3\% \\
\hdashline
\multirow{7}{5em}{Heterophilic} 
& \textsc{chameleon}     &  0.248 &    0.331 &  0.331 &  0.249 &  0.274 &   0.383 &   0.335 &  0.383 &        15.71\% \\
& \textsc{squirrel}      &  0.218 &    0.252 &  0.252 &  0.219 &  0.210 &   0.257 &   0.258 &  0.258 &         2.38\% \\
& \textsc{arxiv-year}    &  0.289 &    0.397 &  0.397 &  0.310 &  0.403 &   0.487 &   0.431 &  0.487 &        22.67\% \\
& \textsc{snap-patents}  &  0.221 &    0.372 &  0.372 &  0.266 &  0.271 &   0.478 &   0.522 &  0.522 &        40.32\% \\
& \textsc{roman-empire}  &  0.046 &    0.365 &  0.365 &  0.045 &  0.042 &   0.535 &   0.609 &  0.609 &        66.85\% \\
\bottomrule
\end{tabular}}
% \end{sc}
\end{small}
\end{center}
\caption{Weighted node homophily for different diffusion matrices, and effective homophily for both undirected ($\effhom_u$) and directed graph ($\effhom_d$). The last column reports the gain in effective homophily obtained by using the directed graph as opposed to the undirected graph.}
\label{tab:node_homophilies}
\end{table*}

%% file: include/04-method.tex
\section{\ours{}}
\label{sec:method}

In this Section, we extend the class of MPNNs to directed graphs, and refer to such generalization as \ours{} (\oursacro{}). We follow the scheme in \Cref{eq:mpnn}, meaning that the update of a node feature is the result of a {\em combination} of its previous state and an {\em aggregation} over its neighbours. Crucially though, the characterization of neighbours now has to account for the edge-directionality. Accordingly, given a node $i\in V$, we perform separate aggregations over the in-neighbours ($j \rightarrow i$) and the out-neighbours ($i\rightarrow j$) respectively:
\begin{align}\label{eq:directed-mpnn}
\begin{split}
    \mathbf{\m}^{(k)}_{i,\inp} &= \mathrm{AGG}^{(k)}_{\inp}\left(\ldblbrace (\mathbf{x}_{j}^{(\ilayer-1)}, \mathbf{x}_{i}^{(\ilayer-1)}) : \, (j,i)\in E\rdblbrace \right) \\  
    \mathbf{\m}^{(k)}_{i,\out} &= \mathrm{AGG}^{(k)}_{\out}\left(\ldblbrace (\mathbf{x}_{j}^{(\ilayer-1)}, \mathbf{x}_{i}^{(\ilayer-1)}) : \, (i,j)\in E \rdblbrace  \right) \\
    \mathbf{x}_{i}^{(k)} &= \mathrm{COM}^{(k)}\left(\mathbf{x}_{i}^{(k-1)},\mathbf{\m}^{(k)}_{i,\inp}, \mathbf{\m}^{(k)}_{i,\out}\right).
\end{split}
\end{align}
\noindent The idea behind our framework is that given {\em any} MPNN, we can readily adapt it to the directed case by choosing how to aggregate over both directions. For this purpose, we replace the neighbour aggregator $\mathrm{AGG}^{(k)}$ with two separate in- and out-aggregators $\mathrm{AGG}^{(k)}_{\inp}$ and $\mathrm{AGG}^{(k)}_{\out}$, which can have {\em independent} sets of parameters for the two aggregation phases and, possibly, different normalizations as discussed below -- however the functional expression in both cases stays the same. As we will show, accounting for both directions separately is fundamental for both the expressivity of the model (see \cref{sec:expressivity}) and its empirical performance (see \Cref{sec:synthetic_experiments}).

\textbf{Extension of common architectures.}
To make our discussion more explicit, we describe extensions of popular MPNNs where the aggregation map is computed by \smash{$\mathbf{\m}^{(k)}_i = (\mathbf{S}\mathbf{x}^{(k-1)})_i$}, where $\mathbf{S}\in\R^{n\times n}$ is a message-passing matrix. In GCN~\cite{kipf2016semi}, \smash{$\mathbf{S} = \mathbf{D}^{-1/2}\mathbf{A}_u\mathbf{D}^{-1/2}$}, where $\mathbf{D}$ is the degree matrix of the undirected graph. In the case of directed graphs, two message-passing matrices $\mathbf{S}_\inp$ and $\mathbf{S}_\out$ are required for in- and out-neighbours respectively. Additionally, the normalization is slightly more subtle, since we now have two different diagonal degree matrices $\mathbf{D}_{\inp}$ and $\mathbf{D}_{\out}$ containing the in-degrees and out-degrees respectively. Accordingly, we propose a normalization of the form $\mathbf{S}_{\out} = \mathbf{D}_{\out}^{-1/2}\mathbf{A}\mathbf{D}_{\inp}^{-1/2}$ i.e. $(\mathbf{S}_{\out})_{ij} = a_{ij}/\sqrt{d^{\out}_{i}d^{\inp}_j}$. To motivate this choice, note that the normalization modulates the aggregation based on the out-degree of $i$ and the in-degree of $j$ as one would expect given that we are computing a message going from $i$ to $j$. We can then take $\mathbf{S}_{\inp} = \mathbf{S}_{\out}^\top$ and write the update at layer $k$ of Dir-GCN as 
\begin{equation}\label{eq:dir-gcn}
    \mathbf{X}^{(k)} =  \sigma\left(\mathbf{S}_{\out}\mathbf{X}^{(k-1)}\mathbf{W}^{(k)}_{\out} + \mathbf{S}^\top_{\out}\mathbf{X}^{(k-1)}\mathbf{W}^{(k)}_{\inp}\right), 
\end{equation}
\noindent for learnable channel-mixing matrices $\mathbf{W}^{(k)}_{\out}, \mathbf{W}^{(k)}_{\inp}$ and with $\sigma$ a pointwise activation map. Finally, we note that in our implementation of \oursacro{} we use an additional learnable or tunable parameter $\alpha$ allowing the framework to weight one direction more than the other (a convex combination), depending on the dataset. \oursacro{} extensions of GAT~\cite{velivckovic2017graph} and GraphSAGE~\cite{hamilton2017inductive} can be found in~\cref{app:extensions}. 

\textbf{\oursacro{} leads to more homophilic aggregations.} Since our main application amounts to relying on the graph-directionality to mitigate heterophily, here we comment on how information is iteratively propagated in a \oursacro{}, generally leading to an aggregation scheme beneficial on most real-world directed heterophilic graphs. We focus on the Dir-GCN formulation, however the following applies to any other \oursacro{} up to changing the message-passing matrices. Consider a 2-layer Dir-GCN as in \Cref{eq:dir-gcn}, and let us remove the pointwise activation $\sigma$~\footnote{Note that this does not affect our discussion, in fact any observation can be extended to the non-linear case by computing the Jacobian of node features as in \citet{topping2021understanding}.}. Then, the node representation can be written as 
\begin{equation}
\resizebox{1\hsize}{!}{$\mathbf{X}^{(2)} = \mathbf{A}_{\out}^{2}\mathbf{X}^{(0)}\mathbf{W}^{(1)}_{\out}\mathbf{W}^{(2)}_{\out} + (\mathbf{A}_{\out}^\top)^{2}\mathbf{X}^{(0)}\mathbf{W}^{(1)}_{\inp}\mathbf{W}^{(2)}_{\inp} \notag + \mathbf{A}_{\out}\mathbf{A}_{\out}^\top\mathbf{X}^{(0)}\mathbf{W}^{(1)}_{\inp}\mathbf{W}^{(2)}_{\out} + \mathbf{A}^{\top}_{\out}\mathbf{A}_{\out}\mathbf{X}^{(0)}\mathbf{W}^{(1)}_{\out}\mathbf{W}^{(2)}_{\inp}.$}
\end{equation}
\noindent We observe that when we aggregate information over multiple layers, % -- hence exploring higher-order hops of the graph -- due to the directionality of our framework, 
the final node representation is derived by also computing convolutions over 2-hop matrices $\mathbf{A}_{\out}^\top\mathbf{A}_{\out}$ and $\mathbf{A}_{\out}\mathbf{A}^{\top}_{\out}$. From the discussion in \Cref{sec:directed_heterophily}, we deduce that this framework may be more suited to handle heterophilic graphs since generally such 2-hop matrices are more likely to encode similarity than $\mathbf{A}_{\out}^{2}$, $(\mathbf{A}_{\out}^{\top})^{2}$ or $\mathbf{A}_u^{2}$ -- this is validated empirically on real-world datasets in \Cref{sec:experiments}.

\textbf{Advantages of two-directional updates.} We discuss the benefits of incorporating both directions in the layer update, as opposed to using a single direction. Although spatial MPNNs can be adapted to directed graphs by simply utilizing $\mathbf{A}$ instead of $\mathbf{A}_u$—resulting in message propagation only along out-edges—relying on a single direction presents three primary drawbacks. 
First, if the layer update only considers one direction, the exploration of the multi-hop neighbourhoods through powers of diffusion operators would not include the mixed terms $\mathbf{A}\mathbf{A}^\top$ and $\mathbf{A}^\top\mathbf{A}$, which have been shown to be particularly beneficial for heterophilic graphs in \Cref{sec:directed_heterophily}.
Second, by using only one direction we disregard the graph entirely for nodes where the out-degree is zero~\footnote{Or in-degree, depending on which direction is selected.}. This phenomenon frequently occurs in real-world graphs, as reported in~\cref{tab:zero_degrees}. Incorporating both directions in the layer update helps mitigate this problem, as it is far less common for a node to have both in- and out-degree to be zero, as also illustrated in~\cref{tab:zero_degrees}.
Third, limiting the update to a single direction reduces expressivity, as we discuss in~\cref{sec:expressivity}.

\textbf{Complexity.}
The complexity of \oursacro{} depends on the specific instantiation of the framework. Dir-GCN, Dir-Sage, and Dir-GAT maintain the same per-layer computational complexity as their undirected counterparts ($\mathcal{O}(md + nd^2)$ for GCN and GraphSage, and $\mathcal{O}(md^2)$ for GAT). However, they have twice as many parameters, owing to their separate weight matrices for in- and out-neighbors.

\subsection{Expressive power of \oursacro{}} \label{sec:expressivity}
It is a well known result that MPNNs are bound in expressivity by the 1-WL test, and that it is possible to construct MPNN models which are as expressive as the 1-WL test~\cite{DBLP:conf/iclr/XuHLJ19}. In this section, we show that \oursacro{} is the optimal way to extend MPNNs to directed graphs. We do so by proving that \oursacro{} models can be constructed to be as expressive as an extension of the 1-WL test to directed graphs~\cite{Grohe2021ColorRA}, referred to as \textit{D-WL} (for a formal definition, see~\cref{sec:appendix_wl}). Additionally, we illustrate its greater expressivity over more straightforward approaches, such as converting the graph to its undirected form and utilizing a standard MPNN (\textit{MPNN-U}) or applying an MPNN that propagates solely along edge direction (\textit{MPNN-D})\footnote{The same results apply to a model which sends messages only along in-edges.}. Formal statements for the theorems in this section along with their proofs can be found in~\cref{app:expressivity_analysis}. 
% 
% \begin{definition}[Model families] \label{def:model_families}
%     Let $\mathcal{M}_\mathrm{MPNN-D}$ be the family of Message Passing Neural Networks (\cref{eq:mpnn}), $\mathcal{M}_\mathrm{MPNN-U}$ that of Message Passing Neural Networks on the undirected form of the graph, and $\mathcal{M}_\mathrm{\oursacro{}}$ that of \oursacro{} models (\cref{eq:directed-mpnn}).
% \end{definition}
% 
% Given a directed graph, a model from $\mathcal{M}_\mathrm{MPNN-D}$ sends messages only along out-edges,\footnote{The same results apply to a model which sends messages only along in-edges.} while a model from $\mathcal{M}_\mathrm{MPNN-U}$ sends messages in both directions without distinction.

\begin{theorem}[Informal]\label{thm:dirgnn-as-expressive-as-d-wl}
    \oursacro{} is as expressive as D-WL if $\mathrm{AGG}^{(k)}_{\out}$, $\mathrm{AGG}^{(k)}_{\inp}$, and $\mathrm{COM}^{(k)}$ are injective for all $k$. 
\end{theorem}

A discussion of how a \oursacro{} can be parametrized to meet these conditions (similarly to what is done in~\citet{DBLP:conf/iclr/XuHLJ19}) can be found in~\cref{app:d-wl-proof}.

\begin{theorem}[Informal]\label{thm:dirgnn-strictly-more-expressive-than-mpnn}
\oursacro{} is strictly more expressive than both MPNN-U and MPNN-D.
\end{theorem}

Intuitively, the theorem states that while all directed graphs distinguished by MPNNs are also separated by \oursacro{}s, there also exist directed graphs separated by the latter but not by the former. This holds true for  MPNNs applied both on the directed and undirected graph.
We observe these theoretical findings to be in line with the empirical results detailed in~\cref{fig:synthetic_gcn_gat} and~\cref{tab:full_direction_ablation}, where \oursacro{} performs comparably or better (typically in the case of heterophily) than MPNNs.

%% file: include/05-related_works.tex
\section{Related Work}
\label{sec:related_work}

\textbf{GNNs for directed graphs.}
While several classical papers have alluded to the extension of their spatial models to directed graphs, empirical validation has not been conducted~\cite{Scarselli:2009ku,li2015gated,gilmer2017neural}. GatedGCN~\cite{DBLP:journals/corr/LiTBZ15} deals with directed graphs, however it aggregates information only from out-neighbors, neglecting potentially valuable information from in-neighbors. More recently,~\citet{vrvcek2022learning} tackle the genome assembly problem by employing a GatedGCN with separate aggregations for in- and out-neighbors.
Various approaches have been developed to generalize spectral convolutions for directed graphs~\cite{spectral-dgcn,motifnet}. Of particular interest are DGCN~\cite{dgcn}, which leverages $\mathbf{A}^\top\mathbf{A}$ and $\mathbf{A}\mathbf{A}^\top$ for its convolution (see~\cref{sec:comparison_with_directed_graph_methods} for a more detailed comparison with Dir-GNN), DiGCN~\cite{digraph}, which uses Personalized Page Rank matrix as a generalized Laplacian and incorporates $k$-hop diffusion matrices, and MagNet~\cite{zhang2021magnet}, which adopts a complex matrix for graph diffusion where the real and imaginary parts represent the undirected adjacency and edge direction respectively. 
The above spectral methods share the following limitations: 1) in-neighbors and out-neighbors share the same weight matrix, which restricts expressivity; 2) they are specialized models, often inspired by GCN, as opposed to broader frameworks; 3) their scalability is severely limited due to their spectral nature. Concurrently to our work, \citet{pmlr-v202-geisler23a} extend transformers to directed graph for the task of graph classification, while~\citet{maskey2023fractional} generalize the concept of oversmoothing to directed graphs. 

\textbf{GNNs for relational graphs.} While counter intuitive at first, a directed graph cannot be equivalently represented by an \textit{undirected} relational graph (see~\cref{sec:alternative_representations} for more details). However, our Dir-GCN model can be considered as a Relational Graph Convolutional Network (R-GCN)~\cite{10.1007/978-3-319-93417-4_38} applied to an augmented \textit{directed} relational graph that incorporates two relation types: one for the original edges and another for the inverse edges added to the graph. Several papers handle multi-relational directed graphs by adding inverse relations~\cite{10.1007/978-3-319-93417-4_38,marcheggiani-titov-2017-encoding,DBLP:journals/corr/abs-1904-08745,Vashishth2020Composition-based}. 
Similarly to the above, directionality can be addressed using an MPNN~\cite{gilmer2017neural} combined with binary edge features, although at the cost of increased memory usage (see~\cref{sec:mpnn_binary_features} for more details).
In our work, however, we are the first to perform an in-depth investigation of the role of directionality in graph learning and its relation with homophily of the graph.

\textbf{Heterophilic GNNs.}
% Heterophilic graphs typically represent a more challenging setting for node classification tasks and in fact, classical GNNs that act as low-pass filters \cite{nt2019revisiting} have been shown to struggle on these \cite{pei2020geom}. 
Several GNN architectures have been proposed to handle heterophily. One way amounts to effectively allow the model to enhance the high-frequency components by learning `generalized' negative weights on the graph \cite{chien2020adaptive,bo2021beyond,luan2022revisiting,bodnar2022neural,di2022graph}.  
A different approach tries to enlarge the neighbourhood aggregation to take advantage of the fact that on heterophilic graphs, the likelihood of finding similar nodes increases beyond the 1-hop \cite{abu-el-haijaMixHopHigherOrderGraph2019,zhu2020beyond,lim2021large,maurya2021improving,li2022finding}.

%% file: include/06-experiments.tex
\section{Experiments} \label{sec:experiments}

\subsection{Synthetic Task} \label{sec:synthetic_experiments}

\textbf{Setup.} In order to show the limits of current MPNNs, we design a synthetic task where the label of a node depends on both its in- and out-neighbors: it is one if the mean of the scalar features of their in-neighbors is greater than the mean of the features of their out-neighbors, or zero otherwise (more details in~\cref{sec:appendix_synthetic_experiment}). We report the results using GraphSage as base MPNN, but similar results were obtained with GCN and GAT and reported in~\cref{fig:synthetic-gcn-gat} of the Appendix. We compare GraphSage on the undirected version of the graph (Sage), with three \oursacro{} extensions of GraphSage using different convex combination coefficients $\alpha$: Dir-Sage($\alpha=0$) (only considering in-edges), Dir-Sage($\alpha=1$) (only considering out-edges) and Dir-Sage($\alpha=0.5$) (considering both in- and out-edges equally).

\textbf{Results.} The results in~\cref{fig:synthetic_task} show that only \textit{Dir-Sage($\alpha$=0.5)}, which accounts for both directions, is able to almost perfectly solve the task. Using only in- or out-edges results in around 75\% accuracy, whereas GraphSage on the undirected graph is no better than a random classifier. 

\input{tables/direction_ablation.tex}

\subsection{Extending Popular GNNs with \oursacro{}}
\label{sec:experiments_extending_gnn}

\textbf{Datasets.} We evaluate on the task of node classification on several directed benchmark datasets with varying levels of homophily: Citesee-Full, Cora-ML~\cite{bojchevski2018deep}, OGBN-Arxiv~\cite{hu2020ogb}, Chameleon, Squirrel~\cite{pei2020geom}, Arxiv-Year, Snap-Patents~\cite{lim2021large} and Roman-Empire~\cite{platonov2023a} (refer to~\cref{tab:datasets-statistics} for dataset statistics). While the first three are mainly homophilic (edge homophily greater than 0.65), the last five are highly heterophilic (edge homophily smaller than 0.24). Refer to~\cref{app:experimental_setup} for more details on the experimental setup and on dataset splits.

\textbf{Setup.} We evaluate the gain of extending popular undirected GNN architectures (GCN~\cite{kipf2016semi}, GraphSage~\cite{hamilton2017inductive} and GAT~\cite{velivckovic2017graph}) with our framework. For this ablation, we use the same hyperparameters (provided in~\cref{sec:appendix_ablation_hyperparams}) for all models and datasets. The aggregated results are plotted in~\cref{fig:aggregate_results}, while the raw numbers are reported in~\cref{tab:direction_ablation}. For \oursacro{}, we take the best results out of $\alpha \in \{0, 0.5, 1\}$ (see~\cref{tab:full_direction_ablation} for the full results).

\textbf{Results.} We report aggregated results in~\cref{fig:aggregate_results}, while~\cref{tab:direction_ablation} shows the results for each dataset. On \textbf{heterophilic datasets}, using directionality brings exceptionally large gains (10\% to 20\% absolute) in accuracy across all three base GNN models. On the other hand, on \textbf{homophilic datasets} using directionality leaves the performance unchanged or slightly hurts. This is in line with the findings of Tab.~\ref{tab:node_homophilies}, which shows that using directionality as in our framework generally increases the effective homophily of heterophilic datasets, while leaving it almost unchanged for homophilic datasets. 
The inductive bias of undirected GNNs to propagate information in the \emph{same} way in both directions is beneficial on homophilic datasets where edges encode a notion of class similarity. Moreover, averaging information across {\em all} your neighbors, independent of direction, leads to a low-pass filtering effect that is indeed beneficial on homophilic graphs \cite{nt2019revisiting}. In contrast, \oursacro{} has to learn to align in- and out-convolutions since they have independent weights.

%\todo[inline]{Improve this: alignenment of weights and low pass filtering}

\subsection{Comparison with State-of-the-Art Models}
\label{sec:experiments_sota}

\input{tables/heterophilic_results.tex}

\textbf{Setup.} Given the importance of directionality on heterophilic tasks, we compare Dir-GNN with state-of-the-art models on heterophilic benchmarks Chameleon, Squirrel~\cite{pei2020geom}, Arxiv-Year, Snap-Patents~\cite{lim2021large} and Roman-Empire~\cite{platonov2023a}. In particular, we compare to \textbf{simple baselines}: MLP and GCN~\cite{kipf2016semi}, \textbf{heterophilic state-of-the-art models}: H$_2$GCN~\cite{zhu2020beyond}, GPR-GNN~\cite{chien2020adaptive}, LINKX~\cite{lim2021large}, FSGNN~\cite{maurya2021improving}, ACM-GCN~\cite{luan2022revisiting}, GloGNN~\cite{li2022finding}, Gradient Gating~\cite{rusch2022gradient}, and \textbf{state-of-the-art models for directed graphs}: DiGCN~\cite{digraph} and MagNet~\cite{zhang2021magnet}.~\cref{sec:appendix_baseline_results} contains more details on how baseline results were obtained. Differently from the results in~\cref{tab:direction_ablation}, we now tune the hyperparameters of our model using a grid search (see~\cref{sec:appendix_grid_search} for the exact ranges).

\textbf{Results.} In~\cref{tab:heterophilic_results} we observe that \oursacro{} obtains {\bf new state-of-the-art} results on all five heterophilic datasets, outperforming complex methods which were specifically designed to tackle heterophily. These results suggest that, when present, {\em using the edge direction can significantly improve learning on heterophilic graphs}, justifying the title of the paper. In contrast, discarding it is so harmful that not even complex architectures can make up for this loss of information. We further note that DiGCN and MagNet, despite being specifically designed for directed graphs, struggle on Squirrel and Chameleon. This is due to their inability to selectively aggregate from one direction while disregarding the other, a strategy that proves particularly advantageous for these two datasets (see~\cref{tab:full_direction_ablation}). Our proposed \oursacro{} framework overcomes this limitation thanks to its distinct weight matrices and the flexibility provided by the $\alpha$ parameter, enabling selective directional aggregation.

%% file: tables/direction_ablation.tex
\begin{table*}[t]
% \vspace{-1mm}
\begin{center}
\begin{small}
\begin{sc}
\resizebox{1.\textwidth}{!}{%
\begin{tabular}{lccc:ccccc}
\toprule
{} & \multicolumn{3}{c:}{Homophilic} & \multicolumn{5}{c}{Heterophilic} \\
{} & citeseer\_full & cora\_ml & ogbn-arxiv & chameleon & squirrel & arxiv-year & snap-patents & roman-empire \\ 
hom.      & 0.949  &  0.792  &  0.655 &  0.235  &  0.223 &  0.221   & 0.218   & 0.05    \\
hom. gain & 1.36\% &  2.84\% & 6.30\% & 15.71\% & 2.38\% &  22.67\% & 40.32\% & 66.85\% \\
\midrule
gcn      &           93.37 $\pm$ 0.22 &           84.37 $\pm$ 1.52 &  \textbf{68.39 $\pm$ 0.01} &           71.12 $\pm$ 2.28 &           62.71 $\pm$ 2.27 &           46.28 $\pm$ 0.39 &           51.02 $\pm$ 0.07 &           56.23 $\pm$ 0.37 \\
dir-gcn  &  \textbf{93.44 $\pm$ 0.59} &  \textbf{84.45 $\pm$ 1.69} &           66.66 $\pm$ 0.02 &  \textbf{78.77 $\pm$ 1.72} &  \textbf{74.43 $\pm$ 0.74} &  \textbf{59.56 $\pm$ 0.16} &  \textbf{71.32 $\pm$ 0.06} &  \textbf{74.54 $\pm$ 0.71} \\
\hdashline
sage     &  \textbf{94.15 $\pm$ 0.61} &  \textbf{86.01 $\pm$ 1.56} &  \textbf{67.78 $\pm$ 0.07} &           61.14 $\pm$ 2.00 &           42.64 $\pm$ 1.72 &           44.05 $\pm$ 0.02 &           52.55 $\pm$ 0.10 &           72.05 $\pm$ 0.41 \\
dir-sage &           94.14 $\pm$ 0.65 &           85.84 $\pm$ 2.09 &           65.14 $\pm$ 0.03 &  \textbf{64.47 $\pm$ 2.27} &  \textbf{46.05 $\pm$ 1.16} &  \textbf{55.76 $\pm$ 0.10} &  \textbf{70.26 $\pm$ 0.14} &  \textbf{79.10 $\pm$ 0.19} \\
\hdashline
gat      &  \textbf{94.53 $\pm$ 0.48} &  \textbf{86.44 $\pm$ 1.45} &  \textbf{69.60 $\pm$ 0.01} &           66.82 $\pm$ 2.56 &           56.49 $\pm$ 1.73 &           45.30 $\pm$ 0.23 &                   OOM &           49.18 $\pm$ 1.35 \\
dir-gat  &           94.48 $\pm$ 0.52 &           86.21 $\pm$ 1.40 &           66.50 $\pm$ 0.16 &  \textbf{71.40 $\pm$ 1.63} &  \textbf{67.53 $\pm$ 1.04} &  \textbf{54.47 $\pm$ 0.14} &                   OOM &  \textbf{72.25 $\pm$ 0.04} \\
\bottomrule
\end{tabular}%
}
\end{sc}
\end{small}
\end{center}
\caption{Ablation study comparing base MPNNs on the undirected graphs versus their \oursacro{} extension on the directed graphs. Homophilic datasets, located to the left of the dashed line, show little to no improvement when incorporating directionality, sometimes even experiencing a minor decrease in performance. Conversely, heterophilic datasets, found to the right of the dashed line, demonstrate large accuracy improvements when directionality is incorporated into the model.}
\label{tab:direction_ablation}
\end{table*}

%% file: tables/heterophilic_results.tex
\begin{table*}[t]
\begin{center}
\begin{small}
\begin{sc}
\resizebox{.7\textwidth}{!}{%
\begin{tabular}{lcccccr}
\toprule
             &  Squirrel                   &  Chameleon                  &  Arxiv-year                 &  Snap-patents               &  Roman-Empire\\
\midrule
MLP          &  28.77 $\pm$ 1.56           &  46.21 $\pm$ 2.99           &  36.70 $\pm$ 0.21           &  31.34 $\pm$ 0.05           &  64.94 $\pm$ 0.62          \\
GCN          &  53.43 $\pm$ 2.01           &  64.82 $\pm$ 2.24           &  46.02 $\pm$ 0.26           &  51.02 $\pm$ 0.06           &  73.69 $\pm$ 0.74          \\
\hdashline
H$_2$GCN     &  37.90 $\pm$ 2.02           &  59.39 $\pm$ 1.98           &  49.09 $\pm$ 0.10           &  OOM                        &  60.11 $\pm$ 0.52          \\
GPR-GNN      &  54.35 $\pm$ 0.87           &  62.85 $\pm$ 2.90           &  45.07 $\pm$ 0.21           &  40.19 $\pm$ 0.03           &  64.85 $\pm$ 0.27          \\
LINKX        &  61.81 $\pm$ 1.80           &  68.42 $\pm$ 1.38           &  56.00 $\pm$ 0.17           &  61.95 $\pm$ 0.12           &  37.55 $\pm$ 0.36          \\
FSGNN        &  74.10 $\pm$ 1.89           &  78.27 $\pm$ 1.28           &  50.47 $\pm$ 0.21           &  65.07 $\pm$ 0.03           &  79.92 $\pm$ 0.56          \\
ACM-GCN      &  67.40 $\pm$ 2.21           &  74.76 $\pm$ 2.20           &  47.37 $\pm$ 0.59           &  55.14 $\pm$ 0.16           &  69.66 $\pm$ 0.62          \\
GloGNN       &  57.88 $\pm$ 1.76           &  71.21 $\pm$ 1.84           &  54.79 $\pm$ 0.25           &  62.09 $\pm$ 0.27           &  59.63 $\pm$ 0.69          \\
Grad. Gating &  64.26 $\pm$ 2.38           &  71.40 $\pm$ 2.38           &  63.30 $\pm$ 1.84           &  69.50 $\pm$ 0.39           &  82.16 $\pm$ 0.78          \\
\hdashline
DiGCN        &  37.74 $\pm$ 1.54           &  52.24 $\pm$ 3.65           &  OOM                        &  OOM                        &  52.71 $\pm$ 0.32          \\ 
MagNet       &  39.01 $\pm$ 1.93           &  58.22 $\pm$ 2.87           &  60.29 $\pm$ 0.27           &  OOM                        &  88.07 $\pm$ 0.27          \\
\hdashline
\oursacro{}  &  \textbf{75.31 $\pm$ 1.92}  &  \textbf{79.71 $\pm$ 1.26}  &  \textbf{64.08 $\pm$ 0.26}  &  \textbf{73.95 $\pm$ 0.05}  &  \textbf{91.23 $\pm$ 0.32} \\
\bottomrule
\end{tabular}%
}
\end{sc}
\end{small}
\end{center}
\caption{Results on real-world directed heterophilic datasets. OOM indicates out of memory. }
\label{tab:heterophilic_results}
\end{table*}

%% file: include/07-conclusion.tex
\section{Conclusion}
\label{sec:conclusion}

\looseness=-1
We introduced \oursacro, a generic framework to extend any spatial graph neural network to directed graphs, which we prove to be strictly more expressive than MPNNs. We showed that treating the graph as directed improves the effective homophily of heterophilic datasets, and validated empirically that augmenting popular GNN architectures with our framework results in large improvements on heterophilic benchmarks, while leaving performance almost unchanged on homophilic benchmarks. Surprisingly, we found simple instantiations of our framework to obtain state-of-the-art results on the five directed heterophilic benchmarks we experimented on, outperforming recent architectures developed specifically for heterophilic settings as well as previously proposed methods for directed graphs.

\textbf{Limitations.} Our research has several areas that could be further refined and explored. First, the theoretical exploration of the conditions that lead to a higher effective homophily in directed graphs compared to their undirected counterparts is still largely unexplored. Furthermore, we have yet to investigate the expressivity advantage of \oursacro{} in the specific context of heterophilic graphs, where empirical gains were most pronounced.  Finally, we haven't empirically investigated different functional forms for aggregating in- and out-edges. These aspects mark potential areas for future enhancements and investigations. 

%% file: include/08-appendix.tex
\crefalias{section}{appsec}
\crefalias{subsection}{appsec}
\crefalias{subsubsection}{appsec}
\newpage

\begin{figure}
    \centering
    \includegraphics[width=0.65\textwidth]{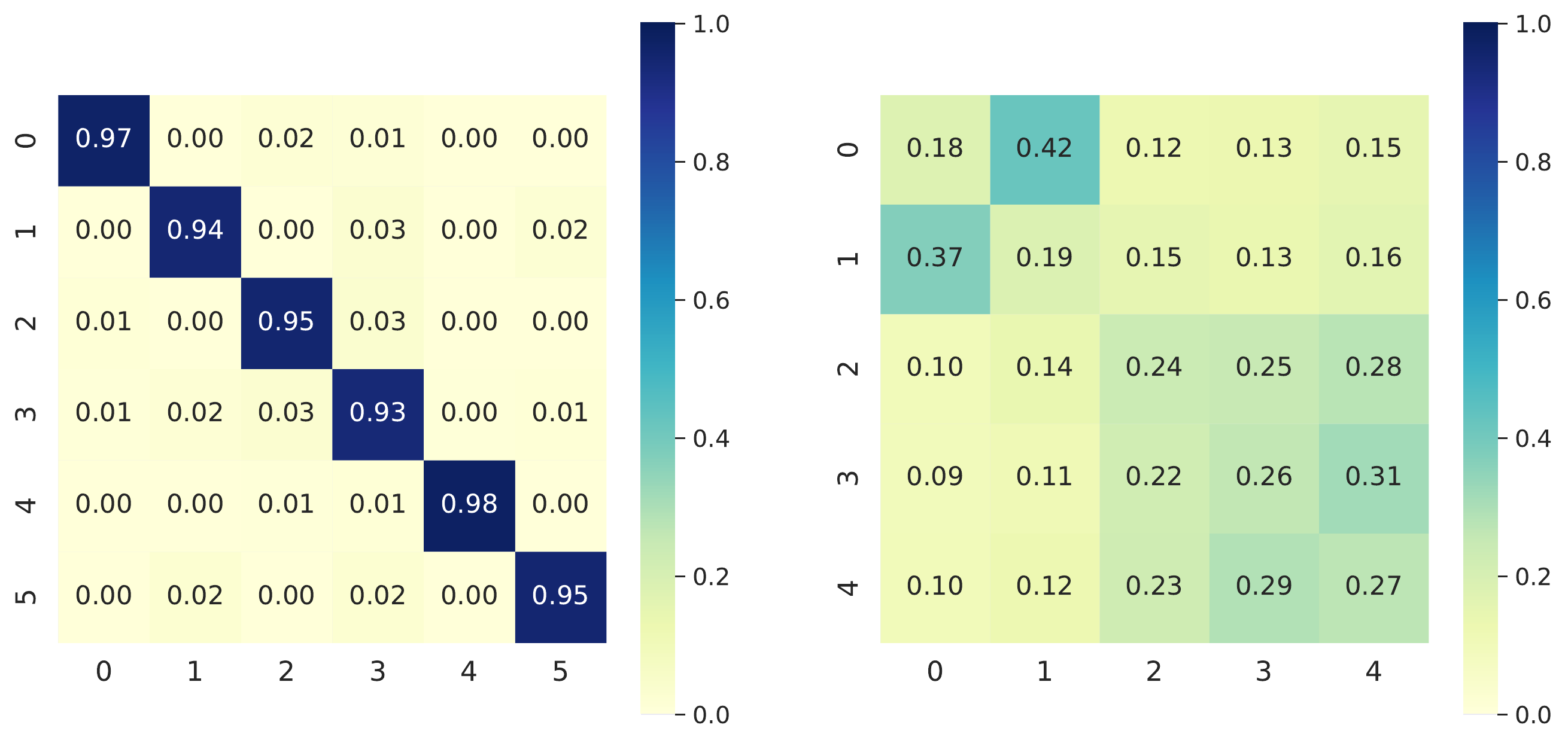}
    \caption{Compatibility matrices for the undirected version of Citeseer-Full (homophilic, left) and Chameleon (heterophilic, right).}
    \label{fig:citeseer_chameleon_compat}
\end{figure}

\begin{figure*}
    \centering
    \includegraphics[width=.8\textwidth]{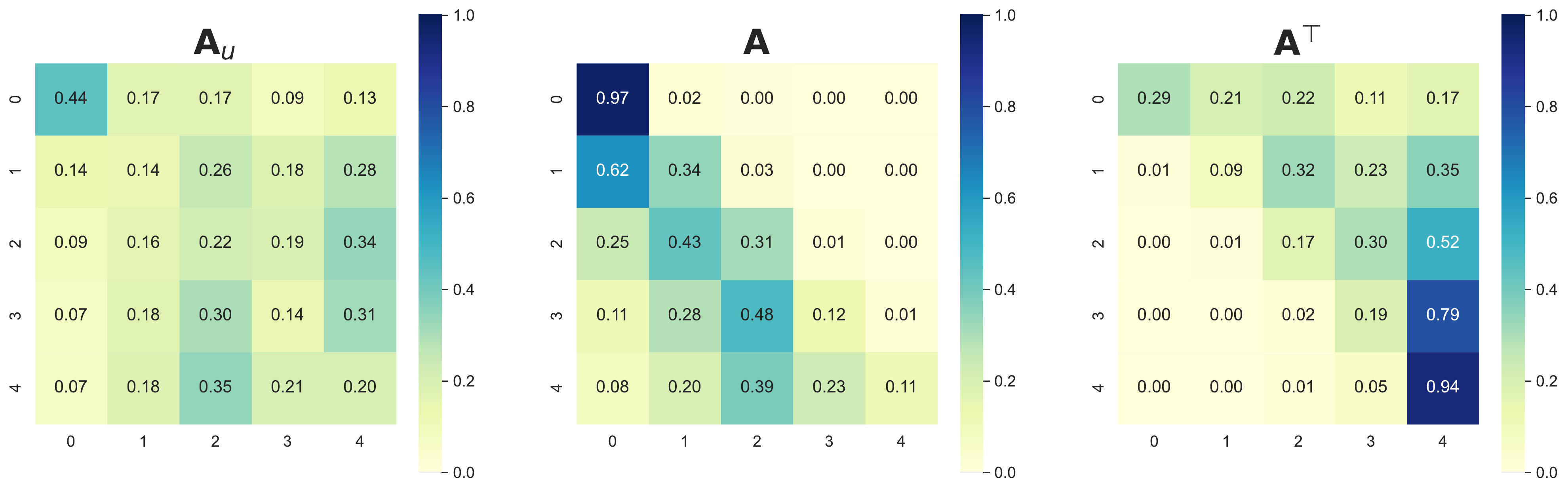}
    \caption{Weighted compatibility matrices of the undirected diffusion operator $\mathbf{A}_u$ and the two directed diffusion operators $\mathbf{A}$ and $\mathbf{A}^\top$ for Arxiv-Year. The last two have rows (classes) which are much more distinguishable then the first, despite still being heterophilic.}
    \label{fig:arxiv-year-compat}
\end{figure*}

\section{Compatibility Matrices}
\cref{fig:citeseer_chameleon_compat} shows the compatibility matrices for both Citeseer-Full (homophilic) and Chameleon (heterophilic). Additionally, ~\cref{fig:arxiv-year-compat} presents the weighted compatibility matrices of the undirected diffusion operator $\mathbf{A}_u$ and the two directed diffusion operators $\mathbf{A}$ and $\mathbf{A}^\top$ for Arxiv-Year. The last two have rows (classes) which are much more distinguishable then the first, despite still being heterophilic. This phenomen, called harmless heterophily, is discussed in~\cref{sec:directed_heterophily}.

\section{Harmless Heterophily Through Directions}
It has been recently shown that heterophily is not necessarily harmful for GNNs, as long as nodes with the same label share similar neighborhood patterns, and different classes have distinguishable patterns~\cite{ma2022is, luan2023graph} . We find that some directed datasets, such as Arxiv-Year and Snap-Patents, show this form of \textit{harmless heterophily} when treated as directed, and instead manifest \textit{harmful heterophily} when made undirected (see~\cref{fig:arxiv-year-compat} in the Appendix). This suggests that using directionality can be beneficial also when using only one layer, as we confirm empirically (see~\cref{fig:arxiv1-layer} in the Appendix).

\textbf{Toy example.} We further illustrate the concepts presented in this Section with the toy example in Fig.~\ref{fig:synthetic_example}, which shows a directed graph with three classes (blue, orange, green). Despite the graph being maximally heterophilic, it presents harmless heterophily, since different classes have very different neighborhood patterns that can be observed clearly from the compatibility matrix of $\mathbf{A}$ (b). When the graph is made undirected (c), we are corrupting this information, and the classes become less distinguishable, making the task harder. We also note that both $\mathbf{A}^\top\mathbf{A}$ and $\mathbf{A}\mathbf{A}^\top$ presents perfect homophily (d), while $\mathbf{A}_u^2$ does not, in line with the discussion in previous paragraphs.

\begin{figure}[t!]
\vspace{-4mm}
\centering
\begin{subfigure}[b]{0.19\textwidth}
     \centering
     \includegraphics[width=\linewidth]{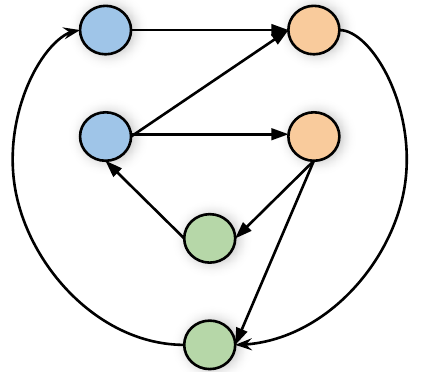}
     \caption{}
\end{subfigure}
\hfill
\begin{subfigure}[b]{0.19\textwidth}
     \centering
     \includegraphics[width=\linewidth]{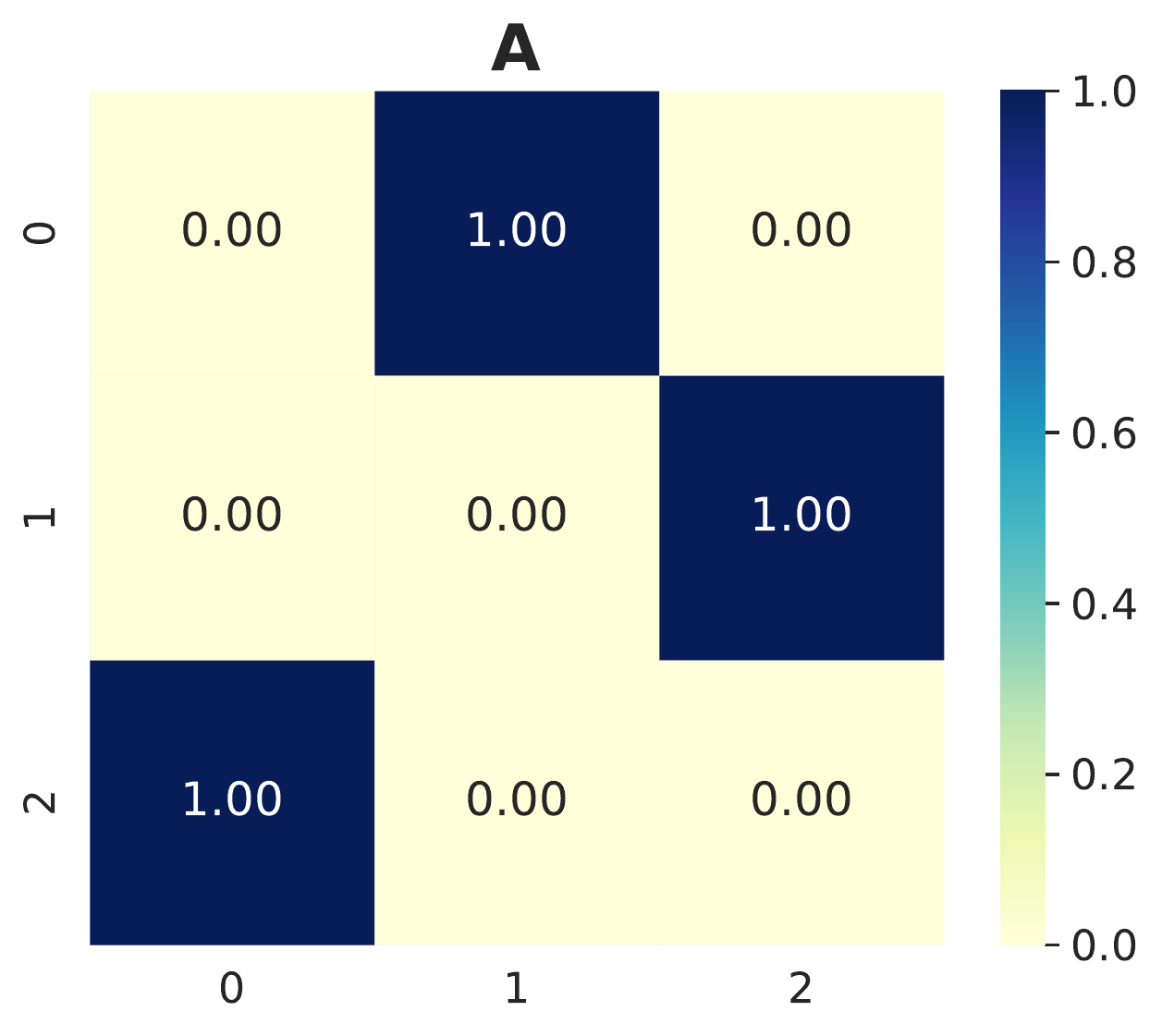}
     \caption{}
\end{subfigure}
\hfill
\begin{subfigure}[b]{0.19\textwidth}
     \centering
     \includegraphics[width=\linewidth]{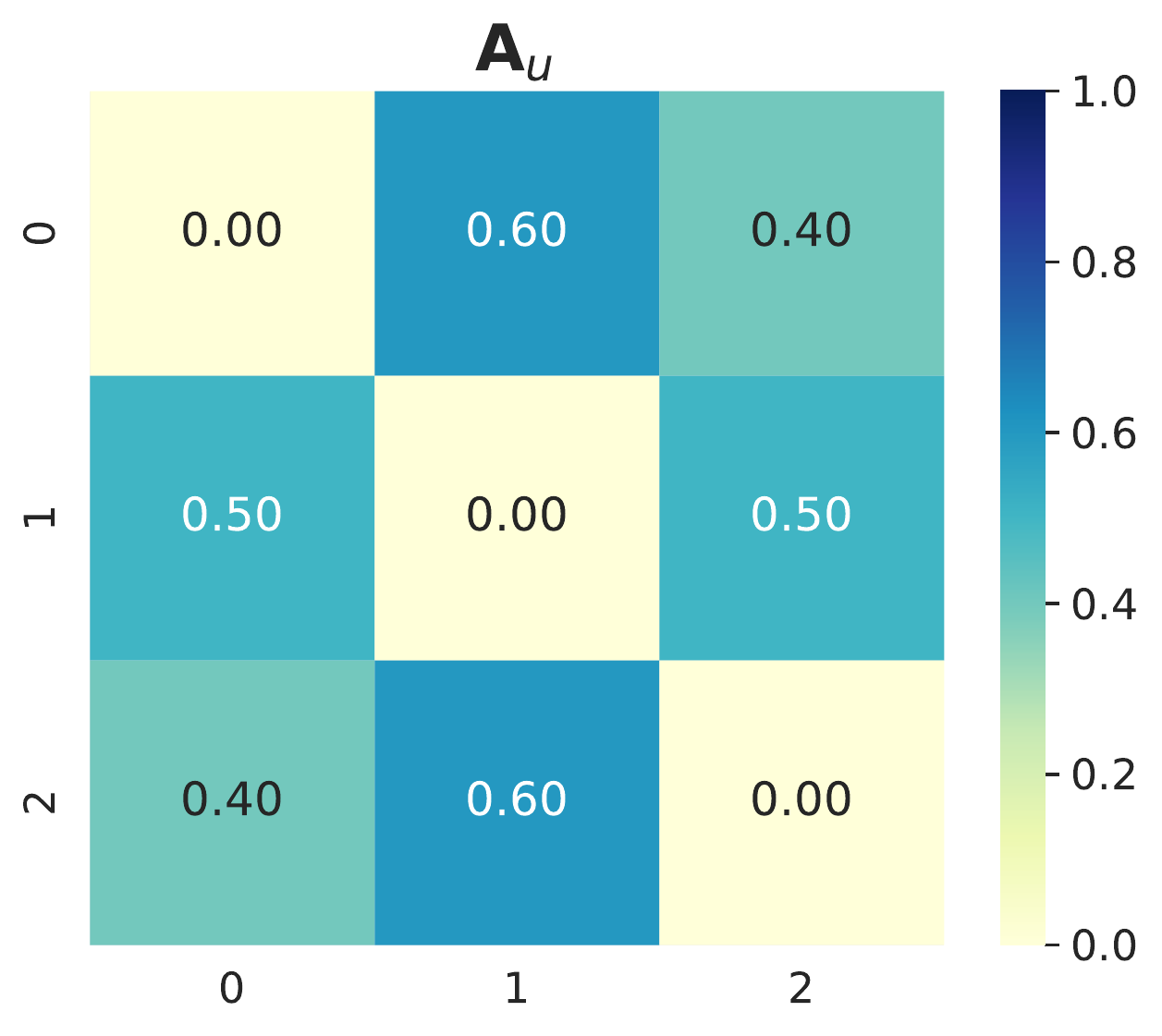}
     \caption{}
\end{subfigure}
\hfill
\begin{subfigure}[b]{0.19\textwidth}
     \centering
     \includegraphics[width=\linewidth]{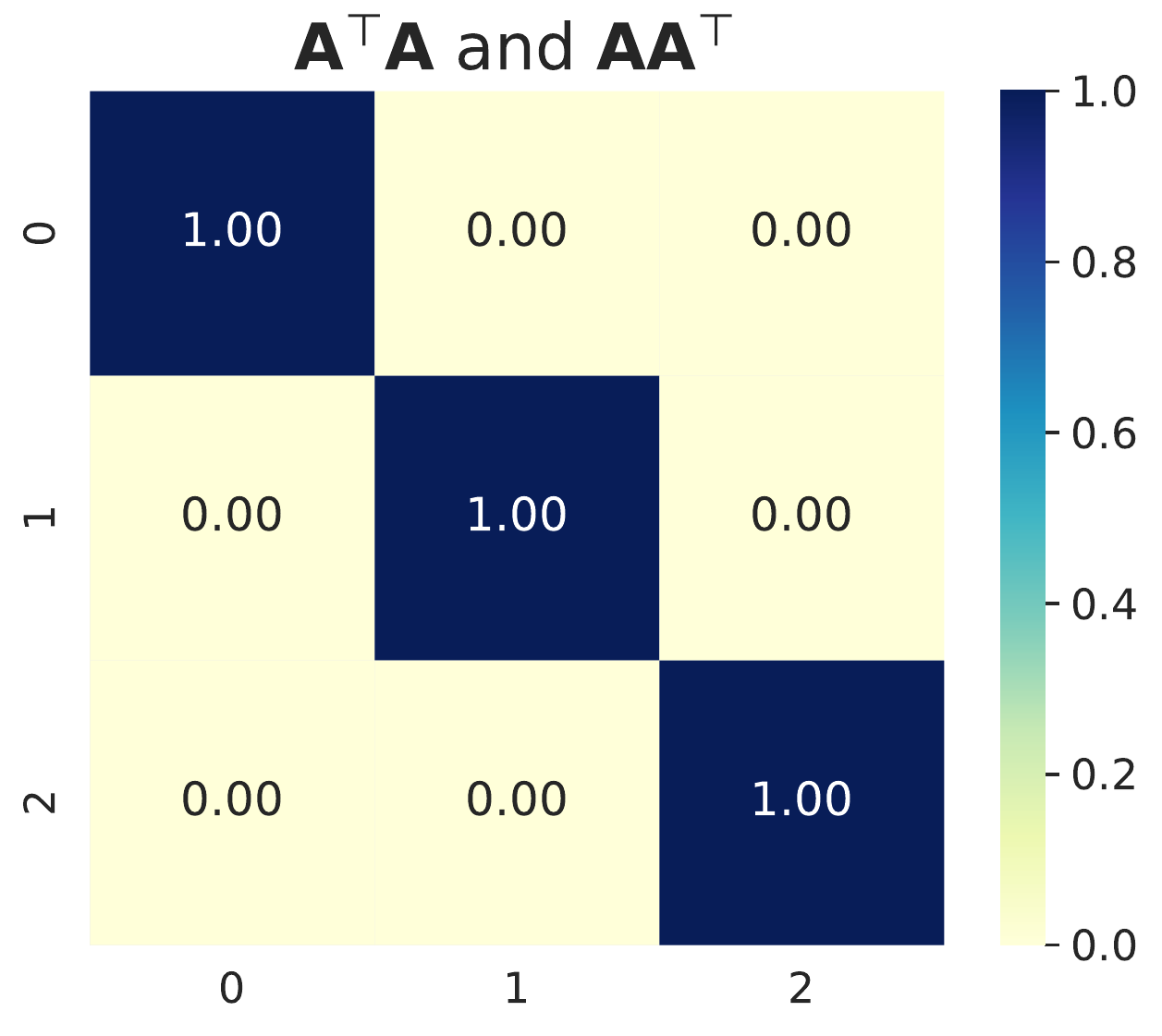}
     \caption{}
\end{subfigure}
\hfill
\begin{subfigure}[b]{0.19\textwidth}
     \centering
     \includegraphics[width=\linewidth]{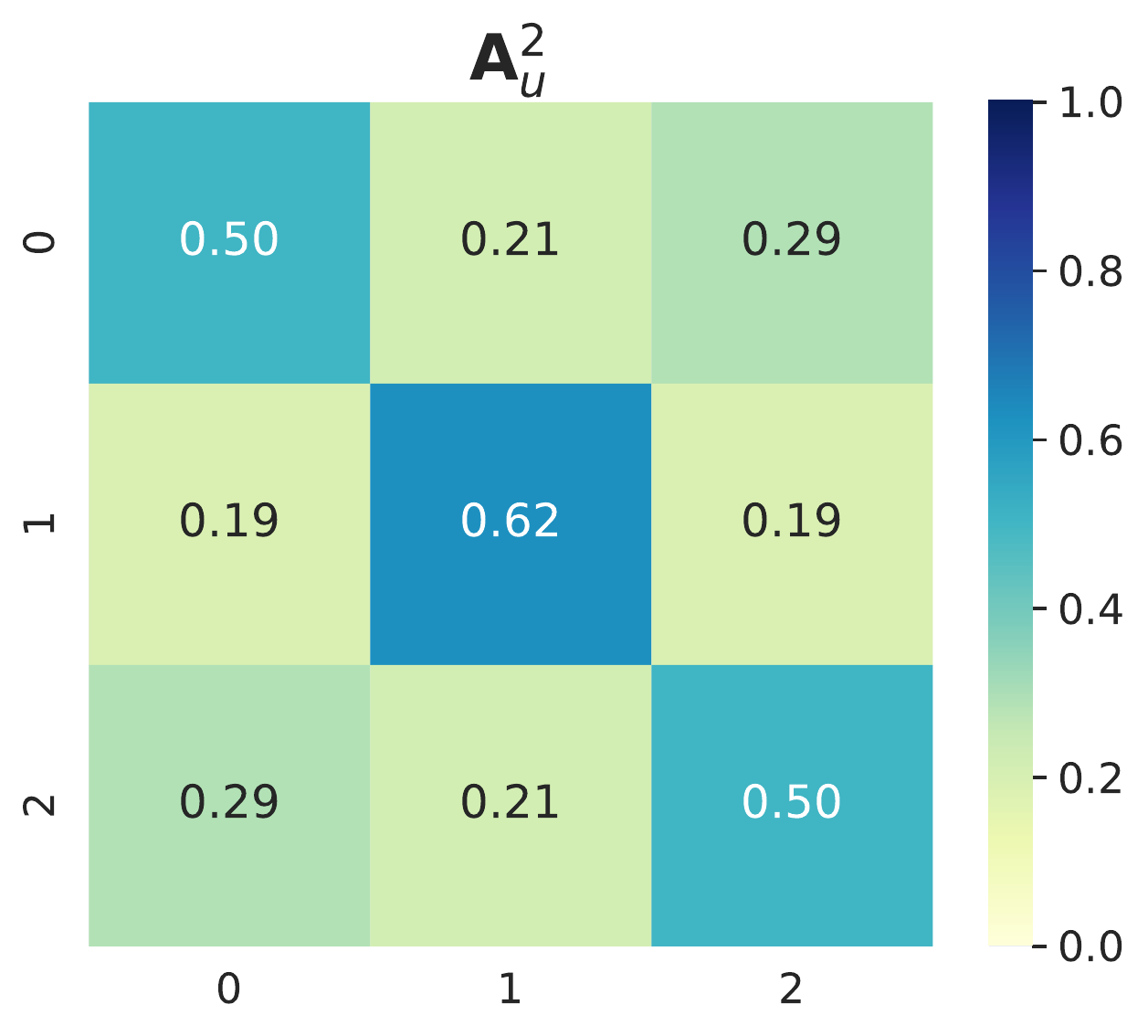}
     \caption{}
\end{subfigure}

\caption{(a) A toy directed graph with three classes showcasing harmless heterophily.
(b) Compatibility matrix of $\mathbf{A}$ showing that classes (blue, orange and green) have very different neighborhoods and can be easily distinguished. (c) Making the graph undirected makes the classes harder to distinguish, making the task harder to solve. 
(d) The mixed directed 2-hops ($\mathbf{A}^\top \mathbf{A}$ and $\mathbf{A} \mathbf{A}^\top$) have perfect homophily, while (e) this is not the case for the undirected 2-hop.}
\label{fig:synthetic_example}
\vspace{-5mm}
\end{figure}

\section{Extension of Popular GNNs} \label{app:extensions}
We first consider an extension of GraphSAGE \cite{hamilton2017inductive} using our \oursacro{} framework. The main choice reduces to that of normalization. In the spirit of GraphSAGE, we require the message-passing matrices $\mathbf{A}_{\inp}$ and $\mathbf{A}_{\out}$ to both be row-stochastic. This is done by taking $\mathbf{A}_{\out} = \mathbf{D}_{\out}^{-1}\mathbf{A}$ and $\mathbf{A}_{\inp} = \mathbf{D}_{\inp}^{-1}\mathbf{A}^\top$, respectively. In this case, the directed version of GraphSAGE becomes
\begin{align*}\label{eq:dir-sage}
    \mathbf{X}^{(k)} =  \sigma(\mathbf{X}^{(k-1)}\mathbf{\Omega}^{(k)} + \mathbf{D}_{\out}^{-1}\mathbf{A}\mathbf{X}^{(k-1)}\mathbf{W}^{(k)}_{\out} + \mathbf{D}_{\inp}^{-1}\mathbf{A}^\top\mathbf{X}^{(k-1)}\mathbf{W}^{(k)}_{\inp}). 
\end{align*}
\noindent Finally, we consider the generalization of GAT \cite{velivckovic2017graph} to the directed case. Here we simply compute attention coefficients over the in- and out-neighbours separately. If we denote the attention coefficient over the edge $(i,j)$ by $\beta_{ij}$, then the update of node $i$ at layer $k$ can be computed as 
\begin{equation*}\label{eq:dir-gat}
    \mathbf{h}_i^{(k)} = \sigma(\sum_{(i,j)\in E}\beta_{ij}^\out \mathbf{W}^{(k)}_{\out}\mathbf{h}_{j}^{(k)} +  \sum_{(j,i)\in E}\beta_{ji}^\inp \mathbf{W}^{(k)}_{\inp}\mathbf{h}_{j}^{(k)}),
\end{equation*}
\noindent where $\beta^\out,\beta^\inp$ are both row-stochastic matrices with support given by $\mathbf{A}$ and $\mathbf{A}^\top$, respectively.

\input{include/09-proofs}

\begin{figure}
    \centering
    \includegraphics[width=0.65\textwidth]{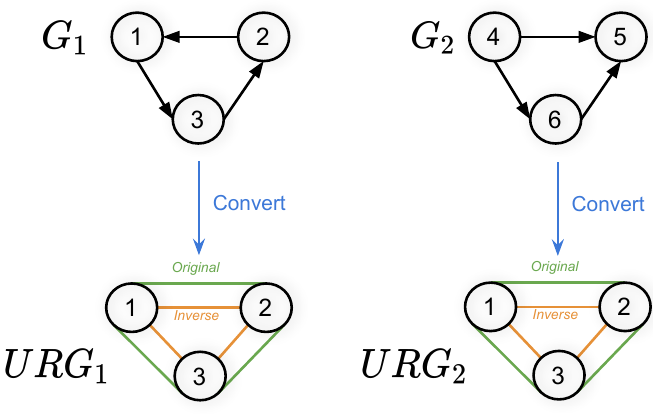}
    \caption{The two non-isomorphic directed graphs $G_1$ and $G_2$ become isomorphic when converted to an undirected relational graph. This shows that is not possible to represent directed graphs with undirected relational graphs without losing information.}
    \label{fig:relational_graph_fail}
\end{figure}

\section{Exploring Alternative Representations for Directed Graphs} \label{sec:alternative_representations}
One might intuitively consider the possibility of representing a directed graph equivalently through an undirected relational graph, having two relations for \textit{original} and \textit{inverse} edges. However, this assumption proves to be erroneous, as illustrated by the following counterexample. Take the two non-isomorphic directed graphs illustrated in~\cref{fig:relational_graph_fail}. Surprisingly, these two directed graphs share an identical representation as an undirected relational graph, as depicted in~\cref{fig:relational_graph_fail}. 

\citet{Kollias2022DirectedGA} showed that it is however possible to equivalently represent a directed graph with an undirected graph having two nodes for each node in the original graph, representing the source and destination role of the node respectively.

\section{MPNN with Binary Edge Features} \label{sec:mpnn_binary_features}
A possible alternative approach to address directed graphs involves using an MPNN~\cite{gilmer2017neural} combined with binary edge features. In this case, the original directed graph is augmented with inverse edges, which are assigned a distinct binary feature compared to the original ones. Given the original directed graph $G=(V, E)$, we define the augmented graph as $G_a=(V, E_a)$ with $E_a= \{ ((i, j), \begin{bmatrix} 0 \ 1 \end{bmatrix} ) : (i, j) \in E \} \cup \{ ((j, i), \begin{bmatrix} 1 \ 0 \end{bmatrix} ) : (i, j) \in E \} $.

An MPNN with edge features can be defined as:

\begin{align*}
    \mathbf{\m}^{(\ilayer)}_{i} &=  \mathrm{AGG}^{(\ilayer)}\left(\ldblbrace (\mathbf{x}_{i}^{(\ilayer)}, \mathbf{x}_{j}^{(\ilayer)}, \mathbf{e}_{ij}):\, (i,j)\in E \rdblbrace \right) \notag \\
    \mathbf{x}_{i}^{(\ilayer)}   &=  \mathrm{COM}^{(\ilayer)}\left(\mathbf{x}_i^{(\ilayer-1)}, \mathbf{\m}^{(\ilayer)}_i\right) 
\end{align*}

However, since messages do not depend only on the source nodes but also on the destination node (through the edge features), this approach necessitates the materialization of explicit edge messages~\cite{tailor2022adaptive}. Given that there are $m$ such messages (one per edge) with dimension $d$, the memory complexity amounts to $\mathcal{O}(m \times d)$. This is significantly higher than the $\mathcal{O}(n \times d)$ memory complexity achieved by instantiations of \oursacro{} such as Dir-GCN or Dir-SAGE, and could lead to out-of-memory issues for even moderately sized datasets. Therefore, while having the same expressivity, our model has better memory complexity. 

\section{Relation with Existing Methods for Directed Graphs} \label{sec:comparison_with_directed_graph_methods}
DGCN~\cite{dgcn} directly employs $\mathbf{A}^\top\mathbf{A}$ and $\mathbf{A}\mathbf{A}^\top$ for spectral convolution, alongside $\mathbf{A}_u$. They observe that these 2-hop matrices enhance feature and label smoothness, a finding that aligns with our observations in~\cref{sec:directed_heterophily}. However, DGCN limits its experiments to homophilic datasets, thereby limiting the applicability of their insights. Furthermore, the model encounters several limitations: it does not provide access to directed 1-hop edges ($\mathbf{A}$ and $\mathbf{A}^\top$); it is constrained to the 2-hop, without the capability to access higher hops—contrastingly, our model consistently benefits from utilizing 5 or 6 layers (see~\cref{tab:best_hyperparameters}); it is limited to a specific, GCN-like architecture, rather than a more general framework; and it lacks scalability due to the explicit computation of $\mathbf{A}^\top\mathbf{A}$ and $\mathbf{A}\mathbf{A}^\top$.

\section{Experimental Details} \label{sec:appendix_experimental_details}
\input{tables/dataset_statistics.tex}
\input{tables/best_hyperparameters.tex}

\subsection{Effective Homophily of Synthetic Graphs}  \label{sec:effective_homophily_synthetic}
For the results in~\cref{fig:synthetic_effective_homophily}, we generate directed synthetic graphs with various homophily levels using a modified preferential attachment process~\cite{Barabasi99emergenceScaling}, inspired by~\citet{zhu2020beyond}. New nodes are incrementally added to the graphs until the desired number of nodes is achieved. Each node is assigned a class label, chosen uniformly at random among $C$ classes, and forms out-edges with exactly $m$ pre-existing nodes, where $m$ is a parameter of the process. The $m$ out-neighbors are sampled without replacement from a distribution that is proportional to both their in-degree and the class compatibility of the two nodes. Consequently, nodes with higher in-degree are more likely to receive new edges, leading to a "rich get richer" effect where a small number of highly connected "hub" nodes emerge. This results in the in-degree distribution of the generated graphs following a power-law, with heterophily controlled by the class compatibility matrix $\mathbf{H}$. In our experiments, we generate graphs comprising 1000 nodes and set $C=5$, $m=2$. Note that by construction, the generated graphs will not have any bidirectional edge.

\subsection{Synthetic Experiment}  \label{sec:appendix_synthetic_experiment}
For the results in~\cref{fig:synthetic_task}, we construct an Erdos-Renyi graph with $5000$ nodes and edge probability of $0.001$, where each node has a scalar feature sampled uniformly at random from $[-1, 1]$. The label of a node is set to 1 if the mean of the features of their in-neighbors is greater than the mean of the features of their out-neighbors, or zero otherwise. 

\subsection{Experimental Setup} \label{app:experimental_setup}
Real-World datasets statistics are reported in table \ref{tab:datasets-statistics}.
All experiments are conducted on a GCP machine with 1 NVIDIA V100 GPU with 16GB of memory, apart from experiments on snap-patents which have been performed on a machine with 1 NVIDIA A100 GPU with 40GB of memory. The total GPU time required to conduct all the experiments presented in this paper is approximately two weeks.
In all experiments, we use the Adam optimizer and train the model for 10000 epochs, using early stopping on the validation accuracy with a patience of 200 for all datasets apart from Chameleon and Squirrel, for which we use a patience of 400. We do not use regularization as it did not help on heterophilic datasets. For Citeseer-Full and Cora-ML we use random 50/25/25 splits, for OGBN-Arxiv we use the fixed split provided by OGB~\cite{hu2020ogb}, for Chameleon and Squirrel we use the fixed GEOM-GCN splits~\cite{pei2020geom}, for Arxiv-Year and Snap-Patents we use the splits provided in \citeauthor{lim2021large}, while for Roman-Empire we use the splits from~\citet{platonov2023a}. We report the mean and standard deviation of the test accuracy, computed over 10 runs in all experiments. 

\subsection{Directionality Ablation Hyperparameters} \label{sec:appendix_ablation_hyperparams}
For the ablation study in~\cref{sec:experiments_extending_gnn}, we use the same hyperparameters for all models and datasets: $\mathrm{learning\_rate} = 0.001$, $\mathrm{hidden\_dimension} = 64$, $\mathrm{num\_layers} = 3$, $\mathrm{norm} = True$, $\mathrm{jk}=max$. $\mathrm{norm}$ refers to applying L2 normalization after each convolutional layer, which we found to be generally useful, while $\mathrm{jk}$ refers to the type of jumping knowledge~\cite{pmlr-v80-xu18c} used.

\subsection{Comparison with State-of-the-Art Results} \label{sec:appendix_grid_search}

To obtain the results for Dir-GNN in~\cref{tab:heterophilic_results}, we perform a grid search over the following hyperparameters: $\mathrm{model\_type} \in\{$Dir-GCN, Dir-SAGE$\}$, $\mathrm{learning\_rate} \in \{0.01, 0.005, 0.001, 0.0005\}$, $\mathrm{hidden\_dimensio}n \in \{32, 64, 128, 256, 512\}$, $\mathrm{num\_layers} \in \{2, 3, 4, 5, 6\}$, $\mathrm{jk} \in \{max, cat, none\}$, $\mathrm{norm} \in \{True, False \}$, $\mathrm{dropout} \in \{0, 0.2, 0.4, 0.6, 0.8, 1\}$ and $\alpha \in \{0, 0.5, 1\}$. The best hyperparameters for each dataset are reported in table \ref{tab:best_hyperparameters}.

\subsection{Baseline Results} \label{sec:appendix_baseline_results}
\paragraph{GNNs for Heterophily}
Results for H$_2$GCN, GPR-GNN and LINKX were taken from \citeauthor{lim2021large}. Results for Gradient Gating are taken from their paper~\cite{rusch2022gradient}. Results for FSGNN are taken from their paper~\cite{maurya2021improving} for Actor, Squirrel and Chameleon, whereas we re-implement it to generate results on Arxiv-year and Snap-Patents, performing the same gridsearch outlined in~\cref{sec:appendix_grid_search}. Results for GloGNN as well as MLP and GCN are taken from~\citeauthor{li2022finding}. Results on Roman-Empire are taken from~\citet{platonov2023a} for GCN, H$_2$GCN, GPR-GNN, FSGNN and GloGNN whereas we re-implement and generate results for MLP, LINKX, ACM-GCN and Gradient Gating performing the same gridsearch outlined in~\cref{sec:appendix_grid_search}.

\paragraph{Directed GNNs}
For DiGCN and MagNet, we used the classes provided by PyTorch Geometric Signed Directed library~\cite{he2022pytorch}. For MagNet, we tuned the $\mathrm{learning\_rate} \in \{0.01, 0.005, 0.001, 0.0005\}$, the $\mathrm{hidden\_dim} \in \{32, 64, 128, 256, 512\}$, the $\mathrm{num\_layers} \in \{2, 3, 4, 5, 6\}$, the $K$ parameter for its chebyshev convolution to $\in \{1, 2\}$, and its $q$ hyperparameter $\in \{0, 0.05, 0.10, 0.15, 0.20\}$. For DiGCN, we tune the $\mathrm{learning\_rate} \in \{0.01, 0.005, 0.001, 0.0005\}$, the $\mathrm{hidden\_dim} \in \{32, 64, 128, 256, 512\}$, the $\mathrm{num\_layers} \in \{2, 3, 4, 5, 6\}$, and the $\alpha$ $\in \{0.05, 0.10, 0.15, 0.20\}$.

\input{tables/zero_degrees.tex}
\input{tables/full_direction_ablation.tex}
\section{Additional Results}

\begin{figure}[t!]
\centering
\label{fig:synthetic_gcn_gat}
% \hspace{-0.5cm}
\begin{minipage}[t]{0.5\linewidth}
  \centering
  \includegraphics[width=1\linewidth]{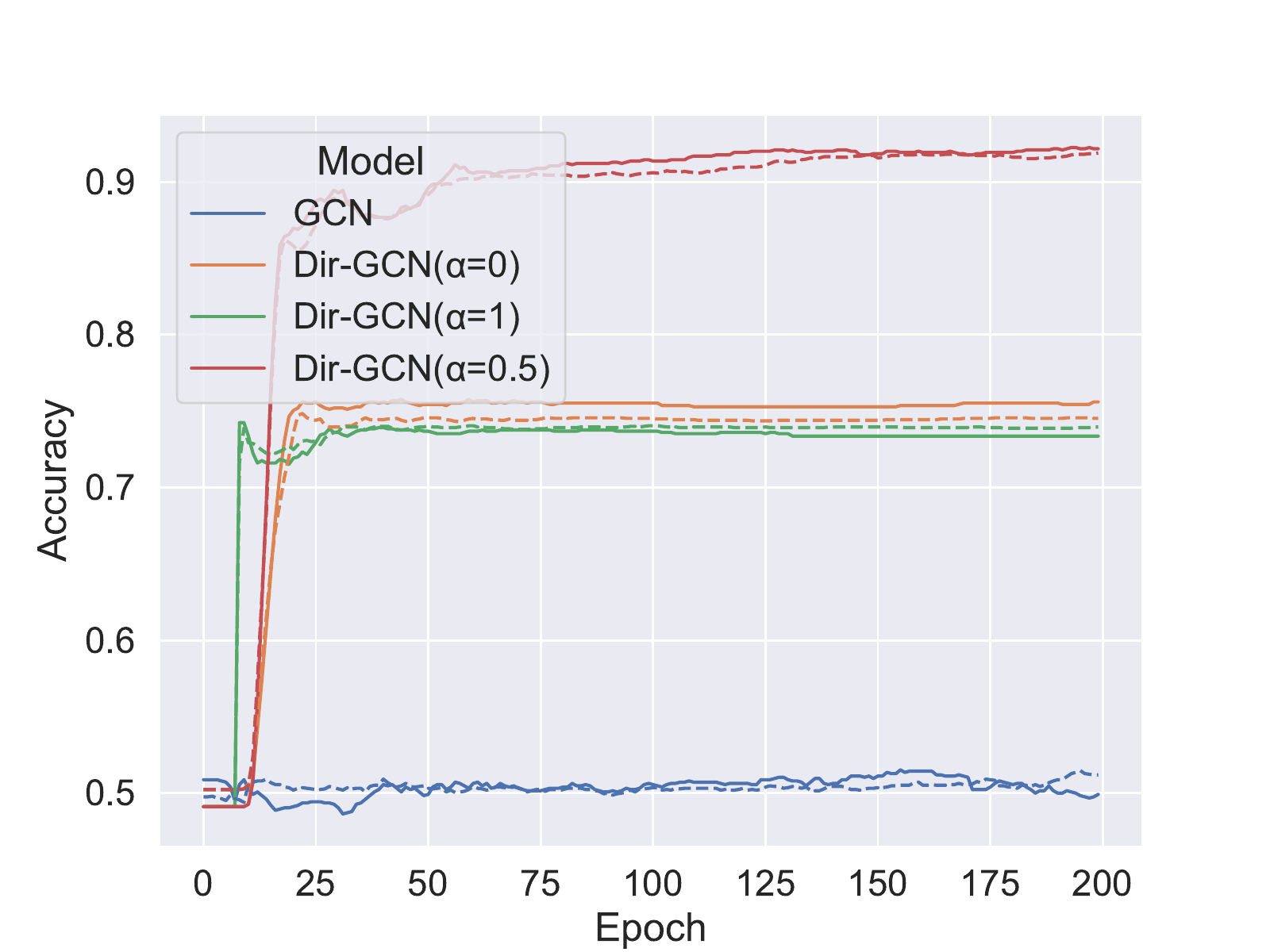}
\end{minipage}%
\begin{minipage}[t]{.5\linewidth}
  \centering
  \includegraphics[width=1\linewidth]{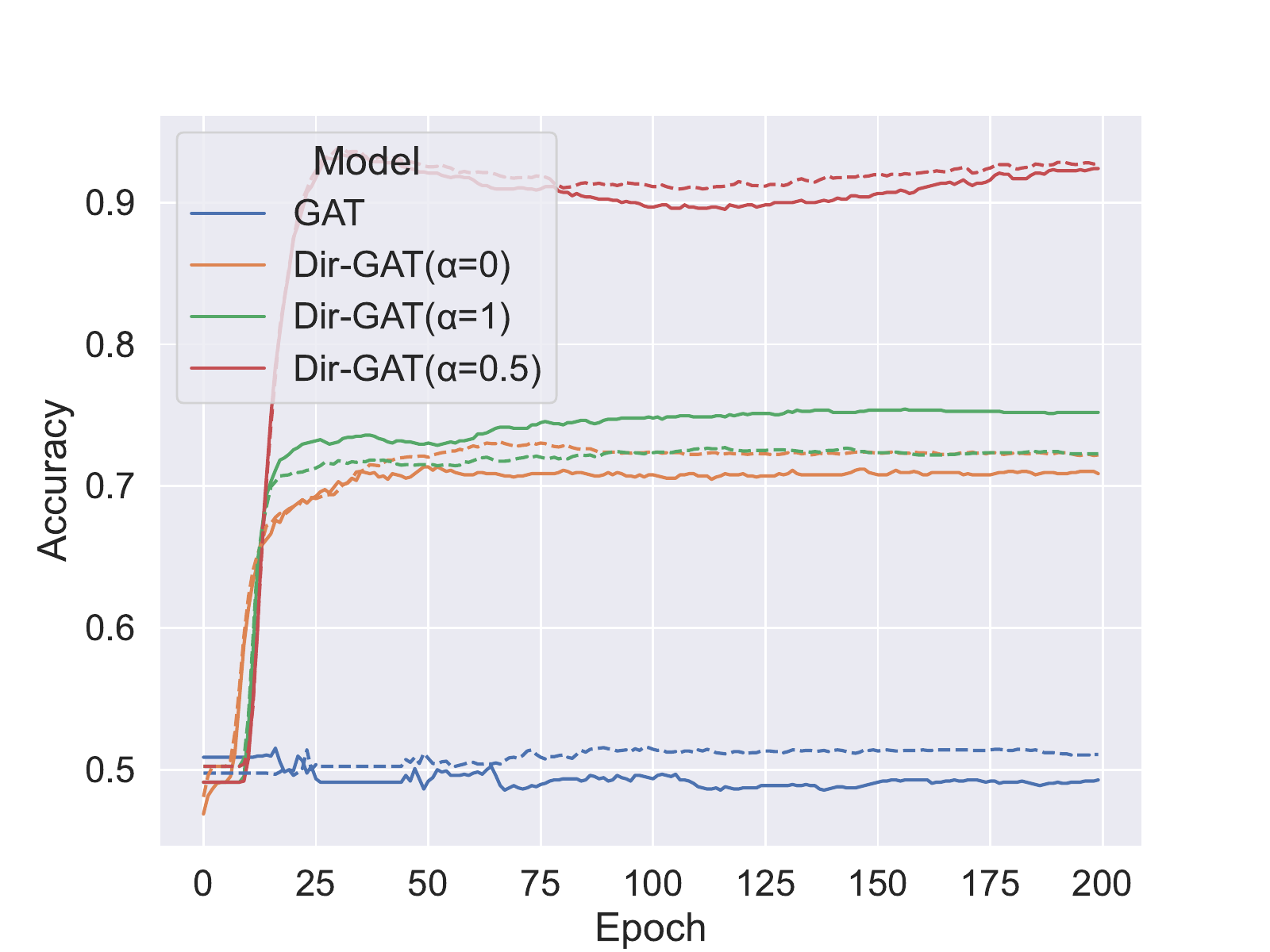}
\end{minipage}
\caption{Validation accuracy (solid lines) and training accuracy (dashed lines) of GCN (left) and GAT (right), as well as their respective extensions using our \oursacro{} framework, on a synthetic task which requires directionality information in order to be solved.}
  \label{fig:synthetic-gcn-gat}
\end{figure}

\subsection{Synthetic Experiment}
We also evaluate GCN and GAT (and their \oursacro{} extensions) on the synthetic task outlined in~\cref{sec:appendix_synthetic_experiment}. Similarly to what observed for GraphSage (\cref{fig:synthetic_task}), the \oursacro{} variant using both directions ($\alpha=0.5$) significantly outperforms the other configurations, despite not reaching 100\% accuracy. The undirected models are akin to a random classifier, whereas the models using only one directions obtain between 70\% and 75\% of accuracy.

\subsection{Ablation Study on Using Directionality}
Table \ref{tab:full_direction_ablation} compares using the undirected graph vs using the directed graph with our framework with different $\alpha$. We observe that only on Chameleon and Squirrel, using only one direction of the edges, in particular the out direction, performs better than using both direction. Moreover, for these two datasets, the gap between the two directions ($\alpha=0$ vs $\alpha=1$) is extremely large (more than 40\% absolute accuracy). We find that this is likely due to the high number of nodes with zero in neighbors, as reported in Table \ref{tab:zero_degrees}. Chameleon and Squirrel have respectively about 62\% and 57\% of nodes with no in-edges: when propagating only over in edges, these nodes would get zero features. 
We observe a similar trend for other datasets, where $\alpha=1$ performs generally better than $\alpha=0$, in line with the fact that all these datasets have more nodes with zero in edges than out edges (Table \ref{tab:zero_degrees}). In general, using both in- and out- edges is the preferred solution.

\begin{figure*}
    \centering
    \includegraphics[width=.5\textwidth]{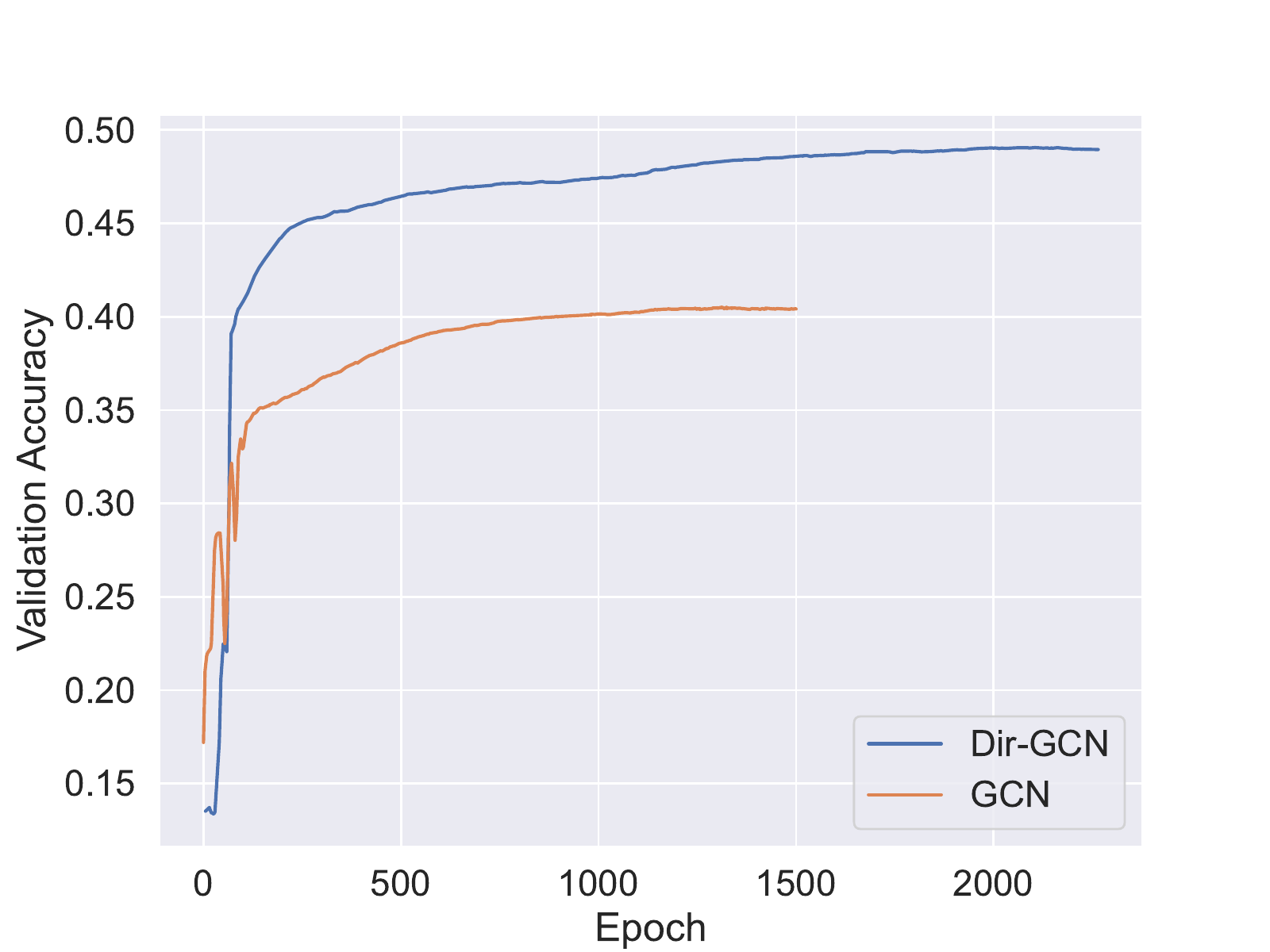}
    \caption{Performance of GCN (on the undirected version of the graph) and Dir-GCN on Arxiv-Year when using only one layer. Remarkably, directionality yields significant benefits, even in the absence of access to the homophilic directed 2-hop. This is largely attributable to the harmless heterophily exhibited by the directed graph.}
    \label{fig:arxiv1-layer}
\end{figure*}

\subsection{Ablation Study on Using a Single Layer}
In~\cref{sec:directed_heterophily} we discuss how Arxiv-Year and Snap-Patents exhibit harmless heterophily when treating as directed. This suggest that even a 1-layer Dir-GNN model should be able to perform much better of its undirected counterpart, despite not being able to access the much more homophilic 2-hop. We verify this empirically by comparing a 1-layer GCN (on the undirected version of the graph) with a 1-layer Dir-GCN on Arxiv-Year. \cref{fig:arxiv1-layer} presents the results, showing that Dir-GCN does indeed significantly outperform GCN. 

\begin{figure*}
    \centering
    \includegraphics[width=0.5\textwidth]{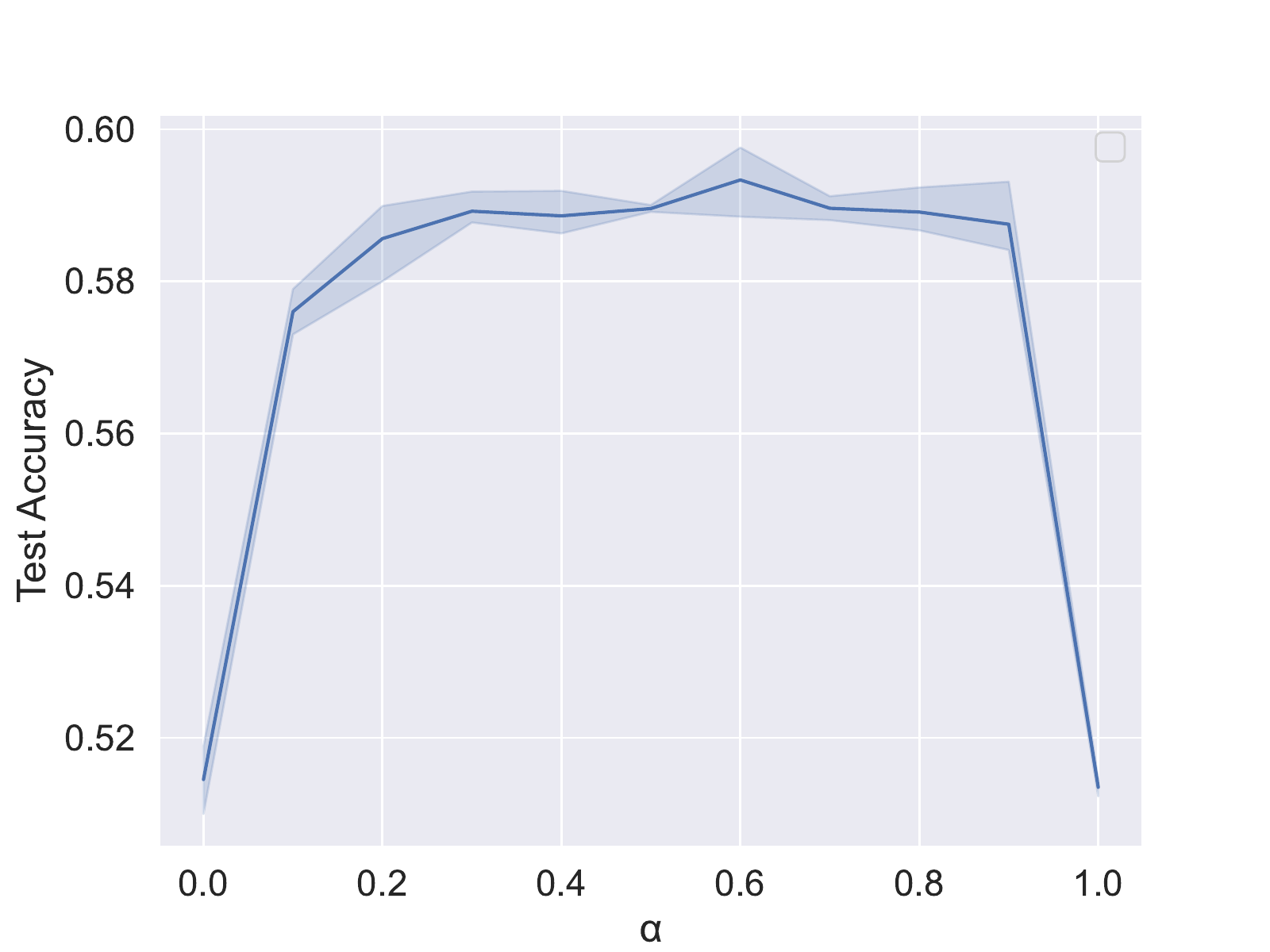}
    \caption{\oursacro{} test accuracy on Arxiv-Year for different values of the hyperparameter $\alpha$.}
    \label{fig:alpha_ablation}
\end{figure*}

\subsection{Ablation Study on Different Values of \texorpdfstring{$\alpha$}{alpha}}
We train \oursacro{} models on Arxiv-Year with varying values of $\alpha$, using the hyperparameters outlined in~\cref{sec:appendix_ablation_hyperparams}. \cref{fig:alpha_ablation} presents the results: while a large drop is observed for $\alpha=0$ and $\alpha=1$, i.e. propagating messages only along one direction, the results for other values of $\alpha$ are largely similar. 

\subsection{Runtime Analysis}
We assess the runtime of Dir-GNN by comparing it with its undirected counterpart. Specifically, in \cref{fig:wall_time}, we illustrate a comparison between the validation accuracies of Dir-Sage and Sage over time, measured in seconds. Both models were executed on a single NVIDIA V100 GPU  with 16GB of memory. Notably, despite employing two separate weight matrices—which enhances its performance—Dir-Sage exhibits only a marginal increase in runtime. The respective runtimes per epoch are $0.71 \pm 0.59$ seconds for Dir-Sage and $0.57 \pm 0.48$ seconds for Sage.

\begin{figure*}
    \centering
    \includegraphics[width=0.5\textwidth]{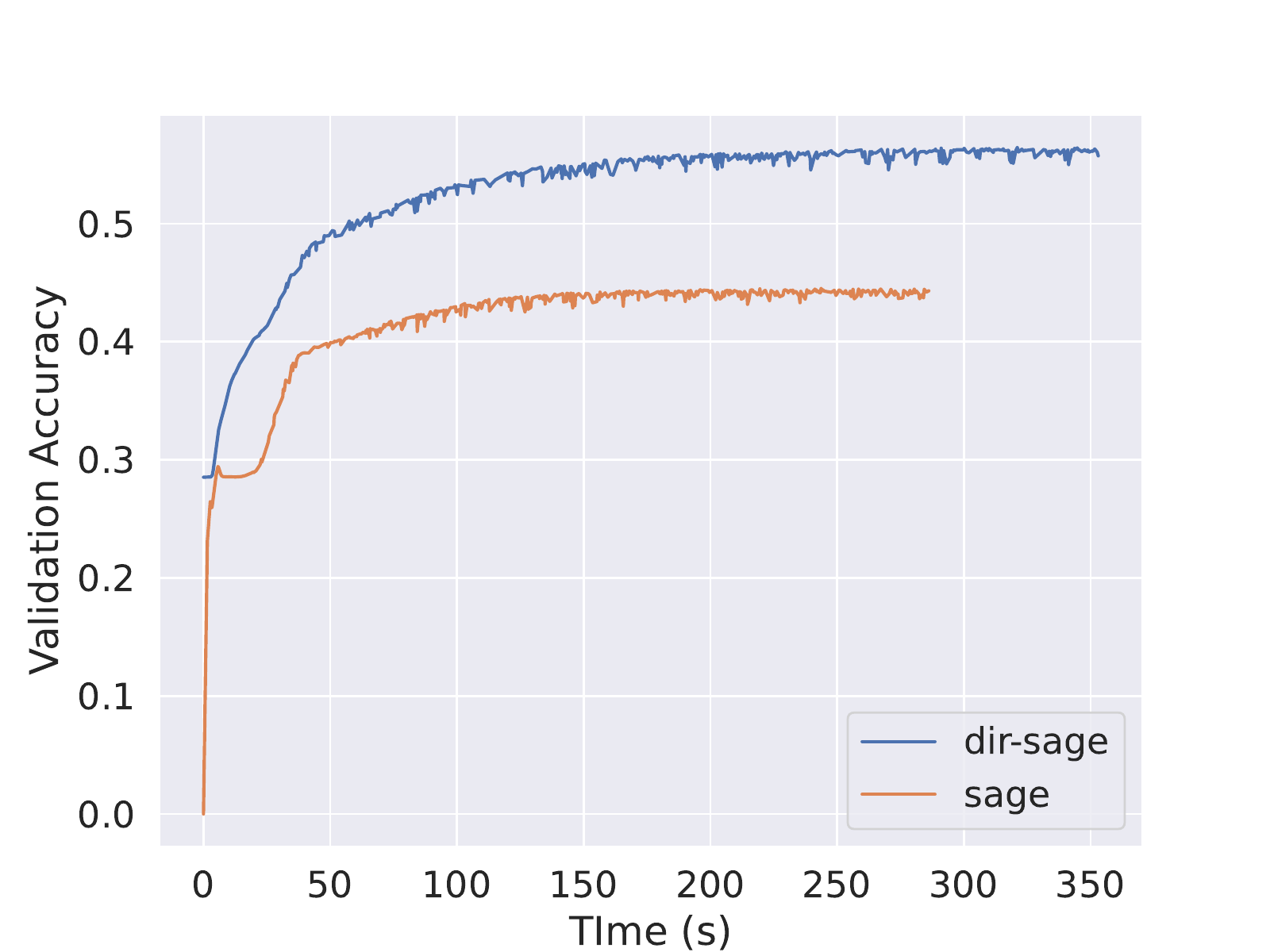}
    \caption{Validation Accuracy of Dir-Sage and Sage, displayed as a function of time (seconds).}
    \label{fig:wall_time}
\end{figure*}

% \section{Compatibility Matrices} \label{sec:compatibility_matrices}

% We report weighted compatibility matrices of all the 1-hop and 2-hop diffusion operators (both undirected and directed) for the datasets used in the paper.

% \begin{figure*}
%     \centering
%     \includegraphics[width=1\textwidth]{images/compatibility_matrices/chameleon-compatibility-matrices.pdf}
%     \caption{Compatibility matrix for chameleon.}
%     \label{fig:chameleon-compat}
% \end{figure*}

% \begin{figure*}
%     \centering
%     \includegraphics[width=1\textwidth]{images/compatibility_matrices/squirrel-compatibility-matrices.pdf}
%     \caption{Compatibility matrix for squirrel.}
%     \label{fig:squirrel-compat}
% \end{figure*}

% \begin{figure*}
%     \centering
%     \includegraphics[width=1\textwidth]{images/compatibility_matrices/arxiv-year-compatibility-matrices.pdf}
%     \caption{Compatibility matrix for arxiv-year.}
%     \label{fig:arxiv-year-compat-full}
% \end{figure*}

% \begin{figure*}
%     \centering
%     \includegraphics[width=1\textwidth]{images/compatibility_matrices/snap-patents-compatibility-matrices.pdf}
%     \caption{Compatibility matrix for snap-patents.}
%     \label{fig:snap-patents-compat}
% \end{figure*}

%% file: include/09-proofs.tex
\section{Analysis of Expressivity} \label{app:expressivity_analysis}

In this appendix we prove the expressivity results reported in \cref{sec:expressivity}, after restating them more formally. It is important to note that we cannot build on the expressivity results from \citet{barcelo2022weisfeiler}, since their scope is limited to undirected relational graphs, and (perhaps surprisingly) it is not possible to equivalently represent a directed graph with an undirected relational graphs, as we show in~\cref{sec:alternative_representations}.

We start by introducing useful concepts which will be instrumental to our discussion. As commonly done, we will assume in our analysis that all nodes have constant scalar node features $c$. 

\subsection{(Directed) Weisfeiler-Lehman Test} \label{sec:appendix_wl}
The 1-dimensional Weisfeiler-Lehman algorithm (1-WL), or color refinement, is a heuristic approach to the graph isomorphism problem, initially proposed by~\citet{weisfeiler1968reduction}. This algorithm is essentially an iterative process of vertex labeling or coloring, aimed at identifying whether two graphs are non-isomorphic.

Starting with an identical coloring of the vertices in both graphs, the algorithm proceeds through multiple iterations. In each round, vertices with identical colors are assigned different colors if their respective sets of equally-colored neighbors are unequal in number. The algorithm continues this process until it either reaches a point where the distribution of vertices colors is different in the two graphs or converges to the same distribution. In the former case, the algorithm concludes that the graphs are not isomorphic and halts. Alternatively, the algorithm terminates with an inconclusive result: the two graphs are `possibly isomorphic'. It has been shown that this algorithm cannot distinguish all non-isomorphic graphs~\cite{cai1992anoptimal}. 

Formally, given an undirected graph $G = (V, E)$, the 1-WL algorithm calculates a node coloring $C^{(t)}: V (G) \rightarrow \mathbb{N}$ for each iteration $t > 0$, as follows:

\begin{equation}
C^{(t)}(i) = \mathrm{RELABEL}\left( C^{(t-1)}(i), \ldblbrace C^{(t-1)}(j) : j \in N(i) \rdblbrace\right)
\end{equation}

where $\mathrm{RELABEL}$ is a function that injectively assigns a unique color, not used in previous iterations, to the pair of arguments. The function $N(i)$ represents the set of neighbours of $i$.

Since we deal with directed graphs, it is necessary to extend the 1-WL test to accommodate directed graphs. We note that a few variants have been proposed in the literature~\cite{Grohe2021ColorRA,DBLP:journals/corr/abs-1904-08745,Kollias2022DirectedGA}. Here, we focus on a variant whereby in- and out-neighbours are treated separately, as discussed in~\citep{Grohe2021ColorRA}. This variant, which we refer to as \textit{D-WL}, refines colours as follows:

\begin{equation}
D^{(t)}(i) = \mathrm{RELABEL}\left( D^{(t-1)}(i), \ldblbrace D^{(t-1)}(j) : j \in N_\out(i) \rdblbrace, \ldblbrace D^{(t-1)}(j) : j \in N_\inp(i) \rdblbrace \right)
\end{equation}

where $N_\out(i)$ and $N_\inp(i)$ are the set of out- and in-neighbors of $i$, respectively. Our first objective is to demonstrate that \oursacro{} is as expressive as D-WL. Establishing this will enable us to further show that \oursacro{} is strictly more expressive than an MPNN operating on either the directed or undirected version of a graph. Let us start by introducing some further auxiliary tools that will be used in our analysis.

\subsection{Expressiveness and color refinements}

A way to compare graph models (or algorithms) in their expressiveness is by contrasting their discriminative power, that is the ability they have to disambiguate between non-isomorphic graphs.

Two graphs are called \emph{isomorphic} whenever there exists a graph isomorphism between the two:
\begin{definition}[Graph isomorphism]\label{def:isomorphism}
    Let $G_1 = (V_1, E_1)$, $G_2 = (V_2, E_2)$ be two (directed) graphs. An \emph{isomorphism} between $G_1, G_2$ is a bijective map $\varphi: V_1 \rightarrow V_2$ which preserves adjacencies, that is: $\forall u, v \in V_1: (u, v) \in E_1 \Longleftrightarrow (\varphi(u), \varphi(v)) \in E_2$. 
\end{definition}
\noindent On the contrary, they are deemed non-isomorphic when such a bijection does not exist. A model that is able to discriminate between two non-isomorphic graphs assigns them distinct representations. This concept is extended to families of models as follows:
\begin{definition}[Graph discrimination]
    Let $G = (V, E)$ be any (directed) graph and $M$ a model belonging to some family $\mathcal{M}$. 
    Let $G_1$ and $G_2$ be two graphs. We say $M$ discriminates $G_1$, $G_2$ iff $M(G_1) \neq M(G_2)$. 
    We write $G_1 \neq_M G_2$. If there exists such a model $M \in \mathcal{M}$, then family $\mathcal{M}$ distinguishes between the two graphs and we write $G_1 \neq_{\mathcal{M}} G_2$. 
\end{definition}

Families of models can be compared in their expressive power in terms of graph disambiguation:
\begin{definition}[At least as expressive]\label{def:at_least_as_expressive}
    Let $\mathcal{M}_1, \mathcal{M}_2$ be two model families. We say $\mathcal{M}_1$ is at least as expressive as $\mathcal{M}_2\ \mathrm{iff}\ \forall G_1 = (V_1, E_1), G_2 = (V_2, E_2), G_1 \neq_{\mathcal{M}_2} G_2 \implies G_1 \neq_{\mathcal{M}_1} G_2$. We write $\mathcal{M}_1 \sqsubseteq \mathcal{M}_2$.
\end{definition}

Intuitively, $\mathcal{M}_1$ is at least as expressive as $\mathcal{M}_2$ if when $\mathcal{M}_2$ discriminates a pair of graphs, also $\mathcal{M}_1$ does. Additionally, a family can be \emph{strictly} more expressive than another:

\begin{definition}[Strictly more expressive] \label{def:strictly_more_expressive}
    Let $\mathcal{M}_1, \mathcal{M}_2$ be two model families. We say $\mathcal{M}_1$ is strictly more expressive than $\mathcal{M}_2\ \mathrm{iff}\ \mathcal{M}_1 \sqsubseteq \mathcal{M}_2 \land \mathcal{M}_2 \not\sqsubseteq \mathcal{M}_1$. Equivalently, $\mathcal{M}_1 \sqsubseteq \mathcal{M}_2 \land \exists G_1 = (V_1, E_1), G_2 = (V_2, E_2),
    \ \mathrm{s.t.}\ G_1 \neq_{\mathcal{M}_1} G_2 \land G_1 =_{\mathcal{M}_2} G_2$.
\end{definition}

Intuitively, $\mathcal{M}_1$ is strictly more expressive than $\mathcal{M}_2$ if $\mathcal{M}_1$ is at least as expressive as $\mathcal{M}_2$ and there exist pairs of graphs that $\mathcal{M}_1$ distinguishes but $\mathcal{M}_1$ does not. 

Many graph algorithms and models operate by generating \emph{node colorings} or representations. These can be gathered into multisets of colors that are compared to assess whether two graphs are non-isomorphic. Other than convenient, in these cases it is interesting to characterise the discriminative power at the level of nodes by means of the concept of {color refinement}~\cite{morris2019weisfeiler,bodnar2022neural,bevilacqua2022equivariant}. 
\begin{definition}[Color refinement]
    Let $G = (V, E)$ be a graph and $C, D$ two coloring functions. Coloring $D$ refines colouring $C$ when $\forall v, w \in V, D(v) = D(w) \implies C(v) = C(w)$.
\end{definition}
\noindent Essentially, when $D$ refines $C$, if any two nodes are assigned the same color by $D$, the same holds for $C$. Equivalently, if two nodes are distinguished by $C$ (because they are assigned different colors), then they are also distinguished by $D$. When, for any graph, $D$ refines $C$, then we write $D \sqsubseteq C$ and, when also the opposite holds, ($C \sqsubseteq D$), we then write $D \equiv C$. As an example, for any $t \geq 0$ it can be shown that, on any graph, the coloring generated by the 1-WL algorithm at round $t+1$ refines that at round $t$, that is $C^{(t+1)} \sqsubseteq C^{(t)}$; this being essentially due to the injectivity property of the $\mathrm{RELABEL}$ function.

Importantly, as we were anticipating above, this concept can be directly translated into graph discrimination as long as graphs are represented by the multiset of their vertices' colors, or an \emph{injection} thereof. This link, which explains the use of the same symbol to refer to the concepts of color refinement and discriminative power, is explicitly shown, for example, in~\citet{pmlr-v139-bodnar21a,bevilacqua2022equivariant}. More concretely, it can be shown that, if coloring $D$ refines coloring $C$, then the algorithm which generates $D$ is at least as expressive as the one generating $C$, as long as multisets of node colours are directly compared to discriminate between graphs, or they are first encoded by a multiset injection before the comparison is carried out. In the following we will resort to the concept of color refinement to prove some of our theoretical results. This approach is not only practically convenient for the required derivations, but it also informs us on the discriminative power models have at the level of nodes, something which is of relevance to us given our focus on node-classification tasks.

Furthermore, even though \oursacro{} outputs node-wise embeddings, it can be augmented with a global readout function to generate a single graph-wise embeddings $\mathbf{x}_{G} = \mathrm{READOUT}\left(\ldblbrace \mathbf{x}_{i}^{(K)} : i \in V \rdblbrace\right)$. We will assume that all models discussed in this section are augmented with a global readout function.

\subsection{MPNNs on Directed Graphs}

Before moving forward to prove our expressiveness results, let us introduce the families of architectures we compare with. These embody straightforward approaches to adapt MPNNs to directed graphs.

Let MPNN-D be a model that performs message-passing by only propagating messages in accordance with the directionality of edges. Its layers can be defined as follows:

\begin{align}
\begin{split}
\mathbf{\m}^{(\ilayer)}_{i} &=  \mathrm{AGG}^{(\ilayer)}\left(\ldblbrace (\mathbf{x}_{j}^{(\ilayer-1)}, \mathbf{x}_{i}^{(\ilayer-1)}):\, j \in N_\out(i) \rdblbrace\right) \\
\mathbf{x}_{i}^{(\ilayer)}   &=  \mathrm{COM}^{(\ilayer)}\left(\mathbf{x}_i^{(\ilayer-1)}, \mathbf{\m}^{(\ilayer)}_i\right) \label{eq:mpnn-d}
\end{split}
\end{align}

Instead, let MPNN-U be a model which propagates messages equally along any incident edge, independent of their directionality. Its layers can be defined as follows:
\begin{align}
\begin{split}
\mathbf{\m}^{(\ilayer)}_{i} &=  \mathrm{AGG}^{(\ilayer)}\left(\ldblbrace (\mathbf{x}_{j}^{(\ilayer-1)}, \mathbf{x}_{i}^{(\ilayer-1)}):\, j \in N_\out(i) \rdblbrace \cup \ldblbrace (\mathbf{x}_{j}^{(\ilayer-1)}, \mathbf{x}_{i}^{(\ilayer-1)}):\, j \in N_\inp(i) \rdblbrace \right) \\
\mathbf{x}_{i}^{(\ilayer)}   &=  \mathrm{COM}^{(\ilayer)}\left(\mathbf{x}_i^{(\ilayer-1)}, \mathbf{\m}^{(\ilayer)}_i\right) \label{eq:mpnn-u}
\end{split}
\end{align}

Note that if there are no bi-directional edges, MPNN-U is equivalent to first converting the graph to its undirected form (where the edge set is redefined as $E^{(u)}=\{ 
(i, j): (i, j) \in E \lor (j, i) \in E \}$) and then running an undirected MPNN(~\cref{eq:mpnn}). In practice, we observe that the number of bi-directional edges is generally small on average, while \emph{extremely} small on specific datasets (see~\cref{tab:datasets-statistics}). In these cases, we expect the empirical performance of the two approaches to be close to each other. We remark that, in our experiments, we opt for the latter strategy as it is easier and more efficient to implement.

We can now formally define families for the models we will be comparing.

\begin{definition}[Model families] \label{def:model_families}
    Let $\mathcal{M}_\mathrm{MPNN-D}$ be the family of Message Passing Neural Networks on the directed graph (\cref{eq:mpnn-d}), $\mathcal{M}_\mathrm{MPNN-U}$ that of Message Passing Neural Networks on the undirected form of the graph (\cref{eq:mpnn-u}), and $\mathcal{M}_\mathrm{\oursacro{}}$ that of \oursacro{} models (\cref{eq:directed-mpnn}).
\end{definition}

\subsection{Comparison with D-WL} \label{app:d-wl-proof}

We start by restating~\cref{thm:dirgnn-as-expressive-as-d-wl} more formally:

\begin{theorem}\label{thm:dirgnn-as-expressive-as-d-wl-formal}
    $\mathcal{M}_{\oursacro{}}$ is as expressive as D-WL if $\mathrm{AGG}^{(k)}_{\out}$, $\mathrm{AGG}^{(k)}_{\inp}$, and $\mathrm{COM}^{(k)}$ are injective for all $k$ and node representations are aggregated via an injective $\mathrm{READOUT}$ function. 
\end{theorem}

We now prove the theorem by showing that D-WL and \oursacro{} (under the hypotheses of the theorem) are equivalent in their expressive power. We will show this in terms of color refinement and, in particular, by showing that, not only the D-WL coloring at any round $t$ refines that induced by any \oursacro{} at the same time step, but also that, when \oursacro{}'s components are injective, the opposite holds.

\begin{proof}[Proof of \Cref{thm:dirgnn-as-expressive-as-d-wl-formal}]

Let us begin by showing that \oursacro{} is upper-bounded in expressive power by the D-WL test. We do this by showing that, at any $t \geq 0$, the D-WL coloring $D^{(t)}$ refines the coloring induced by the representations of any \oursacro{}, that is, on any graph $G = (V, E)$, $\forall v, w \in V, \quad D^{(t)}(v) = D^{(t)}(w) \implies h^{(t)}_v = h^{(t)}_w$, where $h^{(t)}_v$ refers to the representation of node $v$ in output from any \oursacro{} at layer $t > 0$. For $t = 0$ nodes are populated with a constant color: $\forall v \in V: D^{(0)}_v = \bar{c}$, or an appropriate encoding thereof in the case of the Dir-GNN $h^{(0)}_v = \textrm{enc}(\bar{c})$.

We proceed by induction. The base step trivially holds for $t=0$ given how nodes are initialised. As for the recursion step, let us assume the thesis hold for $t>0$; we seek to prove it also hold for $t+1$, showing that $\forall v, w \in V, \quad D^{(t+1)}(v) = D^{(t+1)}(w) \implies h^{(t+1)}_v = h^{(t+1)}_w$. $D^{(t+1)}(v) = D^{(t+1)}(w)$ implies the equality of the inputs of the $\textrm{RELABEL}$ function given it is injective. That is: $D^{(t)}(v) = D^{(t)}(w)$, $\ldblbrace D^{(t)}(u) : u \in N_\out(v) \rdblbrace = \ldblbrace D^{(t)}(u) : u \in N_\out(w) \rdblbrace$, and $\ldblbrace D^{(t)}(u) : u \in N_\inp(v) \rdblbrace = \ldblbrace D^{(t)}(u) : u \in N_\inp(w) \rdblbrace$. By the induction hypothesis, we immediately get $h^{(t)}_v = h^{(t)}_w$. Also, the induction hypothesis, along with \citep[Lemma 2]{bevilacqua2022equivariant}, gives us: $\ldblbrace h^{(t)}_u : u \in N_\out(v) \rdblbrace = \ldblbrace h^{(t)}_u : u \in N_\out(w) \rdblbrace$, and $\ldblbrace h^{(t)}_u : u \in N_\inp(v) \rdblbrace = \ldblbrace h^{(t)}_u : u \in N_\inp(w) \rdblbrace$. Given that $h^{(t)}_v = h^{(t)}_w = \bar{h}$, we also have $\ldblbrace (h^{(t)}_u, h^{(t)}_v) : u \in N_\out(v) \rdblbrace = \ldblbrace (h^{(t)}_u, h^{(t)}_w) : u \in N_\out(w) \rdblbrace$, and $\ldblbrace (h^{(t)}_u, h^{(t)}_v) : u \in N_\inp(v) \rdblbrace = \ldblbrace (h^{(t)}_u, h^{(t)}_w) : u \in N_\inp(w) \rdblbrace$: it would be sufficient, for example, to construct the well-defined function $\varphi: h \mapsto (h, \bar{h})$ and invoke~\citep[Lemma 3]{bevilacqua2022equivariant}. These all represents the only inputs to a \oursacro{} layer -- the two $\textrm{AGG}^{(t)}$ and the $\textrm{COM}^{(t)}$ function in particular. Being well defined functions, they must return equal outputs for equal inputs, so that  $h^{(t+1)}_v = h^{(t+1)}_w$.

In a similar way, we show that, when $\textrm{AGG}$ and $\textrm{COM}$ functions are injective, the opposite hold, that is, $\forall v, w \in V, h^{(t)}_v = h^{(t)}_w \implies D^{(t)}(v) = D^{(t)}(w)$. The base step holds for $t=0$ for the same motivations above. Let us assume the thesis holds for $t>0$ and seek to show that for $t+1, \forall v, w \in V, h^{(t+1)}_v = h^{(t+1)}_w \implies D^{(t+1)}(v) = D^{(t+1)}(w)$. If $h^{(t+1)}_v = h^{(t+1)}_w$, then $\mathrm{COM}^{(t)}\left(\mathbf{h}_{v}^{(t)}, \mathbf{\m}^{(t)}_{v,\out}, \mathbf{\m}^{(t)}_{v,\inp}\right) = \mathrm{COM}^{(t)}\left(\mathbf{h}_{w}^{(t)}, \mathbf{\m}^{(t)}_{w,\out}, \mathbf{\m}^{(t)}_{w,\inp}\right)$. As $\mathrm{COM}^{(t)}$ is injective, it must also hold $\mathbf{h}_{v}^{(t)} = \mathbf{h}_{w}^{(t)}$, which, by the induction hypothesis, gives $D^{(t)}(v) = D^{(t)}(w)$. Furthermore, by the same argument, we must also have $\mathbf{\m}^{(t)}_{v,\out} = \mathbf{\m}^{(t)}_{w,\out}$, and $\mathbf{\m}^{(t)}_{v,\inp} = \mathbf{\m}^{(t)}_{w,\inp}$. At this point we recall that, for any node $v$, $\mathbf{\m}^{(t)}_{v,\out} = \textrm{AGG}^{(t)}_\out(\ldblbrace (h^{(t)}_u, h^{(t)}_v) : u \in N_\out(v) \rdblbrace)$ and $\mathbf{\m}^{(t)}_{v,\inp} = \textrm{AGG}^{(t)}_\inp(\ldblbrace (h^{(t)}_u, h^{(t)}_v) : u \in N_\inp(v) \rdblbrace)$, where, by our assumption, $\textrm{AGG}^{(t)}_\inp$, $\textrm{AGG}^{(t)}_\out$ are injective. This implies the equality between the multisets in input, i.e. $\ldblbrace (h^{(t)}_u, h^{(t)}_v) : u \in N_\out(v) \rdblbrace = \ldblbrace (h^{(t)}_u, h^{(t)}_w) : u \in N_\out(w) \rdblbrace$, and $\ldblbrace (h^{(t)}_u, h^{(t)}_v) : u \in N_\inp(v) \rdblbrace = \ldblbrace (h^{(t)}_u, h^{(t)}_w) : u \in N_\inp(w) \rdblbrace$. From these equalities it clearly follows $\ldblbrace h^{(t)}_u : u \in N_\out(v) \rdblbrace = \ldblbrace h^{(t)}_u : u \in N_\out(w) \rdblbrace$, and $\ldblbrace h^{(t)}_u : u \in N_\inp(v) \rdblbrace = \ldblbrace h^{(t)}_u : u \in N_\inp(w) \rdblbrace$ -- one can invoke~\citep[Lemma 3]{bevilacqua2022equivariant} with the well defined function $\varphi: (h_1, h_2) \mapsto h_1$. Again, by the induction hypothesis, and \citep[Lemma 2]{bevilacqua2022equivariant}, we have $\ldblbrace D^{(t)}(u) : u \in N_\inp(v) \rdblbrace = \ldblbrace D^{(t)}(u) : u \in N_\inp(w) \rdblbrace$ and $\ldblbrace D^{(t)}(u) : u \in N_\out(v) \rdblbrace = \ldblbrace D^{(t)}(u) : u \in N_\out(w) \rdblbrace$. Finally, as all and only inputs to the $\textrm{RELABEL}$ function are equal, $D^{(t+1)}(v) = D^{(t+1)}(w)$. The proof then terminates: as the $\textrm{READOUT}$ function is assumed to be injective, having proved the refinement holds at the level of nodes, this is enough to also state that, if two graphs are distinguished by D-WL they are also distinguished by a \oursacro{} satisfying the injectivity assumptions above.
\end{proof}

As for the existence and implementation of these injective components, constructions can be found in \citet{DBLP:conf/iclr/XuHLJ19} and \citet{corso2020principal}.
In particular, in \citep[Lemma 5]{DBLP:conf/iclr/XuHLJ19}, the authors show that, for a countable $\mathcal{X}$, there exist maps $f: \mathcal{X} \rightarrow \mathbb{R}^n$ such that function $h: X \mapsto \sum_{x \in X} f(x)$ is injective for subsets $X \subset \mathcal{X}$ of bounded cardinality, and any multiset function $g$ can be decomposed as $g(X) = \varphi \big ( \sum_{x \in X} f(x) \big )$ for some function $\varphi$. As $\mathcal{X}$ is countable, there always exists an injection $Z: x \rightarrow \mathbb{N}$, and function $f$ can be constructed, for example, as $f(x) = N^{-Z(x)}$, with $N$ being the maximum (bounded) cardinality of subsets $X \subset \mathcal{X}$. These constructions are used to build multiset aggregators in the GIN architecture~\citep{DBLP:conf/iclr/XuHLJ19} when operating on features from a countable set and neighbourhoods of bounded size. Under the same assumptions, the same constructions can be readily adapted to express the aggregators $\text{AGG}^{(k)}_\out, \text{AGG}^{(k)}_\inp$ as well as $\text{READOUT}$ in our \oursacro{}. 
Similarly to the above, under the same assumptions, injective maps for elements $(c, X), c \in \mathcal{X}, X \subset \mathcal{X}$ can be constructed as $h(c, X) = (1+\epsilon)f(c) + \sum_{x \in X}f(x)$ for infinitely many choice of $\epsilon$, including all irrational numbers, and any function $g$ on couples $(c, X)$ can be decomposed as $g(c, X) = \varphi \big ( (1+\epsilon)f(c) + \sum_{x \in X}f(x) \big )$~\citep[Corollary 6]{DBLP:conf/iclr/XuHLJ19}. The same approach can be extended to our use-case. In fact, for irrationals $\epsilon_X, \epsilon_Y$, an injection on triple $(c, X, Y)$ (with $c \in \mathcal{X}, X, Y \subset \mathcal{X}$ of bounded size and $\mathcal{X}$ countable) can be built as $h(c, X, Y) = \ell \big ( (1+\epsilon_X)f_X(c) + \sum_{x \in X} f_X(x), (1+\epsilon_Y)f_Y(c) + \sum_{y \in Y} f_Y(y) \big )$, where $\ell$ is an injection on a countable set and $(1+\epsilon_X)f_X(c) + \sum_{x \in X} f_X(x), (1+\epsilon_Y)f_Y(c) + \sum_{y \in Y} f_Y(y)$ realise injections over couples $(c, X), (c, Y)$ as described above. This construction can be used to express the $\text{COM}^{(k)}$ components of \oursacro{}. 
In practice, in view of the Universal Approximation Theorem (UAT)~\citep{HORNIK1989359}, \citet{DBLP:conf/iclr/XuHLJ19} propose to use Multi-Layer Perceptrons (MLPs) to learn the required components described above, functions $f$ and $\varphi$ in particular. We note that, in order to resort to the original statement of the UAT, this approach additionally requires boundedness of set $\mathcal{X}$ itself. Similar practical parameterizations can be used to build our desired \oursacro{} layers. Last, we refer readers to~\citep{corso2020principal} for constructions which can be adopted in the case where initial node features have a continuous, uncountable support.

\subsection{Comparison with MPNNs}

In this subsection we prove \cref{thm:dirgnn-strictly-more-expressive-than-mpnn}, which we restate more formally and split into two separate parts, one regarding MPNN-D and the other regarding MPNN-U. We start by proving that \oursacro{} is stricly more expressive than MPNN-D, i.e an MPNN which operates on directed graphs by only propagating messages according to the directionality of edges:

\begin{theorem}\label{thm:dirgnn-strictly-more-expressive-than-mpnn-d}
    $\mathcal{M}_\mathrm{\oursacro{}}$ is strictly more powerful than $\mathcal{M}_\mathrm{MPNN-D}$.
\end{theorem}

We begin by first proving the following lemmas:

\begin{lemma}\label{lemma:dirgnn-at-least-as-expressive-as-mpnn-d}
    $\mathcal{M}_\mathrm{\oursacro{}}$ is at least as expressive as $\mathcal{M}_\mathrm{MPNN-D}$ ($\mathcal{M}_\mathrm{\oursacro{}} \sqsubseteq \mathcal{M}_\mathrm{MPNN-D}$). 
\end{lemma}

\begin{proof}[Proof of Lemma~\ref{lemma:dirgnn-at-least-as-expressive-as-mpnn-d}]
    We prove this Lemma by noting that the \oursacro{} architecture generalizes that of an MPNN-D, so that a \oursacro{} model can (learn to) simulate a standard MPNN-D by adopting particular weights. Specifically, \oursacro{} defaults to MPNN-D (which only sends messages along the out edges) if $\mathrm{COM}^{(k)}\left(\mathbf{x}_{i}^{(k-1)},\mathbf{\m}^{(k)}_{i,\inp}, \mathbf{\m}^{(k)}_{i,\out}\right) = \mathrm{COM}^{(k)}\left(\mathbf{x}_{i}^{(k-1)}, \mathbf{\m}^{(k)}_{i,\out}\right)$, i.e. $\mathrm{COM}$ ignores in-messages, and the two readout modules coincide. Importantly, the direct implication of the above is that whenever an MPNN-D model distinguishes two graphs, then there exists a \oursacro{} which can implement such a model and then discriminate the two graphs as well.
\end{proof}

\begin{figure*}
    \centering
    \includegraphics[width=.7\textwidth]{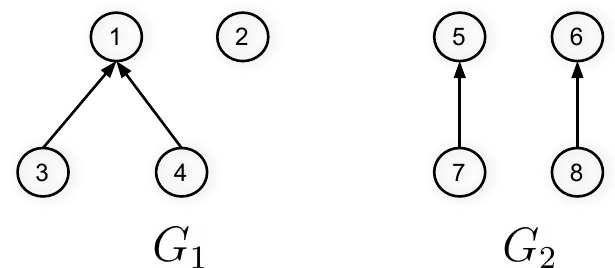}
    \vspace{3mm}
    \caption{Two non-isomorphic directed graphs that cannot be distinguished by any MPNN-D model but can be distinguished by \oursacro{}.}
    \label{fig:mpnn-d_fails}
\end{figure*}

\begin{lemma}\label{lemma:exists-graphs-distinguished-by-dirgnn-but-not-by-mpnn-d}
    There exist graph pairs discriminated by a \oursacro{} model which are not discriminated by any MPNN-D model.
\end{lemma}
\begin{proof}[Proof of Lemma~\ref{lemma:exists-graphs-distinguished-by-dirgnn-but-not-by-mpnn-d}]
    Let $G_1$ and $G_2$ be the non-isomorphic graphs illustrated in~\cref{fig:mpnn-d_fails}. To confirm that they are not isomorphic, simply note that node $1$ in $G_1$ has an in-degree of two, while no node in $G_2$ has an in-degree of two. 
    
    To prove that no MPNN-D model can distinguish between the two graphs, we will show that any MPNN-D induce the same coloring for the two graphs. In particular, we will show that, if $C^{(t)}(v)$ refers to the representation a MPNN-D computes for node $v$ at time step $t$, then $C^{(t)}(1) = C^{(t)}(2) = C^{(t)}(5) = C^{(t)}(6)$ and $C^{(t)}(3) = C^{(t)}(4) = C^{(t)}(7) = C^{(t)}(8)$ for any $t \geq 0$.

    We proceed by induction. The base step trivially holds for $t=0$ given that nodes are all initialised with the same color. As for the inductive step, let us assume that the statement holds for $t$ and prove that it also holds for $t+1$.
    Assume $C^{(t)}(1) = C^{(t)}(2) = C^{(t)}(5) = C^{(t)}(6)$ and $C^{(t)}(3) = C^{(t)}(4) = C^{(t)}(7) = C^{(t)}(8)$ (induction hypothesis).
    Then we have:
    \begin{align*}
        C^{(t+1)}(1) &= \mathrm{COM}^{(t)}\left(C^{(t)}(1), \mathrm{AGG}^{(t)}(\ldblbrace  \rdblbrace)\right) \\
        C^{(t+1)}(2) &= \mathrm{COM}^{(t)}\left(C^{(t)}(2), \mathrm{AGG}^{(t)}(\ldblbrace  \rdblbrace)\right) \\
        C^{(t+1)}(5) &= \mathrm{COM}^{(t)}\left(C^{(t)}(5), \mathrm{AGG}^{(t)}(\ldblbrace  \rdblbrace)\right) \\
        C^{(t+1)}(6) &= \mathrm{COM}^{(t)}\left(C^{(t)}(6), \mathrm{AGG}^{(t)}(\ldblbrace  \rdblbrace)\right) \\
    \end{align*}
    The induction hypothesis then gives us that $C^{(t+1)}(1) = C^{(t+1)}(2) = C^{(t+1)}(5) = C^{(t+1)}(6)$.
    As for the other nodes, we have:
    \begin{align*}
        C^{(t+1)}(3) &= \mathrm{COM}^{(t)}\left(C^{(t)}(3), \mathrm{AGG}^{(t)}(\ldblbrace (C^{(t)}(3), C^{(t)}(1)) \rdblbrace)\right) \\
        C^{(t+1)}(4) &= \mathrm{COM}^{(t)}\left(C^{(t)}(4), \mathrm{AGG}^{(t)}(\ldblbrace (C^{(t)}(4), C^{(t)}(1)) \rdblbrace)\right) \\
        C^{(t+1)}(7) &= \mathrm{COM}^{(t)}\left(C^{(t)}(7), \mathrm{AGG}^{(t)}(\ldblbrace (C^{(t)}(7), C^{(t)}(5)) \rdblbrace)\right) \\
        C^{(t+1)}(8) &= \mathrm{COM}^{(t)}\left(C^{(t)}(8), \mathrm{AGG}^{(t)}(\ldblbrace (C^{(t)}(8), C^{(t)}(6)) \rdblbrace)\right) \\
    \end{align*}
    The induction hypothesis then gives us that $C^{(t+1)}(3) = C^{(t+1)}(4) = C^{(t+1)}(7) = C^{(t+1)}(8)$. Importantly, the above holds for any parameters of the $\textrm{COM}^{(t)}$ and $\textrm{AGG}^{(t)}$ functions. As \emph{any} MPNN-D will always compute the same set of node representations for the two graphs, it follows that no MPNN-D can disambiguate between the two graphs, no matter the way they are aggregated. To conclude our proof, we show that there exists \oursacro{} models that can discriminate the two graphs. In view of~\cref{thm:dirgnn-as-expressive-as-d-wl}, it is enough to show that the two graphs are disambiguated by D-WL.
    Applying D-WL to the two graphs leads to different colorings after two iterations (see~\cref{tab:wl-d-1}), so the D-WL algorithm terminates deeming the two graphs non-isomorphic. Then, by~\cref{thm:dirgnn-as-expressive-as-d-wl}, there exist \oursacro{}s which distinguish them. In fact, it is easy to even construct simple 1-layer architecture that can assign the two graphs distinct representations, an exercise which we leave to the reader. Importantly, note how \oursacro{} can distinguish between the two graphs hinging on the discrimination of non-isomorphic nodes such as 1, 2, something no MPNN-D is capable of doing.
\end{proof}

\input{tables/wl-d-1}

With the two results above prove \cref{thm:dirgnn-strictly-more-expressive-than-mpnn-d}.

\begin{proof}[Proof of Theorem~\ref{thm:dirgnn-strictly-more-expressive-than-mpnn-d}]
    The theorem follows directly from Lemmas~\ref{lemma:dirgnn-at-least-as-expressive-as-mpnn-d} and~\ref{lemma:exists-graphs-distinguished-by-dirgnn-but-not-by-mpnn-d}.
\end{proof}

\begin{figure*}
    \centering
    \includegraphics[width=.55\textwidth]{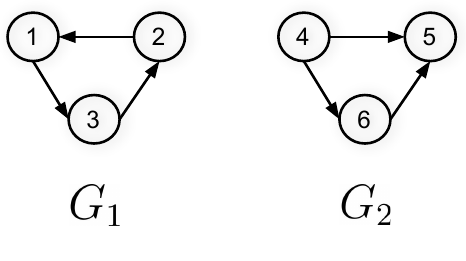}
    \caption{Two non-isomorphic directed graphs that cannot be distinguished by any MPNN-U model but can be distinguished by \oursacro{}.}
    \label{fig:mpnn-u_fails}
\end{figure*}

\input{tables/wl-d-2}

Next, we focus on the comparison with MPNN-U, i.e. an MPNN on the undirected form of the graph: 

\begin{theorem}\label{thm:dirgnn-strictly-more-expressive-than-mpnn-u}
    $\mathcal{M}_\mathrm{\oursacro{}}$ is strictly more expressive than $\mathcal{M}_\mathrm{MPNN-U}$.
\end{theorem}

Instrumental to us is to consider a variant of the 1-WL test MPNN-U can be regarded as the neural counterpart of. In the following we will show that such a variant, which we call U-WL, generates colorings which refine the ones induced by \emph{any} MPNN-U and that, in turn, are refined by the D-WL test. In view of \Cref{thm:dirgnn-as-expressive-as-d-wl}, this will be enough to show that there exists Dir-GNNs refining any MPNN-U instantiation, so that, ultimately, $\mathcal{M}_{\textrm{Dir-GNN}} \sqsubseteq \mathcal{M}_{\textrm{MPNN-U}}$.

\begin{lemma}\label{lemma:dirgnn-at-least-as-expressive-as-mpnn-u}
    $\mathcal{M}_\mathrm{\oursacro{}}$ is at least as expressive as $\mathcal{M}_\mathrm{MPNN-U}$ ($\mathcal{M}_\mathrm{\oursacro{}} \sqsubseteq \mathcal{M}_\mathrm{MPNN-U}$). 
\end{lemma}

\begin{proof}[Proof of Lemma~\ref{lemma:dirgnn-at-least-as-expressive-as-mpnn-u}]
    Let us start by introducing the U-WL test, which, on an undirected graph, refines node colors as:
    $$  
        A^{(t+1)}(v) = \textrm{RELABEL} \big ( A^{(t)}(v), \ldblbrace  A^{(t)}(u) : u \in N_{\out}(v) \rdblbrace \cup \ldblbrace  A^{(t)}(u) : u \in N_{\inp}(v) \rdblbrace \big ),
    $$
    that is, by the gathering neighbouring colors from each incident edge, independent of its direction. It is easy to show that U-WL generates a coloring that, at any round $t \geq 0$ is refined by the coloring generated by D-WL, i.e., for any graph $G = (V, E)$ it holds that $\forall v, w \in V, D^{(t)}(v) = D^{(t)}(w) \implies A^{(t)}(v) = A^{(t)}(w)$, where $D$ refers to the coloring of D-WL. Again, proceeding by induction, we have the following. First, the base step hold trivially for $t = 0$. We assume the thesis holds true for $t$ and seek to show it also holds for $t+1$. If $D^{(t+1)}(v) = D^{(t+1)}(w)$ then, by the injectivity of $\textrm{RELABEL}$ we must have $D^{(t)}(v) = D^{(t)}(w)$, which implies $A^{(t)}(v) = A^{(t)}(w)$ via the induction hypothesis. Additionally, we have $\ldblbrace D^{(t)}(u) : u \in N_\out(v) \rdblbrace = \ldblbrace D^{(t)}(u) : u \in N_\out(w) \rdblbrace$, and $\ldblbrace D^{(t)}(u) : u \in N_\inp(v) \rdblbrace = \ldblbrace D^{(t)}(u) : u \in N_\inp(w) \rdblbrace$ which, by the induction hypothesis and~\citep[Lemma 2]{bevilacqua2022equivariant}, gives $\mathcal{A}^{(t)}_{\out, v} = \ldblbrace A^{(t)}(u) : u \in N_\out(v) \rdblbrace = \ldblbrace A^{(t)}(u) : u \in N_\out(w) \rdblbrace = \mathcal{A}^{(t)}_{\out,w}$, and $\mathcal{A}^{(t)}_{\inp, v} = \ldblbrace A^{(t)}(u) : u \in N_\inp(v) \rdblbrace = \ldblbrace A^{(t)}(u) : u \in N_\inp(w) \rdblbrace = \mathcal{A}^{(t)}_{\inp,w}$. From these equalities we then derive $\mathcal{A}^{(t)}_{\out,v} \cup \mathcal{A}^{(t)}_{\inp,v} = \mathcal{A}^{(t)}_{\out,w} \cup \mathcal{A}^{(t)}_{\inp,w}$. Indeed, let us suppose that, instead, $\mathcal{A}^{(t)}_{\out,v} \cup \mathcal{A}^{(t)}_{\inp,v} \neq \mathcal{A}^{(t)}_{\out,w} \cup \mathcal{A}^{(t)}_{\inp,w}$ and that, w.l.o.g., this is due by the existence of a color $\bar{a}$ such that its number of appearances in $\mathcal{A}^{(t)}_{\out,w} \cup \mathcal{A}^{(t)}_{\inp,w}$ is larger than that in $\mathcal{A}^{(t)}_{\out,v} \cup \mathcal{A}^{(t)}_{\inp,v}$. We write $\#^{\mathcal{A}^{(t)}_{\out,w} \cup \mathcal{A}^{(t)}_{\inp,w}}(\bar{a}) > \#^{\mathcal{A}^{(t)}_{\out,v} \cup \mathcal{A}^{(t)}_{\inp,v}}(\bar{a})$. Then, as these are all multisets, we can rewrite $\#^{\mathcal{A}^{(t)}_{\out,w}}(\bar{a}) + \#^{\mathcal{A}^{(t)}_{\inp,w}}(\bar{a}) > \#^{\mathcal{A}^{(t)}_{\out,v}}(\bar{a}) + \#^{\mathcal{A}^{(t)}_{\inp,v}}(\bar{a})$. Since $\mathcal{A}^{(t)}_{\out,w} = \mathcal{A}^{(t)}_{\out,v}$, we must have that $\#^{\mathcal{A}^{(t)}_{\out,w}}(\bar{a}) = \#^{\mathcal{A}^{(t)}_{\out,v}}(\bar{a})$, which leads to necessarily having $\#^{\mathcal{A}^{(t)}_{\inp,w}}(\bar{a}) \neq \#^{\mathcal{A}^{(t)}_{\inp,v}}(\bar{a})$. However, this entails a contradiction, because by hypothesis we had that $\mathcal{A}^{(t)}_{\inp,w} = \mathcal{A}^{(t)}_{\inp,v}$. Last, given that $A^{(t)}(v) = A^{(t)}(w)$, and $\mathcal{A}^{(t)}_{\out,v} \cup \mathcal{A}^{(t)}_{\inp,v} = \mathcal{A}^{(t)}_{\out,w} \cup \mathcal{A}^{(t)}_{\inp,w}$, being the only inputs to the $\textrm{RELABEL}$ function in U-WL, then we also have that $A^{(t+1)}(v) = A^{(t+1)}(w)$, concluding the proof of the refinement.

    Now, in view of \Cref{thm:dirgnn-as-expressive-as-d-wl}, it is sufficient to show that U-WL refines the coloring induced by any MPNN-U; the lemma will then follow by transitivity. We want to show that $t \geq 0$, the U-WL coloring $A^{(t)}$ refines the coloring induced by the representations of any MPNN-U, that is, on any graph $G = (V, E)$, $\forall v, w \in V, \quad A^{(t)}(v) = A^{(t)}(w) \implies h^{(t)}_v = h^{(t)}_w$, with $h^{(t)}_v$ referring to the representation an MPNN-U assigns to node $v$ at time step $t$. The thesis is easily proved. It clearly holds for $t=0$ if initial node representations are produced by a well-defined function $\textrm{enc}: \bar{c} \mapsto \textrm{enc}(\bar{c})$. Then, if we assume the thesis holds for $t > 0$, we can show it also holds for $t+1$. Indeed, if $A^{(t+1)}(v) = A^{(t+1)}(w)$, from the injectivity of $\textrm{RELABEL}$, it follows that $A^{(t)}(v) = A^{(t)}(w)$ and $\mathcal{A}^{(t)}_{\out,v} \cup \mathcal{A}^{(t)}_{\inp,v} = \mathcal{A}^{(t)}_{\out,w} \cup \mathcal{A}^{(t)}_{\inp,w}$. By the induction hypothesis, we have $h^{(t)}_v = h^{(t)}_w$ and, jointly due to~\citep[Lemma 2]{bevilacqua2022equivariant}, $\ldblbrace h^{(t)}_u : u \in N_\out(v)\rdblbrace \cup \ldblbrace h^{(t)}_u : u \in N_\inp(v)\rdblbrace = \ldblbrace h^{(t)}_u : u \in N_\out(w)\rdblbrace \cup \ldblbrace h^{(t)}_u : u \in N_\inp(w)\rdblbrace$. Given that $h^{(t)}_v = h^{(t)}_w = \bar{h}$, it also clearly holds that $\ldblbrace (h^{(t)}_u, h^{(t)}_v) : u \in N_\out(v)\rdblbrace \cup \ldblbrace (h^{(t)}_u, h^{(t)}_v) : u \in N_\inp(v)\rdblbrace = \ldblbrace (h^{(t)}_u, h^{(t)}_w) : u \in N_\out(w)\rdblbrace \cup \ldblbrace (h^{(t)}_u, h^{(t)}_w) : u \in N_\inp(w)\rdblbrace$ -- it is sufficient to construct the well-defined function $\varphi: h \mapsto (h, \bar{h})$ and invoke~\citep[Lemma 3]{bevilacqua2022equivariant}. These are the inputs to the well-defined functions constituting the update equations of an MPNN-U architecture, eventually entailing $h^{(t+1)}_v = h^{(t+1)}_w$.
\end{proof}

Now, with the following lemma we show that, not only $\mathcal{M}_\mathrm{\oursacro{}}$ is at least as expressive as $\mathcal{M}_\mathrm{MPNN-U}$, there actually exist pairs of graphs distinguished by former family but not by the latter.

\begin{lemma}\label{lemma:exists-graphs-distinguished-by-dirgnn-but-not-by-mpnn-u}
    There exist graph pairs distinguished by a \oursacro{} model which are not distinguished by any MPNN-U model.
\end{lemma}
\begin{proof}[Proof of Lemma~\ref{lemma:exists-graphs-distinguished-by-dirgnn-but-not-by-mpnn-u}]
    Let $G_1$ and $G_2$ be the non-isomorphic graphs illustrated in \cref{fig:mpnn-u_fails}. From~\cref{tab:wl-u} we observe that U-WL is not able to distinguish between the two graphs, as after the first iteration all nodes still have the same color: the U-WL is at convergence and terminates concluding that the two graphs are possibly isomorphic. From \cref{lemma:dirgnn-at-least-as-expressive-as-mpnn-u}, we conclude that no MPNN-U can distinguish between the two graphs. On the other hand, applying D-WL to the two graphs leads to different colorings after two iterations (see~\cref{tab:wl-d-2}, so the D-WL algorithm terminates deeming the two graphs non-isomorphic. Then, by~\cref{thm:dirgnn-as-expressive-as-d-wl}, there exists \oursacro{}s which distinguish them. In fact, simple \oursacro{} architectures which distinguish the two graphs are easy to construct. Importantly, we note, again, how these architectures distinguish between the two graphs by disambiguating non-isomorphic nodes such as 4, 6, something no MPNN-U is capable of doing.
\end{proof}

Last, the two results above are sufficient to prove~\cref{thm:dirgnn-strictly-more-expressive-than-mpnn-u}.

\begin{proof}[Proof of Theorem~\ref{thm:dirgnn-strictly-more-expressive-than-mpnn-u}]
    The theorem follows directly from Lemmas~\ref{lemma:dirgnn-at-least-as-expressive-as-mpnn-u} and~\ref{lemma:exists-graphs-distinguished-by-dirgnn-but-not-by-mpnn-u}.
\end{proof}

\input{tables/wl-u.tex}

%% file: tables/wl-d-1.tex
\begin{table*}[t]
\vskip 0.15in
\begin{center}
\begin{small}
\begin{sc}
\begin{tabular}{l||cccc|cccc}
\toprule
Iteration     & Node 1  & Node 2 & Node 3 & Node 4 & Node 5  & Node 6 & Node 7 & Node 8 \\
\midrule
1             &  $A$      &  $A$     &  $A$     &  $A$     &  $A$      &  $A$     &  $A$     &  $A$     \\
2             &  $B$      &  $C$     &  $C$     &  $D$     &  $E$      &  $E$     &  $C$     &  $C$     \\
% 3             &  F      &  G     &  G     &  H     &  I      &  I     &  L     &  L     \\
% 4             &  M      &  N     &  N     &  O     &  P      &  P     &  Q     &  Q     \\
\bottomrule
\end{tabular}

\begin{tabular}{lcccr}
\toprule
$M(v)$                                               & $\mathrm{RELABEL(M(v))}$   \\
\midrule
initialize                                           &   $A$                        \\
% Iteration 1
$(A, \ldblbrace \rdblbrace, \ldblbrace (A, A), (A, A) \rdblbrace)$             &   $B$                        \\
$(A, \ldblbrace (A, A) \rdblbrace, \ldblbrace \rdblbrace)$                     &   $C$                        \\
$(A, \ldblbrace  \rdblbrace, \ldblbrace \rdblbrace)$                           &   $D$                        \\
$(A, \ldblbrace  \rdblbrace, \ldblbrace (A, A) \rdblbrace)$                    &   $E$                        \\
% % Iteration 2
% (B, \ldblbrace  \rdblbrace, \ldblbrace C, C \rdblbrace)            &   F                        \\
% (C, \ldblbrace B \rdblbrace, \ldblbrace \rdblbrace)                &   G                        \\
% (D, \ldblbrace  \rdblbrace, \ldblbrace \rdblbrace)                 &   H                        \\
% (E, \ldblbrace  \rdblbrace, \ldblbrace C \rdblbrace)               &   I                        \\
% (C, \ldblbrace E \rdblbrace, \ldblbrace \rdblbrace)                &   L                        \\
% % Iteration 3
% (F, \ldblbrace  \rdblbrace, \ldblbrace G, G \rdblbrace)            &   M                        \\
% (G, \ldblbrace F \rdblbrace, \ldblbrace \rdblbrace)                &   N                        \\
% (H, \ldblbrace  \rdblbrace, \ldblbrace \rdblbrace)                 &   O                        \\
% (I, \ldblbrace  \rdblbrace, \ldblbrace L \rdblbrace)               &   P                        \\
% (L, \ldblbrace I \rdblbrace, \ldblbrace \rdblbrace)                &   Q                        \\
\bottomrule
\end{tabular}

\end{sc}
\end{small}
\end{center}
\caption{Node colorings at different iterations, as well as the $\mathrm{RELABEL}$ hash function, when applying D-WL to the two graphs in~\cref{fig:mpnn-d_fails}.}
\label{tab:wl-d-1}
\end{table*}

%% file: tables/wl-d-2.tex
\begin{table*}[t]
\vskip 0.15in
\begin{center}
\begin{small}
\begin{sc}
\begin{tabular}{l||ccc|ccc}
\toprule
Iteration     & Node 1  & Node 2 & Node 3 & Node 4 & Node 5  & Node 6 \\
\midrule
1             &  $A$      &  $A$     &  $A$     &  $A$     &  $A$      &  $A$     \\
2             &  $B$      &  $B$     &  $B$     &  $C$     &  $B$      &  $D$     \\
% 3             &  E      &  E     &  E     &  F     &  G      &  H     \\
% 4             &  I      &  I     &  I     &  L     &  M      &  N     \\
\bottomrule
\end{tabular}

\begin{tabular}{lcccr}
\toprule
$M(v)$                                               & $\mathrm{RELABEL(M(v))}$   \\
\midrule
initialize                                           &   $A$                        \\
% Iteration 1
$(A, \ldblbrace (A, A) \rdblbrace, \ldblbrace (A, A) \rdblbrace)$              &   $B$                        \\
$(A, \ldblbrace (A, A), (A, A) \rdblbrace, \ldblbrace \rdblbrace)$             &   $C$                        \\
$(A, \ldblbrace  \rdblbrace, \ldblbrace (A, A), (A, A)\rdblbrace)$             &   $D$                        \\
% % Iteration 2
% (B, \{\{ B \}\}, \{\{ B \}\})              &   E                        \\
% (C, \{\{ B, D \}\}, \{\{ \}\})             &   F                        \\
% (B, \{\{ D \}\}, \{\{ C \}\})              &   G                        \\
% (D, \{\{  \}\}, \{\{ C, B\}\})             &   H                        \\
% % Iteration 3
% (E, \{\{ E \}\}, \{\{ E \}\})              &   I                        \\
% (F, \{\{ G, H \}\}, \{\{ \}\})             &   L                        \\
% (G, \{\{ F \}\}, \{\{ H \}\})              &   M                        \\
% (H, \{\{  \}\}, \{\{ F, G\}\})             &   N                        \\
\bottomrule
\end{tabular}

\end{sc}
\end{small}
\end{center}
\caption{Node colorings at different iterations, , as well as the $\mathrm{RELABEL}$ hash function, when applying D-WL to the two graphs in~\cref{fig:mpnn-u_fails}.}
\label{tab:wl-d-2}
\end{table*}

%% file: tables/wl-u.tex
\begin{table*}[t]
\begin{center}
\begin{small}
\begin{sc}
\begin{tabular}{l||cccc|cccc}
\toprule
Iteration     & Node 1  & Node 2 & Node 3 & Node 4 & Node 5  & Node 6 & Node 7 & Node 8 \\
\midrule
1             &  $A$      &  $A$     &  $A$     &  $A$     &  $A$      &  $A$     &  $A$     &  $A$     \\
2             &  $E$      &  $E$     &  $E$     &  $E$     &  $E$      &  $E$     &  $E$     &  $E$     \\
\bottomrule
\end{tabular}

\begin{tabular}{lcccr}
\toprule
$M(v)$                                                      & $\mathrm{RELABEL(M(v))}$   \\
\midrule
initialize                                                  &   $A$                        \\
% Iteration 1
$(A, \ldblbrace (A, A) \rdblbrace, \ldblbrace (A, A), (A, A) \rdblbrace)$             &   $E$                        \\
\bottomrule
\end{tabular}

\end{sc}
\end{small}
\end{center}
\caption{Node colorings at different iterations, as well as the $\mathrm{RELABEL}$ hash function, when applying U-WL to the two graphs in~\cref{fig:mpnn-u_fails}.}
\label{tab:wl-u}
\end{table*}

%% file: tables/dataset_statistics.tex
\begin{table*}[t]
\begin{center}
\begin{small}
\begin{sc}
    \resizebox{1.\textwidth}{!}{%
    \begin{tabular}{lcccccr}
    \toprule
    Dataset       & \# Nodes   & \# Edges    & \# Feat. & \# C & Unidirectional Edges & Edge hom. \\
    \midrule
    citeseer-full &     4,230  &      5,358  &     602  &   6  & 99.61\%     & 0.949 \\
    cora-ml       &     2,995  &      8,416  &    2,879 &   7  & 96.84\%     & 0.792 \\
    ogbn-arxiv    &   169,343  &   1,166,243 &     128  &  40  & 99.27\%     & 0.655 \\
    chameleon     &     2,277  &     36,101  &    2,325 &   5  & 85.01\%     & 0.235 \\
    squirrel      &     5,201  &    217,073  &    2,089 &   5  & 90.60\%     & 0.223 \\
    arxiv-year    &   169,343  &   1,166,243 &     128  &  40  & 99.27\%     & 0.221 \\
    snap-patents  &  2,923,922 &  13,975,791 &     269  &   5  & 99.98\%     & 0.218 \\
    roman-empire  &     22,662 &      44,363 &     300  &  18  & 65.24\%     & 0.050 \\
    \bottomrule
    \end{tabular}
    }
\end{sc}
\end{small}
\end{center}
\caption{Statistics of the datasets used in this paper.}
\label{tab:datasets-statistics}
\end{table*}

%% file: tables/best_hyperparameters.tex
\begin{table*}[t]
\begin{center}
\begin{small}
\begin{sc}
    \resizebox{1.\textwidth}{!}{
    \begin{tabular}{lcccccccr}
        \toprule
        Dataset       & model\_type &   lr   & \# hidden\_dim & \# num\_layers & jk     & norm  & dropout & $\alpha$ \\
        \midrule
        chameleon     &  Dir-GCN    & 0.005  &     128         &    5           &  max  & True  &  0      & 1. \\
        squirrel      &  Dir-GCN    & 0.01   &     128         &    4           &  max  & True  &  0      & 1. \\
        arxiv-year    &  Dir-GCN    & 0.005  &     256         &     6          &  cat  & False &  0      & 0.5 \\
        snap-patents  &  Dir-GCN    & 0.01   &     32          &     5          &   max & True  &  0      & 0.5 \\
        roman-empire  &  Dir-SAGE   & 0.01   &     256         &     5          &   cat & False &  0.2    & 0.5 \\
        \bottomrule
    \end{tabular}
    }
\end{sc}
\end{small}
\end{center}
\caption{Best hyperparameters for each dataset, determined through grid search, for our model.}
\label{tab:best_hyperparameters}
\end{table*}

%% file: tables/zero_degrees.tex
\begin{table*}[t]
\begin{center}
\begin{small}
\begin{sc}
\resizebox{0.6\textwidth}{!}{%
\begin{tabular}{lrrr}
\toprule
{} &  in\_degree &  out\_degree &  total\_degree \\
\midrule
cora\_ml               &      41.70\% &       11.65\% &          0.00\% \\
citeseer\_full         &      63.45\% &       21.35\% &          0.00\% \\
ogbn-arxiv            &      36.62\% &       10.30\% &          0.00\% \\
chameleon             &      62.06\% &        0.00\% &          0.00\% \\
squirrel              &      57.60\% &        0.00\% &          0.00\% \\
arxiv-year            &      36.62\% &       10.30\% &          0.00\% \\
snap-patents          &      23.38\% &       30.16\% &          6.09\% \\
directed-roman-empire &       0.00\% &        0.00\% &          0.00\% \\
\bottomrule
\end{tabular}}
\end{sc}
\end{small}
\end{center}
\caption{Percentage of nodes with either in-, out- or total-degree equal to zero.}
\label{tab:zero_degrees}
\end{table*}

%% file: tables/full_direction_ablation.tex
\begin{table*}[t]
\begin{center}
\begin{small}
\begin{sc}
\resizebox{1.\textwidth}{!}{%
\begin{tabular}{lccc:ccccr}
\toprule
& citeseer\_full & cora\_ml & ogbn-arxiv & chameleon & squirrel & arxiv-year & snap-patents & roman-empire \\ 
hom. & 0.949 &  0.792 &  0.655 &  0.235 &  0.223 &  0.221 & 0.218 & 0.050 \\ 
\midrule
gcn             &   93.37$\pm$0.22 &  84.37$\pm$1.52 &  68.39$\pm$0.01 &  71.12$\pm$2.28 &  62.71$\pm$2.27 &  46.28$\pm$0.39 &  51.02$\pm$0.07 &           56.23$\pm$0.37 \\
dir-gcn($\alpha$=0.0)  &   93.21$\pm$0.41 &  84.45$\pm$1.69 &  23.70$\pm$0.20 &  29.78$\pm$1.27 &  33.03$\pm$0.78 &  50.51$\pm$0.45 &  51.71$\pm$0.06 &           42.69$\pm$0.41 \\
dir-gcn($\alpha$=1.0)  &   93.44$\pm$0.59 &  83.81$\pm$1.44 &  62.93$\pm$0.21 &  78.77$\pm$1.72 &  74.43$\pm$0.74 &  50.52$\pm$0.09 &  62.24$\pm$0.04 &           45.52$\pm$0.14 \\
dir-gcn($\alpha$=0.5)  &   92.97$\pm$0.31 &  84.21$\pm$2.48 &  66.66$\pm$0.02 &  72.37$\pm$1.50 &  67.82$\pm$1.73 &  59.56$\pm$0.16 &  71.32$\pm$0.06 &           74.54$\pm$0.71 \\
\hdashline
sage            &   94.15$\pm$0.61 &  86.01$\pm$1.56 &  67.78$\pm$0.07 &  61.14$\pm$2.00 &  42.64$\pm$1.72 &  44.05$\pm$0.02 &  52.55$\pm$0.10 &           72.05$\pm$0.41 \\
dir-sage($\alpha$=0.0) &   94.05$\pm$0.25 &  85.84$\pm$2.09 &  52.08$\pm$0.17 &  48.33$\pm$2.40 &  35.31$\pm$0.52 &  47.45$\pm$0.32 &  52.53$\pm$0.03 &           76.47$\pm$0.14 \\
dir-sage($\alpha$=1.0) &   93.97$\pm$0.67 &  85.73$\pm$0.35 &  65.14$\pm$0.03 &  64.47$\pm$2.27 &  46.05$\pm$1.16 &  50.37$\pm$0.09 &  61.59$\pm$0.05 &           68.81$\pm$0.48 \\
dir-sage($\alpha$=0.5) &   94.14$\pm$0.65 &  85.81$\pm$1.18 &  65.06$\pm$0.28 &  60.22$\pm$1.16 &  43.29$\pm$1.04 &  55.76$\pm$0.10 &  70.26$\pm$0.14 &           79.10$\pm$0.19 \\
\hdashline
gat             &   94.53$\pm$0.48 &  86.44$\pm$1.45 &  69.60$\pm$0.01 &  66.82$\pm$2.56 &  56.49$\pm$1.73 &  45.30$\pm$0.23 &          OOM &           49.18$\pm$1.35 \\
dir-gat($\alpha$=0.0)  &   94.48$\pm$0.52 &  86.13$\pm$1.58 &  52.57$\pm$0.05 &  40.44$\pm$3.11 &  28.28$\pm$1.02 &  46.01$\pm$0.06 &          OOM &           53.58$\pm$2.51 \\
dir-gat($\alpha$=1.0)  &   94.08$\pm$0.69 &  86.21$\pm$1.40 &  66.50$\pm$0.16 &  71.40$\pm$1.63 &  67.53$\pm$1.04 &  51.58$\pm$0.19 &          OOM &           56.24$\pm$0.41 \\
dir-gat($\alpha$=0.5)  &   94.12$\pm$0.49 &  86.05$\pm$1.71 &  66.44$\pm$0.41 &  55.57$\pm$1.02 &  37.75$\pm$1.24 &  54.47$\pm$0.14 &          OOM &           72.25$\pm$0.04 \\

\bottomrule
\end{tabular}}
\end{sc}
\end{small}
\end{center}
\caption{Ablation study comparing base MPNNs on the undirected graph versus their \oursacro{} extensions on the directed graph. We conducted experiments with $\alpha=0$ (only in-edges), $\alpha=1$ (only out-edges), and $\alpha=0.5$ (both in- and out-edges, but with different weight matrices). For homophilic datasets (to the left of the dashed line), incorporating directionality does not significantly enhance or may slightly impair performance. However, for heterophilic datasets (to the right of the dashed line), the inclusion of directionality substantially improves accuracy.}
\label{tab:full_direction_ablation}
\end{table*}

%% file: log_2023.bbl
\begin{thebibliography}{61}
\providecommand{\natexlab}[1]{#1}
\providecommand{\url}[1]{\texttt{#1}}
\expandafter\ifx\csname urlstyle\endcsname\relax
  \providecommand{\doi}[1]{doi: #1}\else
  \providecommand{\doi}{doi: \begingroup \urlstyle{rm}\Url}\fi

\bibitem[Zhou et~al.(2018)Zhou, Cui, Hu, Zhang, Yang, Liu, Wang, Li, and Sun]{zhou2018gnn}
Jie Zhou, Ganqu Cui, Shengding Hu, Zhengyan Zhang, Cheng Yang, Zhiyuan Liu, Lifeng Wang, Changcheng Li, and Maosong Sun.
\newblock Graph neural networks: A review of methods and applications, 2018.

\bibitem[Kipf and Welling(2016)]{kipf2016semi}
Thomas~N Kipf and Max Welling.
\newblock Semi-supervised classification with graph convolutional networks.
\newblock In \emph{ICLR}, 2016.

\bibitem[Veli{\v{c}}kovi{\'c} et~al.(2017)Veli{\v{c}}kovi{\'c}, Cucurull, Casanova, Romero, Lio, and Bengio]{velivckovic2017graph}
Petar Veli{\v{c}}kovi{\'c}, Guillem Cucurull, Arantxa Casanova, Adriana Romero, Pietro Lio, and Yoshua Bengio.
\newblock Graph attention networks.
\newblock In \emph{ICLR}, 2017.

\bibitem[Hamilton et~al.(2017)Hamilton, Ying, and Leskovec]{hamilton2017inductive}
Will Hamilton, Zhitao Ying, and Jure Leskovec.
\newblock Inductive representation learning on large graphs.
\newblock \emph{NeurIPS}, 2017.

\bibitem[Shuman et~al.(2013)Shuman, Narang, Frossard, Ortega, and Vandergheynst]{shuman2013emerging}
David~I Shuman, Sunil~K Narang, Pascal Frossard, Antonio Ortega, and Pierre Vandergheynst.
\newblock The emerging field of signal processing on graphs: Extending high-dimensional data analysis to networks and other irregular domains.
\newblock \emph{IEEE Signal Processing Magazine}, 2013.

\bibitem[Sandryhaila and Moura(2013)]{sandryhaila2013discrete}
Aliaksei Sandryhaila and Jos{\'e}~MF Moura.
\newblock Discrete signal processing on graphs.
\newblock \emph{IEEE Trans. Signal Processing}, 2013.

\bibitem[Bruna et~al.(2014)Bruna, Zaremba, Szlam, and LeCun]{bruna2013spectral}
Joan Bruna, Wojciech Zaremba, Arthur Szlam, and Yann LeCun.
\newblock Spectral networks and locally connected networks on graphs.
\newblock In \emph{ICLR}, 2014.

\bibitem[Defferrard et~al.(2016)Defferrard, Bresson, and Vandergheynst]{defferrard2016convolutional}
Micha{\"e}l Defferrard, Xavier Bresson, and Pierre Vandergheynst.
\newblock Convolutional neural networks on graphs with fast localized spectral filtering.
\newblock In \emph{NeurIPS}, 2016.

\bibitem[Gilmer et~al.(2017)Gilmer, Schoenholz, Riley, Vinyals, and Dahl]{gilmer2017neural}
Justin Gilmer, Samuel~S Schoenholz, Patrick~F Riley, Oriol Vinyals, and George~E Dahl.
\newblock Neural message passing for quantum chemistry.
\newblock In \emph{ICML}, 2017.

\bibitem[Sen et~al.(2008)Sen, Namata, Bilgic, Getoor, Gallagher, and Eliassi-Rad]{sen:aimag08}
Prithviraj Sen, Galileo Namata, Mustafa Bilgic, Lise Getoor, Brian Gallagher, and Tina Eliassi-Rad.
\newblock Collective classification in network data.
\newblock \emph{AI Magazine}, 2008.

\bibitem[Fey and Lenssen(2019)]{fey2019graph}
Matthias Fey and Jan~Eric Lenssen.
\newblock Fast graph representation learning with pytorch geometric, 2019.

\bibitem[Zhu et~al.(2020)Zhu, Yan, Zhao, Heimann, Akoglu, and Koutra]{zhu2020beyond}
Jiong Zhu, Yujun Yan, Lingxiao Zhao, Mark Heimann, Leman Akoglu, and Danai Koutra.
\newblock Beyond homophily in graph neural networks: Current limitations and effective designs.
\newblock In \emph{NeurIPS}, 2020.

\bibitem[Nt and Maehara(2019)]{nt2019revisiting}
Hoang Nt and Takanori Maehara.
\newblock Revisiting graph neural networks: All we have is low-pass filters.
\newblock \emph{arXiv}, 2019.

\bibitem[Altenburger and Ugander(2018)]{monophily}
Kristen~M. Altenburger and Johan Ugander.
\newblock {Monophily in social networks introduces similarity among friends-of-friends}.
\newblock \emph{Nature Human Behaviour}, 2\penalty0 (4):\penalty0 284--290, April 2018.
\newblock \doi{10.1038/s41562-018-0321-8}.
\newblock URL \url{https://ideas.repec.org/a/nat/nathum/v2y2018i4d10.1038_s41562-018-0321-8.html}.

\bibitem[Lim et~al.(2021)Lim, Hohne, Li, Huang, Gupta, Bhalerao, and Lim]{lim2021large}
Derek Lim, Felix~Matthew Hohne, Xiuyu Li, Sijia~Linda Huang, Vaishnavi Gupta, Omkar~Prasad Bhalerao, and Ser-Nam Lim.
\newblock Large scale learning on non-homophilous graphs: New benchmarks and strong simple methods.
\newblock In \emph{NeurIPS}, 2021.

\bibitem[Topping et~al.(2022)Topping, Di~Giovanni, Chamberlain, Dong, and Bronstein]{topping2021understanding}
Jake Topping, Francesco Di~Giovanni, Benjamin~Paul Chamberlain, Xiaowen Dong, and Michael~M Bronstein.
\newblock Understanding over-squashing and bottlenecks on graphs via curvature.
\newblock \emph{ICLR}, 2022.

\bibitem[Xu et~al.(2019)Xu, Hu, Leskovec, and Jegelka]{DBLP:conf/iclr/XuHLJ19}
Keyulu Xu, Weihua Hu, Jure Leskovec, and Stefanie Jegelka.
\newblock How powerful are graph neural networks?
\newblock In \emph{7th International Conference on Learning Representations, {ICLR} 2019, New Orleans, LA, USA, May 6-9, 2019}. OpenReview.net, 2019.

\bibitem[Grohe et~al.(2021)Grohe, Kersting, Mladenov, and Schweitzer]{Grohe2021ColorRA}
Martin Grohe, Kristian Kersting, Martin Mladenov, and Pascal Schweitzer.
\newblock Color refinement and its applications.
\newblock In \emph{An Introduction to Lifted Probabilistic Inference}. The MIT Press, 2021.

\bibitem[Scarselli et~al.(2009)Scarselli, Gori, Tsoi, Hagenbuchner, and Monfardini]{Scarselli:2009ku}
Franco Scarselli, Marco Gori, Ah~Chung Tsoi, Markus Hagenbuchner, and Gabriele Monfardini.
\newblock {The Graph Neural Network Model}.
\newblock \emph{IEEE Transactions on Neural Networks (TNN)}, 2009.

\bibitem[Li et~al.(2015)Li, Tarlow, Brockschmidt, and Zemel]{li2015gated}
Yujia Li, Daniel Tarlow, Marc Brockschmidt, and Richard Zemel.
\newblock Gated graph sequence neural networks, 2015.

\bibitem[Li et~al.(2016)Li, Tarlow, Brockschmidt, and Zemel]{DBLP:journals/corr/LiTBZ15}
Yujia Li, Daniel Tarlow, Marc Brockschmidt, and Richard~S. Zemel.
\newblock Gated graph sequence neural networks.
\newblock In Yoshua Bengio and Yann LeCun, editors, \emph{4th International Conference on Learning Representations, {ICLR} 2016, San Juan, Puerto Rico, May 2-4, 2016, Conference Track Proceedings}, 2016.
\newblock URL \url{http://arxiv.org/abs/1511.05493}.

\bibitem[Vr{\v{c}}ek et~al.(2022)Vr{\v{c}}ek, Bresson, Laurent, Schmitz, and {\v{S}}iki{\'c}]{vrvcek2022learning}
Lovro Vr{\v{c}}ek, Xavier Bresson, Thomas Laurent, Martin Schmitz, and Mile {\v{S}}iki{\'c}.
\newblock Learning to untangle genome assembly with graph convolutional networks.
\newblock \emph{arXiv preprint arXiv:2206.00668}, 2022.

\bibitem[Ma et~al.(2019)Ma, Hao, Yang, Li, Jin, and Chen]{spectral-dgcn}
Yi~Ma, Jianye Hao, Yaodong Yang, Han Li, Junqi Jin, and Guangyong Chen.
\newblock Spectral-based graph convolutional network for directed graphs, 2019.

\bibitem[Monti et~al.(2018)Monti, Otness, and Bronstein]{motifnet}
Federico Monti, Karl Otness, and Michael~M. Bronstein.
\newblock Motifnet: A motif-based graph convolutional network for directed graphs.
\newblock In \emph{IEEE Data Science Workshop (DSW)}, 2018.

\bibitem[Tong et~al.(2020{\natexlab{a}})Tong, Liang, Sun, Rosenblum, and Lim]{dgcn}
Zekun Tong, Yuxuan Liang, Changsheng Sun, David~S. Rosenblum, and Andrew Lim.
\newblock Directed graph convolutional network.
\newblock \emph{arXiv}, 2020{\natexlab{a}}.

\bibitem[Tong et~al.(2020{\natexlab{b}})Tong, Liang, Sun, Li, Rosenblum, and Lim]{digraph}
Zekun Tong, Yuxuan Liang, Changsheng Sun, Xinke Li, David Rosenblum, and Andrew Lim.
\newblock Digraph inception convolutional networks.
\newblock In \emph{NeurIPS}, 2020{\natexlab{b}}.

\bibitem[Zhang et~al.(2021)Zhang, He, Brugnone, Perlmutter, and Hirn]{zhang2021magnet}
Xitong Zhang, Yixuan He, Nathan Brugnone, Michael Perlmutter, and Matthew Hirn.
\newblock Magnet: A neural network for directed graphs.
\newblock In \emph{NeurIPS}, 2021.

\bibitem[Geisler et~al.(2023)Geisler, Li, Mankowitz, Cemgil, G\"{u}nnemann, and Paduraru]{pmlr-v202-geisler23a}
Simon Geisler, Yujia Li, Daniel~J Mankowitz, Ali~Taylan Cemgil, Stephan G\"{u}nnemann, and Cosmin Paduraru.
\newblock Transformers meet directed graphs.
\newblock In Andreas Krause, Emma Brunskill, Kyunghyun Cho, Barbara Engelhardt, Sivan Sabato, and Jonathan Scarlett, editors, \emph{Proceedings of the 40th International Conference on Machine Learning}, volume 202 of \emph{Proceedings of Machine Learning Research}, pages 11144--11172. PMLR, 23--29 Jul 2023.
\newblock URL \url{https://proceedings.mlr.press/v202/geisler23a.html}.

\bibitem[Maskey et~al.(2023)Maskey, Paolino, Bacho, and Kutyniok]{maskey2023fractional}
Sohir Maskey, Raffaele Paolino, Aras Bacho, and Gitta Kutyniok.
\newblock A fractional graph laplacian approach to oversmoothing, 2023.

\bibitem[Schlichtkrull et~al.(2018)Schlichtkrull, Kipf, Bloem, van den Berg, Titov, and Welling]{10.1007/978-3-319-93417-4_38}
Michael Schlichtkrull, Thomas~N. Kipf, Peter Bloem, Rianne van den Berg, Ivan Titov, and Max Welling.
\newblock Modeling relational data with graph convolutional networks.
\newblock In \emph{The Semantic Web}. Springer International Publishing, 2018.

\bibitem[Marcheggiani and Titov(2017)]{marcheggiani-titov-2017-encoding}
Diego Marcheggiani and Ivan Titov.
\newblock Encoding sentences with graph convolutional networks for semantic role labeling.
\newblock In \emph{Proceedings of the 2017 Conference on Empirical Methods in Natural Language Processing}. Association for Computational Linguistics, 2017.

\bibitem[Jaume et~al.(2019)Jaume, Nguyen, Mart{\'{\i}}nez, Thiran, and Gabrani]{DBLP:journals/corr/abs-1904-08745}
Guillaume Jaume, An{-}phi Nguyen, Mar{\'{\i}}a~Rodr{\'{\i}}guez Mart{\'{\i}}nez, Jean{-}Philippe Thiran, and Maria Gabrani.
\newblock edgnn: a simple and powerful {GNN} for directed labeled graphs.
\newblock \emph{CoRR}, 2019.

\bibitem[Vashishth et~al.(2020)Vashishth, Sanyal, Nitin, and Talukdar]{Vashishth2020Composition-based}
Shikhar Vashishth, Soumya Sanyal, Vikram Nitin, and Partha Talukdar.
\newblock Composition-based multi-relational graph convolutional networks.
\newblock In \emph{International Conference on Learning Representations}, 2020.

\bibitem[Chien et~al.(2020)Chien, Peng, Li, and Milenkovic]{chien2020adaptive}
Eli Chien, Jianhao Peng, Pan Li, and Olgica Milenkovic.
\newblock Adaptive universal generalized pagerank graph neural network.
\newblock In \emph{ICLR}, 2020.

\bibitem[Bo et~al.(2021)Bo, Wang, Shi, and Shen]{bo2021beyond}
Deyu Bo, Xiao Wang, Chuan Shi, and Huawei Shen.
\newblock Beyond low-frequency information in graph convolutional networks.
\newblock \emph{AAAI}, 2021.

\bibitem[Luan et~al.(2022)Luan, Hua, Lu, Zhu, Zhao, Zhang, Chang, and Precup]{luan2022revisiting}
Sitao Luan, Chenqing Hua, Qincheng Lu, Jiaqi Zhu, Mingde Zhao, Shuyuan Zhang, Xiao-Wen Chang, and Doina Precup.
\newblock Revisiting heterophily for graph neural networks.
\newblock In \emph{NeurIPS}, 2022.

\bibitem[Bodnar et~al.(2022)Bodnar, Di~Giovanni, Chamberlain, Li{\`o}, and Bronstein]{bodnar2022neural}
Cristian Bodnar, Francesco Di~Giovanni, Benjamin~Paul Chamberlain, Pietro Li{\`o}, and Michael~M Bronstein.
\newblock Neural sheaf diffusion: A topological perspective on heterophily and oversmoothing in gnns.
\newblock In \emph{NeurIPS}, 2022.

\bibitem[Di~Giovanni et~al.(2022)Di~Giovanni, Rowbottom, Chamberlain, Markovich, and Bronstein]{di2022graph}
Francesco Di~Giovanni, James Rowbottom, Benjamin~P Chamberlain, Thomas Markovich, and Michael~M Bronstein.
\newblock Graph neural networks as gradient flows.
\newblock \emph{arXiv}, 2022.

\bibitem[{Abu-El-Haija} et~al.(2019){Abu-El-Haija}, Perozzi, Kapoor, Alipourfard, Lerman, Harutyunyan, Steeg, and Galstyan]{abu-el-haijaMixHopHigherOrderGraph2019}
Sami {Abu-El-Haija}, Bryan Perozzi, Amol Kapoor, Nazanin Alipourfard, Kristina Lerman, Hrayr Harutyunyan, Greg~Ver Steeg, and Aram Galstyan.
\newblock {{MixHop}}: {{Higher-Order Graph Convolutional Architectures}} via {{Sparsified Neighborhood Mixing}}.
\newblock In \emph{ICML}, 2019.

\bibitem[Maurya et~al.(2021)Maurya, Liu, and Murata]{maurya2021improving}
Sunil~Kumar Maurya, Xin Liu, and Tsuyoshi Murata.
\newblock Improving graph neural networks with simple architecture design.
\newblock \emph{arXiv}, 2021.

\bibitem[Li et~al.(2022)Li, Zhu, Cheng, Shan, Luo, Li, and Qian]{li2022finding}
Xiang Li, Renyu Zhu, Yao Cheng, Caihua Shan, Siqiang Luo, Dongsheng Li, and Weining Qian.
\newblock Finding global homophily in graph neural networks when meeting heterophily.
\newblock In \emph{ICML}, 2022.

\bibitem[Bojchevski and Günnemann(2018)]{bojchevski2018deep}
Aleksandar Bojchevski and Stephan Günnemann.
\newblock Deep gaussian embedding of graphs: Unsupervised inductive learning via ranking.
\newblock In \emph{ICLR}, 2018.

\bibitem[Hu et~al.(2020)Hu, Fey, Zitnik, Dong, Ren, Liu, Catasta, and Leskovec]{hu2020ogb}
Weihua Hu, Matthias Fey, Marinka Zitnik, Yuxiao Dong, Hongyu Ren, Bowen Liu, Michele Catasta, and Jure Leskovec.
\newblock Open graph benchmark: Datasets for machine learning on graphs.
\newblock In \emph{NeurIPS}, 2020.

\bibitem[Pei et~al.(2020)Pei, Wei, Chang, Lei, and Yang]{pei2020geom}
Hongbin Pei, Bingzhe Wei, Kevin~Chen{-}Chuan Chang, Yu~Lei, and Bo~Yang.
\newblock Geom-gcn: Geometric graph convolutional networks.
\newblock In \emph{ICLR}, 2020.

\bibitem[Platonov et~al.(2023)Platonov, Kuznedelev, Diskin, Babenko, and Prokhorenkova]{platonov2023a}
Oleg Platonov, Denis Kuznedelev, Michael Diskin, Artem Babenko, and Liudmila Prokhorenkova.
\newblock A critical look at the evaluation of {GNN}s under heterophily: Are we really making progress?
\newblock In \emph{The Eleventh International Conference on Learning Representations}, 2023.

\bibitem[Rusch et~al.(2023)Rusch, Chamberlain, Mahoney, Bronstein, and Mishra]{rusch2022gradient}
T~Konstantin Rusch, Benjamin~P Chamberlain, Michael~W Mahoney, Michael~M Bronstein, and Siddhartha Mishra.
\newblock Gradient gating for deep multi-rate learning on graphs.
\newblock In \emph{ICLR}, 2023.

\bibitem[Ma et~al.(2022)Ma, Liu, Shah, and Tang]{ma2022is}
Yao Ma, Xiaorui Liu, Neil Shah, and Jiliang Tang.
\newblock Is homophily a necessity for graph neural networks?
\newblock In \emph{ICLR}, 2022.

\bibitem[Luan et~al.(2023)Luan, Hua, Xu, Lu, Zhu, Chang, Fu, Leskovec, and Precup]{luan2023graph}
Sitao Luan, Chenqing Hua, Minkai Xu, Qincheng Lu, Jiaqi Zhu, Xiao-Wen Chang, Jie Fu, Jure Leskovec, and Doina Precup.
\newblock When do graph neural networks help with node classification: Investigating the homophily principle on node distinguishability.
\newblock In \emph{NeurIPS}, 2023.

\bibitem[Barcelo et~al.(2022)Barcelo, Galkin, Morris, and Orth]{barcelo2022weisfeiler}
Pablo Barcelo, Mikhail Galkin, Christopher Morris, and Miguel~Romero Orth.
\newblock Weisfeiler and leman go relational.
\newblock In \emph{The First Learning on Graphs Conference}, 2022.
\newblock URL \url{https://openreview.net/forum?id=wY_IYhh6pqj}.

\bibitem[Weisfeiler and Leman(1968)]{weisfeiler1968reduction}
Boris Weisfeiler and Andrei Leman.
\newblock The reduction of a graph to canonical form and the algebra which appears therein.
\newblock \emph{NTI Series}, 1968.

\bibitem[Cai et~al.(1992)Cai, F{\"{u}}rer, and Immerman]{cai1992anoptimal}
Jin{-}yi Cai, Martin F{\"{u}}rer, and Neil Immerman.
\newblock An optimal lower bound on the number of variables for graph identifications.
\newblock \emph{Combinatorica}, 1992.

\bibitem[Kollias et~al.(2022)Kollias, Kalantzis, Id'e, Lozano, and Abe]{Kollias2022DirectedGA}
G.~Kollias, Vasileios Kalantzis, Tsuyoshi Id'e, Aur{\'e}lie~C. Lozano, and Naoki Abe.
\newblock Directed graph auto-encoders.
\newblock In \emph{AAAI Conference on Artificial Intelligence}, 2022.

\bibitem[Morris et~al.(2019)Morris, Ritzert, Fey, Hamilton, Lenssen, Rattan, and Grohe]{morris2019weisfeiler}
Christopher Morris, Martin Ritzert, Matthias Fey, William~L Hamilton, Jan~Eric Lenssen, Gaurav Rattan, and Martin Grohe.
\newblock Weisfeiler and leman go neural: Higher-order graph neural networks.
\newblock In \emph{Proceedings of the AAAI conference on artificial intelligence}, 2019.

\bibitem[Bevilacqua et~al.(2022)Bevilacqua, Frasca, Lim, Srinivasan, Cai, Balamurugan, Bronstein, and Maron]{bevilacqua2022equivariant}
Beatrice Bevilacqua, Fabrizio Frasca, Derek Lim, Balasubramaniam Srinivasan, Chen Cai, Gopinath Balamurugan, Michael~M Bronstein, and Haggai Maron.
\newblock Equivariant subgraph aggregation networks.
\newblock In \emph{International Conference on Learning Representations (ICLR)}, 2022.

\bibitem[Bodnar et~al.(2021)Bodnar, Frasca, Wang, Otter, Montufar, Li{\'o}, and Bronstein]{pmlr-v139-bodnar21a}
Cristian Bodnar, Fabrizio Frasca, Yuguang Wang, Nina Otter, Guido~F Montufar, Pietro Li{\'o}, and Michael Bronstein.
\newblock Weisfeiler and lehman go topological: Message passing simplicial networks.
\newblock In \emph{Proceedings of the 38th International Conference on Machine Learning}. PMLR, 2021.

\bibitem[Corso et~al.(2020)Corso, Cavalleri, Beaini, Li\`{o}, and Veli\v{c}kovi\'{c}]{corso2020principal}
Gabriele Corso, Luca Cavalleri, Dominique Beaini, Pietro Li\`{o}, and Petar Veli\v{c}kovi\'{c}.
\newblock Principal neighbourhood aggregation for graph nets.
\newblock In \emph{Advances in Neural Information Processing Systems}. Curran Associates, Inc., 2020.

\bibitem[Hornik et~al.(1989)Hornik, Stinchcombe, and White]{HORNIK1989359}
Kurt Hornik, Maxwell Stinchcombe, and Halbert White.
\newblock Multilayer feedforward networks are universal approximators.
\newblock \emph{Neural Networks}, 1989.

\bibitem[Tailor et~al.(2022)Tailor, Opolka, Lio, and Lane]{tailor2022adaptive}
Shyam~A. Tailor, Felix Opolka, Pietro Lio, and Nicholas~Donald Lane.
\newblock Adaptive filters for low-latency and memory-efficient graph neural networks.
\newblock In \emph{International Conference on Learning Representations}, 2022.

\bibitem[Barabasi and Albert(1999)]{Barabasi99emergenceScaling}
Albert-Laszlo Barabasi and Reka Albert.
\newblock Emergence of scaling in random networks.
\newblock \emph{Science}, 1999.

\bibitem[Xu et~al.(2018)Xu, Li, Tian, Sonobe, Kawarabayashi, and Jegelka]{pmlr-v80-xu18c}
Keyulu Xu, Chengtao Li, Yonglong Tian, Tomohiro Sonobe, Ken-ichi Kawarabayashi, and Stefanie Jegelka.
\newblock Representation learning on graphs with jumping knowledge networks.
\newblock In \emph{Proceedings of the 35th International Conference on Machine Learning}. PMLR, 2018.

\bibitem[He et~al.(2022)He, Zhang, Huang, Rozemberczki, Cucuringu, and Reinert]{he2022pytorch}
Yixuan He, Xitong Zhang, Junjie Huang, Benedek Rozemberczki, Mihai Cucuringu, and Gesine Reinert.
\newblock {PyTorch Geometric Signed Directed: A Software Package on Graph Neural Networks for Signed and Directed Graphs}.
\newblock \emph{arXiv}, 2022.

\end{thebibliography}
